\documentclass{article}

\usepackage[final]{neurips_2022}
\usepackage[utf8]{inputenc} % allow utf-8 input
\usepackage[T1]{fontenc}    % use 8-bit T1 fonts
\usepackage{hyperref}       % hyperlinks
\usepackage{url}            % simple URL typesetting
\usepackage{booktabs}       % professional-quality tables
\usepackage{amsfonts}       % blackboard math symbols
\usepackage{nicefrac}       % compact symbols for 1/2, etc.
\usepackage{microtype}      % microtypography
\usepackage{xcolor}         % colors

\usepackage{comment}
\usepackage{graphicx}
\usepackage{amsmath}
\usepackage{algorithm}
\usepackage{algorithmic}
\usepackage{multirow}

\usepackage{wrapfig}
\usepackage{adjustbox}
\usepackage{amsmath, amsthm, amssymb}
\newtheorem{proposition}{Proposition}
\newtheorem{lemma}{Lemma}

\def\semicolon{;}
\def\applytolist#1{
    \expandafter\def\csname multi#1\endcsname##1{
        \def\multiack{##1}\ifx\multiack\semicolon
            \def\next{\relax}
        \else
            \csname #1\endcsname{##1}
            \def\next{\csname multi#1\endcsname}
        \fi
        \next}
    \csname multi#1\endcsname}
\def\calc#1{\expandafter\def\csname c#1\endcsname{{\mathcal #1}}}
\applytolist{calc}QWERTYUIOPLKJHGFDSAZXCVBNM;
\def\bbc#1{\expandafter\def\csname bb#1\endcsname{{\mathbb #1}}}
\applytolist{bbc}QWERTYUIOPLKJHGFDSAZXCVBNM;

\usepackage[linewidth=1.2pt,linecolor=red]{mdframed}
\usepackage{subcaption}
\usepackage{makecell}

\title{Sparse Interaction Additive Networks via Feature Interaction Detection and Sparse Selection}

\author{%
James Enouen \\
Department of Computer Science \\
University of Southern California \\
Los Angeles, CA \\
\texttt{enouen@usc.edu} \\
\And
Yan Liu \\
Department of Computer Science \\
University of Southern California \\
Los Angeles, CA \\
\texttt{yanliu.cs@usc.edu} \\
}

\begin{document}

\maketitle

\begin{abstract}
There is currently a large gap in performance between the statistically rigorous methods like linear regression or additive splines and the powerful deep methods using neural networks.  
Previous works attempting to close this gap have failed to fully investigate the exponentially growing number of feature combinations which deep networks consider automatically during training.  
In this work, we develop a tractable selection algorithm to efficiently identify the necessary feature combinations by leveraging techniques in feature interaction detection.
Our proposed Sparse Interaction Additive Networks (SIAN) construct a bridge from these simple and interpretable models to fully connected neural networks.  
SIAN achieves competitive performance against state-of-the-art methods across multiple large-scale tabular datasets and consistently finds an optimal tradeoff between the modeling capacity of neural networks and the generalizability of simpler methods.
\end{abstract}

\section{Introduction}
Over the past decade, deep learning has achieved significant success in providing solutions to challenging AI problems like computer vision, language processing, and game playing \cite{he2015resnet,Vaswani2017transformer,deepmind2017alphaZero}.
As deep-learning-based AI models increasingly serve as important solutions to applications in our daily lives, we are confronted with some of their major disadvantages.
First, current limitations in our theoretical understanding of deep learning makes fitting robust neural networks a balancing act between accurate training fit and low generalization gap.
Questions which ask how the test performance will vary given the choice of architecture, dataset, and training algorithm remain poorly understood and difficult to answer.
Second, infamously known as blackbox models, deep neural networks greatly lack in interpretability.
This continues to lead to a variety of downstream consequences in applications: unexplainable decisions for stakeholders, the inability to distinguish causation from correlation, and a misunderstood sensitivity to adversarial examples.

In contrast,  simpler machine learning models such as linear regression, splines, and the generalized additive model (GAM) \cite{hastie1990originalGAM} naturally win in interpretability and robustness.
Their smaller number of parameters often have clear and meaningful interpretations and these methods rarely succumb to overfitting the training data.
The main shortcoming of these simpler methods is their inability to accurately fit more complex data distributions.

There is a growing strand of literature attempting to merge the interpretability of additive models with the streamlined differentiable power of deep neural networks (DNNs).
Key works in this direction such as NAM and NODE-GAM have been able to successfully model one-dimensional main effects and two-dimensional interaction effects using differentiable training procedures \cite{agarwal2020nam, chang2022nodegam}.
Although many other works have found similar success in modeling one- and two-dimensional interactions, few have made practical attempts towards feature interactions of size three or greater, which we will refer to throughout as {higher-order} feature interactions.
In this way, no existing works have been able to use the hallmark ability of neural networks to model higher-order interactions: amassing the influence of hundreds of pixels in computer vision and {combining} specific words from multiple paragraphs in language processing.
In this work, we bring interpretable models one step closer towards the {impressive} differentiable 
power of neural networks by developing a simple but effective selection algorithm and an efficient implementation to be able to train neural additive models which fit higher-order interactions of degrees three and greater. 
%We develop interaction sparsity through a combination of heredity and feature interaction detection to be able to construct and train neural additive models of dimension three and higher.
The proposed Sparse Interaction Additive Networks (SIANs) consistently achieve results that are competitive with state-of-the-art methods across multiple datasets.
We summarize our contributions as follows:
%1. NAMs of 3+
%2. blocksparse, faster implementation
%3. insights into generalization vs. capacity
\begin{itemize}
  \item We develop a feature interaction selection algorithm which efficiently selects from the exponential number of higher-order feature interactions by {leveraging heredity} and {interaction detection}.  This allows us to construct higher-order neural additive models for medium-scale datasets unlike previous works in neural additive modeling which only consider univariate and bivariate functions. % We intro
  \item We provide further insights into the tradeoffs faced by neural networks between better generalization and better functional capacity.  By tuning the hyperparameters of our SIAN model, we can gradually interpolate from one-dimensional additive models to full-complexity neural networks.  We observe fine-grained details about how the generalization gap increases as we add capacity to our neural network model.
  \item We design a block sparse implementation of neural additive models which is able to greatly improve the training speed and memory efficiency of neural-based additive models.  These improvements over previous NAM implementations allow shape functions to be computed in parallel, allowing us to train larger and more complex additive models.
  We provide code for our implementation and hope the tools provided can be useful throughout the additive modeling community for faster training of neural additive models. \footnote{Available at \href{https://github.com/enouenj/sparse-interaction-additive-networks}{github.com/EnouenJ/sparse-interaction-additive-networks}}
\end{itemize}
% https://github.com/EnouenJ/sparse-interaction-additive-networks

\section{Related Work}
The generalized additive model (GAM) \cite{hastie1990originalGAM} has existed for decades as a more expressive alternative to linear regression, replacing each linear coefficient with a nonparametric function.
Two-dimensional extensions appeared in the literature shortly after its introduction \cite{wahba1994ssanova}.
Over the years, an abundance of works have used the GAM model as an interpretable method for making predictions, with the choice of functional model typically reflecting the most popular method during the time period: regression splines, random forests, boosting machines, kernel methods, and most recently neural networks \cite{hooker2007functionalANOVA,caruana2015intelligible,kandasamy16salsa,yang2020gamiNet}. 
Two of the most prominent neural network based approaches are NAM and NODE-GA$^2$M \cite{agarwal2020nam, chang2022nodegam}.
The former stacks multilayer perceptrons to build a one-dimensional GAM; the latter connects differentiable decision trees to build a two-dimensional GAM.
Both have demonstrated competitive performance and interpretable trends learned over multiple tabular datasets.
Other neural network extensions \cite{wang2021partiallyInterpretableEstimatorPIE,oneill2021regressionNetworks} 
increase the modeling capacity to higher-order interactions by first training a two-dimensional model and then training an additional blackbox neural network to fit the residual error.
%forego having a completely interpretable prediction and instead train a blackbox neural network to fit the residual error after learning a two-dimensional additive model.
While this approach does have higher modeling capacity, it foregoes interpretable insights on the higher-order feature interactions and suffers the same inclination to overfit held naturally by deep neural networks. 

Other works in additive modeling instead focus on extending the univariate GAM to sparse additive models in the high-dimensional inference regime \cite{lin2006cosso,ravikumar2009spam,meier2009hdam,xu2022snam}.
Further extensions of these methods to sparse bivariate models for high-dimensional inference also exist \cite{tyagi2016spam2,liu2020ssam}.
These works extend classical high-dimensional inference techniques like LASSO and LARS from linear parameters to additive functions by shrinking the effect of minimally important variables and emphasizing sparse solutions. 
We note that, unlike these works, we do not use sparsity in the soft-constrained sense to shrink features from a fixed selection, but instead adaptively use feature interaction measurements to hierarchically build a feasible set of interactions.
The only existing work in additive modeling which uses hierarchy to induce sparsity in the same sense as this work is the GAMI-Net which uses a three stage procedure to select candidate pairs under a hierarchical constraint \cite{yang2020gamiNet}.
Extending their procedure to three or higher dimensions would require four or more stages of training and is left unexplored in their work.

Although the theoretical extension to three-dimensional additive models is clear, 
there is currently a lack of discussion surrounding the practical challenges faced when trying to model three-dimensional shape functions.
One of the few works to pursue practical implementation of higher-order GAMs on real-world datasets is the work of SALSA \cite{kandasamy16salsa}.
This work uses a specialized kernel function to fit additive models of order three and higher.
Their work also corroborates our finding that optimal performance is achieved by different orders for different datasets. 
However, the kernel-based approaches used in this work make it unsuitable beyond small-scale datasets with few samples.
%only a few thousand samples. 
This makes our work one of the first to train additive models of higher-order interactions which leverage the automatic differentiation and GPU optimization toolkits which have become commonplace in modern workstations.

\section{Methods}
\paragraph{Notation} We denote a $d$-dimensional input as $x\in\bbR^d$ with its $i$-th component as $x_i$; its corresponding output is denoted $y\in\bbR$.
We consider one-dimensional $y$ as in regression and binary classification.
We will use $f(x)$ to denote the function or model used to approximate $y$, implicitly considering the additive noise model $y=f(x)+\varepsilon$ for some noise term $\varepsilon$.
We denote a subset of the set of features by $\cI=\{i_1,\dots,i_{|\cI|}\}\subseteq[d]:=\{1,\dots,d\}$.
Its cardinality is denoted $|\cI|$, its complement $\setminus\cI$, and its power set $\cP(\cI)$.
For $x\in\bbR^d$, we define $x_\cI\in\bbR^d$ such that: 
\[
(x_\cI)_i =  \left\{\begin{array}{lr}
        x_i  & \text{if}\quad i\in \cI\\
        0 & \text{otherwise }\\
        \end{array}\right.
\]

\begin{figure*}[t!]
    \centering
    %\fbox{
    \includegraphics[width=0.9\textwidth]{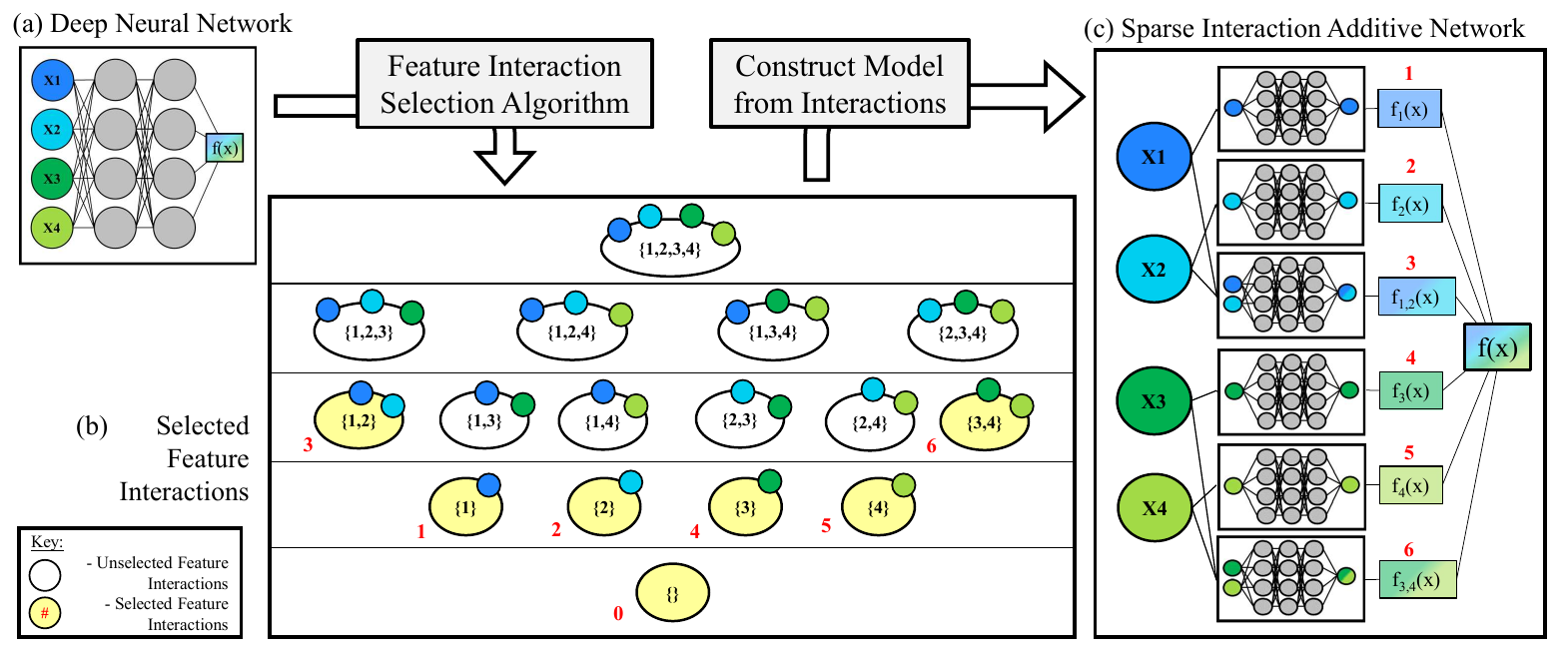}
    %}
    %\includegraphics[width=0.9\textwidth]{images/good_images/fis_pipeline_01_27_2022.PNG}
    \caption{SIAN Pipeline Diagram. (a) We train a base DNN to be able to learn the feature interactions from the dataset given the $d=4$ input features of $X_1,X_2,X_3,X_4$. (b) We feed our DNN into the FIS algorithm to be able to select from the $2^d=16$ possible subsets of $\{X_1,X_2,X_3,X_4\}$. (c) We use our selected feature subsets as the GAM array (of length 6 in the image) for our full SIAN network. The empty set function $f_{\{\}}$ is a constant we absorb into the last additive layer of $f(x)$. Finally, we train our SIAN neural network using the specified architecture.}
    \label{fig:sian_diagram}
\end{figure*}

%\subsubsection{Archipelago}
%\section{Methods}

\subsection{Generalized Additive Models}
\label{sec:gam}

We first consider the generalized additive model (GAM) \cite{hastie1990originalGAM}, which extends linear regression  by allowing each input feature to have a nonlinear relationship with the output.
\begin{equation}\label{eq:GAM_1}
g(y) = f_{\emptyset} + f_1(x_1) + \dots + f_d(x_d)
\end{equation}
Each of the $f_i$ `reshapes' their respective feature $x_i$ and then adds the reshaped feature to the total prediction.
These $f_i$ are hence called shape functions and were traditionally fit using regression splines.
The function $g$ is the link function which will be the identity function for regression and the inverse-sigmoid $g(y)=\log(\frac{y}{1-y})$ for classification.
$f_\emptyset$ is a normalizing constant.
This original formulation where each shape function considers only one feature we will further refer to as GAM-$1$.

\paragraph{Feature Interactions}
In order to extend this definition we must consider the interplay which occurs between multiple input features.
A `non-additive feature interaction' between features $\cI\subseteq[d]$ for the function $f$ is said to exist when the function $f$ cannot be decomposed into a sum of $|\cI|$ arbitrary subfunctions such that each subunction $f_i$ excludes one of the interacting variables $x_i$:
$f(x) \neq \sum_{i\in\cI} f_i(x_{\{1, 2 , \dots, d\}\setminus i})$ \cite{friedman2008predictive, sorokina2008detecting, tsang2018neural}.
In other words, the entire feature set  $\{x_i$ : $i\in\cI\}$ must be simultaneously known to be able to correctly predict the output $f(x)$. 
%For example, $f(x_1,x_2) = x_1 x_2$ \textbf{does} have an additive feature interaction whereas $f(x_1,x_2) = \log( x_1 x_2)$ \textbf{does not} since it can be decomposed as $\log(x_1) + \log(x_2)$.

The goal of \textit{feature interaction detection} is to uncover these groups of features which depend on one another.
For smooth functions, this can be quantitatively done by finding the sets $\mathcal{I}\subseteq[d]$ such that the \textit{interaction strength}, $\omega(\mathcal{I})$, is positive and large.
\begin{equation}
\omega({\cI}) := 
\bbE_{x}\left[\frac{\partial^{|{\cI|}} f(x)}{\partial x_{i_1}\partial x_{i_2} \dots\partial x_{i_{|\mathcal{I}|}}}\right]^2 > 0.
\label{eq:interaction_strength}
\end{equation}

We may now adjust the GAM definition to capture feature interactions by considering a set of $T$ specified interactions, $\{\mathcal{I}_t\}_{t=1}^T$, where  each $\cI_t\subseteq[d]$ is an interaction of size 
$|\cI_t|$:
\begin{equation}\label{eq:GAM_K}
g(y) = f_\emptyset + \sum_i f_i(x_i) +
        \sum_{t} f_{\mathcal{I}_t}(x_{\mathcal{I}_t})
\end{equation}
The third term extends GAMs to full capacity models which can represent nonlinear dependencies of arbitrary feature sets.
For instance, if our set of interactions $\{\mathcal{I}_t\}_{t=1}^T$ includes the complete feature set $[d]$, then our model has {exactly} the same capacity as the underlying functional model we choose for the shape functions (splines, random forests, deep neural networks, etc.)
%In our case, we will be using neural networks as soon we will soon describe.
An abundance of previous works have used this extended version of the GAM model, often called the Smoothing Spline ANOVA model or the functional ANOVA model \cite{caruana2015intelligible,ravikumar2009spam,meier2009hdam,hooker2007functionalANOVA,wahba1994ssanova}.
%SS-anova, fnl-anova, spam, hdam, ga2m, 
%cite 

Throughout this work, we will use GAM-$K$ to refer to a GAM whose highest order interaction in $\{\mathcal{I}_t\}_{t=1}^T$ is of cardinality $K$.
(i.e. $|\cI_t|\leq K$ $\forall t\in[T]$.)
For instance, we will call the NODE-GA$^2$M model \cite{chang2022nodegam} a GAM-2 model since the interaction sets are all possible feature pairs: $\{ \{i,j\} : i<j \hspace{0.5em}\text{for}\hspace{0.5em} i,j\in[d]\}$.
We will refer to our SIAN networks as SIAN-$K$ using the same convention.

\subsection{Feature Interaction Selection}
A key concern of using neural networks to fit the shape functions is keeping the number of networks low enough that our training time computation is kept reasonable.
While this is typically not a problem for the GAM-$1$, this can quickly become an issue for the GAM-$2$.
For instance, if we have an input variable $x$ with 30 features, then including all pairwise functions would need to cover the
$\binom{30}{2}=435$ possible pairs.
Although learning $465$ linear coefficients is reasonable,
%we should keep in mind that this will correspond to 
training hundreds of neural networks becomes less so.
Moreover, this quantity only grows exponentially as we increase $K$ to higher-order interactions (3, 4, 5, etc.)
%In an effort to not only combat this growing complexity, but also regularize our model according to the sparsity structure of the true feature interactions, we will combine the methods of feature interaction detection with our novel sparse selection algorithm to drastically reduce the number of tuples.
%We find that the resulting SIAN model is able to maintain high quality performance with a much greater level of interpretability as we will show in our results.
%In order to construct our SIAN model, we must first provide the sparse blueprints for the GAM architecture.
%Hence, we introduce our Feature Interaction Selection (FIS) algorithm to be able to detect and specify the interactions which should be used in our SIAN model.
In an effort to combat this growing complexity, we introduce a Feature Interaction Selection (FIS) algorithm which {depends} on two key ingredients: an interaction detection procedure and a heredity condition.

\paragraph{Feature Interaction Detection}
In recent years, there has been a growing body of work focused on detecting and measuring the strength of feature interactions from large-scale data.
Three of the most popular and generally applicable of these methods are the Shapley Additive Explanations (SHAP) \cite{lundberg2017shapleySHAP,lundberg2020local2global,dhamdhere2019shapley}, Integrated Hessians~\citep{sundararajan2017integratedGradients,janizek2020explaining}, and Archipelago  \cite{tsang2020archipelago}.
For our experiments, we primarily use an adaptation of Archipelago for higher-order interactions because of its compatibility with blackbox models.
In contrast, Integrated Hessians is only applicable to sufficiently smooth networks (ReLU networks are not compatible) and SHAP only has fast implementations available for tree-based approaches.
Moreover, although both SHAP and Integrated Hessians have clear ideological extensions to higher-order interactions, there are no currently available implementations.
A detailed discussion surrounding Archipelago and other detection techniques can be found in Appendix \ref{app_sec:feature_interaction_selection_details}.

\paragraph{Heredity}
The practice of only modeling the pairwise interaction effect $\{i,j\}\subseteq[d]$ for some features $i,j\in[d]$ when both of the main effects $\{i\}$ and $\{j\}$ are already being modeled has a long history throughout statistics \cite{peixoto1987hierarchicalPolynomials,Chipman1995BayesianVS,bien2013hierarchicalLasso}.
There are two main versions of this hierarchical principle explored in the literature: strong heredity and weak heredity.
%Consider the interaction pair $\{i,j\}$, there are two typical notions of heredity.  
If we are given that $\omega(\{i,j\})>0$, strong heredity implies that both $\omega(\{i\})>0$ and $\omega(\{j\})>0$ whereas weak heredity implies that $\omega(\{i\})>0$ or $\omega(\{j\})>0$.
For our definition of a feature interaction, we have that strong heredity holds; however, our algorithm will instead focus on a computational version of heredity which asks that $\tau$ percent of subsets are above threshold $\theta$ in order for a pair (triple, tuple, etc.) to be considered as a possible interaction.

\vspace{-\baselineskip }
%\begin{wrapfigure}[37]{LB}{.5\textwidth}
%\begin{wrapfigure}[34]{LB}{.5\textwidth}
\begin{wrapfigure}[34]{RB}{.5\textwidth}
%\begin{wrapfigure}[37]{LB}{1.\textwidth}
\begin{minipage}{.5\textwidth}
%\begin{minipage}{1.\textwidth}
\begin{algorithm}[H]%[tb]
\caption{Feature Interaction Selection (FIS)}
\label{alg:feature_interaction_selection}
\textbf{Inputs}: Trained prediction model $f(x)$, and validation dataset $X^*=\{x^{(1)},\dots,x^{(V)}\}$\\
\textbf{Parameters}: Cutoff index $K$, cutoff threshold $\tau$, strength threshold $\theta$ \\
\textbf{Output}: $\cI$, a family of feature interactions with index at most $K$ and strength above $\theta$

\begin{algorithmic}[1] %[1] enables line numbers
%\STATE Set $\cI \gets \{\{i\} : i\in[d]\} \cup \emptyset$
%\COMMENT{// The true detected interactions so far}
%\STATE Set $\cJ \gets \{\{i,j\} : i,j\in[d]; i<j\}$
%\COMMENT{// The next set of interactions to check}
\STATE Set $\cI \gets \emptyset$
%\COMMENT{~detected interactions so far}
\STATE Set $\cJ \gets \{\{i\} : i\in[d]\}$
%\COMMENT{~interactions to check}
\STATE $k \gets 1$
\WHILE{$k\leq K$}
%\STATE place holder
\FOR{$J$ in $\cJ$}
\STATE $\omega_J(X^*) \gets 0$
\FOR {$x^{(v)}\in X^*$}
\STATE $\omega_J(x^{(v)}) \gets \text{Archipelago}(f,x^{(v)},J)$
\STATE $\omega_J(X^*) \gets \omega_J(X^*) + \omega_J(x^{(v)})$
\ENDFOR
\STATE $\omega_J(X^*) \gets \omega_J(X^*)\cdot\frac{1}{|X^*|}  $
\IF{$\omega_J(X^*) > \theta$}
\STATE $\cI \gets \cI \cup \{J\}$
\ENDIF
\ENDFOR
\STATE $\cJ' \gets \{ J : J\in\cP([d]); |J| = k+1 \}$
\FOR{$J$ in $\cJ'$}
\STATE $\nu(J) \gets \frac{1}{{|J|}}\sum_{I\subseteq J}[{1}_{I\in\cI}\cdot {1}_{\{|J|-|I|=1\}}] $
\ENDFOR
\STATE $\cJ \gets \{ J : J\in\cJ'; \nu(J) > \tau \}$
\STATE $k \gets k + 1$
\ENDWHILE
\STATE \textbf{return} $\cI$
\end{algorithmic}
\end{algorithm}
\end{minipage}
\end{wrapfigure}

\paragraph{Algorithm}
In Algorithm \ref{alg:feature_interaction_selection}, we detail how we use Archipelago to build our FIS algorithm.
The visual overview of the SIAN pipeline is also depicted in Figure \ref{fig:sian_diagram} above. %yyy
We start by training a reference DNN 
which is required for the inductive insights generated by Archipelago.
We then pass the trained function approximator to the FIS algorithm along with hyperparameters $K$ (interpretability index), $\tau$ (heredity strength threshold), and $\theta$ (interaction strength threshold).
This procedure efficiently searches and ranks the $2^d$ possible feature subsets to produce a final set of interactions which we use to specify the SIAN model architecture. 
%Our algorithm depends on the fact that the existence of the interaction $\{1,2,3\}$ implies the existence of the interactions $\{1,2\}, \{1,3\}, \{2,3\}$ to sequentially construct larger interactions.
Finally, we train our SIAN neural network using typical gradient descent methods.
%In our experiments, we consistently find that there is an optimal interpretability index $K^*$ where the SIAN model is not only able to match the base DNN's test performance, but actually outperform the deep networks, kernel machines, random forests, and boosting machines in our experiments.
In Appendix \ref{app_sec:theoretical_discussion}, we provide a further theoretical discussion in which we prove exact recovery of the true feature interactions and show how this sparse selection leads to provably lower generalization gaps in a toy additive noise model.

\subsection{Block Sparse Implementation}
In order to improve the time and memory efficiency of the SIAN model, we implement a block sparse construction for neural additive models.
The default scheme of SIAN is to use a network module for each of the shape functions and additively combine the output features, following the implementation strategies of NAM and other popular neural additive models.
However, since each shape function network is computed sequentially, this greatly bottlenecks the computation speed of the model.
We instead construct a large, block sparse network which computes the hidden representations of all shape function networks simultaneously, leveraging the fact that each shape subnetwork has the same depth.
%Leveraging the fact that each shape function network has the same depth, we instead construct a large block sparse network which computes the concatenated hidden representations of each shape function.
%subsumes all shape functions by concatenating the hidden representations of each shape function at each layer.
Using this block sparse network allows for shape features to be computed in parallel, leading to a significant improvement in training speed.
The main consequence of this design is a higher footprint in memory; therefore, we also develop a compressed sparse matrix implementation of the network which has a greatly reduced memory footprint for saving network parameters.
SIAN is able to interchange between these different modes with minimal overhead, allowing for  
faster training in the block sparse module, lower memory footprint in the compressed module, and convenient visualization of shape functions in the default module.
We provide further numerical details of our gains in Section \ref{sec:compute_speedups}.
%yyy

\section{Datasets}
Our experiments focus on seven machine learning datasets.
Two are in the classification setting, the MIMIC-III Healthcare and the Higgs datasets \cite{2016mimicIII,pierre2014exoticParticlesHiggsDataset}.
The other five are in the regression setting, namely the Appliances Energy, Bike Sharing, California Housing Prices, Wine Quality, and Song Year datasets \cite{candanedo2017appliancesEnergyDataset,2013bikeSharing,1997caliHousing,cortez2009wineQualityDataset,BertinMahieux2011millionSongYearDataset}.
More details about each dataset are provided in Table \ref{tab:datasets}.
We evaluate the regression datasets using mean-squared error (MSE).
We measure the performance on the classification datasets using both the area under the receiver operating characteristic (AUROC) and the area under the precision-recall curve (AUPRC) metrics.
We report both metrics for the MIMIC dataset since the positive class is only 9\% of examples and report only AUROC for the Higgs dataset since it is relatively well-balanced.

\subsection{Experiment Details}
For the baseline DNNs we are using hidden layer sizes [256,128,64] with ReLU activations. 
%In the appendix we provide evidence that the FIS base model needs to only be trained for a handful of epochs (10) before it gives accurate feature interaction strength estimates.
For the GAM subnetworks we are using hidden layer sizes [16,12,8] with ReLU activations.
We use L1 regularization of size 5e--5.
%We also consider DNNs of sizes (1024,512,256) and (128,64,32) when appropriate (referred to as DNN-L and DNN-S respectively.)
%The parameters for the FIS algorithm are the degree of interactions K and the feature interaction strength threshold $\tau$.
In the main results section, we report the results for each $K\in\{1,2,3,5\}$ using only a single value of $\tau$ and $\theta$.
The hyperparameter $\tau$ was taken to be $0.5$ throughout and $\theta$ was selected from a handful of potential values using a validation set.
We train all networks using Adagrad
%\cite{duchi2011adaptive} 
with a learning rate of 5e--3.  
%Setting this $\tau$ varied by dataset but usually corresponded to ...
All models are evaluated on a held-out test dataset over five folds of training-validation split unless three folds are specified.
Three folds are used for NODE-GAM on all datasets as well as Song Year and Higgs for all models.
We respect previous testing splits when applicable, otherwise we subdivide the data using an 80-20 split to generate a testing set.
In addition to NODE-GA2M, 
we compare against the interpretable models LASSO and GA$^2$M EBM as well as the popular machine learning models of support vector machines (SVM), random forests (RF), and extreme gradient boosting (XGB) \cite{chen2016xgboost,nori2019interpretml}.

%thanks christ, latex is so obnoxious
%https://tex.stackexchange.com/questions/444449/tables-side-by-side-with-two-captions-for-each-one

\begin{table}[h]
\begin{minipage}{.538\linewidth}
    \centering

%\vspace{-2.6em}
\caption{Real-world datasets. $n$ is the number of samples, $d$ is the number of features, and $p$ is the percentage of data samples with the positive class label.}
    \label{tab:datasets}

    \medskip

\resizebox{.95\columnwidth}{!}{
\begin{tabular}{lrrc}
\toprule
% Dataset & $n$  & $d$ & $\%$pos \\
 Dataset & \multicolumn{1}{r}{$n$}  & \multicolumn{1}{r}{$d$} & \multicolumn{1}{c}{$p$} \\
\midrule
 Higgs Boson & $11,000,000$  & $28$ & \multicolumn{1}{r}{ $53.0\%$ } \\
 MIMIC-III & $32,254$  & $30$ & \multicolumn{1}{r}{ $9.2\%$ }\\
% Shopping & --  & -- & -- \\
 %   $11,000,000$ $ 1.10\cdot 10^7$
 \midrule
 Energy Appliances & $19,735$  & $30$ & -- \\
 Bike Sharing & $17,379$  & $13$ & -- \\
 California Housing & $20,640$  & $8$ & -- \\
 Wine Quality & $ 6,497$  & $12$ & -- \\
 Song Year & $ 515,345$  & $90$ & -- \\
\bottomrule
\end{tabular}
}
\end{minipage}\hfill
\begin{minipage}{.37\linewidth}
    \centering

    \caption{MIMIC-III Performance.} 
    \label{table:mimic_results_icml}

    \medskip

\resizebox{1.\columnwidth}{!}{
\begin{tabular}{r|c|c} % centered columns (4 columns)
\hline\hline %inserts double horizontal lines
Model & {AUROC} ($\Uparrow$) & {AUPRC} ($\Uparrow$)  \\
\hline\hline
SAPS II & $ 0.792 $  & $ 0.281 $ \\
SOFA & $ 0.703$ & $ 0.225 $ \\
%\hline
LASSO &  $ 0.568 $ & $ 0.396$ \\
%nam & & \\
GA$^2$M EBM & $ 0.840 $  & $ 0.375 $   \\
NODE-GA$^2$M & $ 0.826 $  & $ 0.345 $   \\
\hline
SIAN-1 & $ 0.848 $ & $ 0.409 $  \\
SIAN-2 & $\mathbf{ 0.855 }$ & $\mathbf{ 0.423 }$ \\
SIAN-3 & $\mathbf{ 0.856 }$ & $\mathbf{ 0.425 }$ \\
SIAN-5 & $\mathbf{ 0.856 }$ & $\mathbf{ 0.425 }$ \\
\hline
RF  & $ 0.821 $ & $ 0.369 $ \\
SVM & $ 0.831 $ & $ 0.404 $  \\
XGB & $ 0.843 $  & $ {0.382} $  \\
DNN & $ {0.844} $ & $ {0.382} $  \\
\hline

\hline %inserts single line
\end{tabular}
}
\end{minipage}
\end{table}

\section{Results}
Across seven different datasets, SIAN achieves an average rank of 3.00 out of the 8 models we consider.
The next best performing model, NODE-GA2M, has an average rank of 3.71 out of 8.
The third best performing model, DNN, has an average rank of 3.86 out of 8.
We find that SIAN achieves consistent performance by being able to adapt to both the low-dimensional datasets and the high-dimensional datasets, finding a balance between good training fit and good generalization.

\begin{table*}[h]
\centering 
\caption{{Test} metrics for six of seven datasets. %update underline for SIAN degree?
($\Uparrow$)/($\Downarrow$) indicates higher/lower is better, respectively.}
%\vskip 0.10in
\resizebox{1.\columnwidth}{!}{%
\begin{tabular}{r|c c c c c c} 
\hline\hline %inserts double horizontal lines
% & \multicolumn{6}{c}{Dataset} \\
Model & Appliances Energy ($\Downarrow$) & Bike Sharing ($\Downarrow$) & California Housing ($\Downarrow$) &  Wine Quality ($\Downarrow$) & Song Year ($\Downarrow$) & Higgs Boson ($\Uparrow$)\\
\hline\hline
LASSO &                0.740$\pm$0.002 & 1.053$\pm$0.001 & 0.478$\pm$0.000 & 0.575$\pm$0.002  &     1.000$\pm$0.008    &  0.635$\pm$0.000  \\
%nam &                                &  0.650$\pm$0.058 & 0.649$\pm$0.035 &                &                     &                  \\ %I am, with some disappointment, deciding to remove the NAM column from the submission because there is less than 48 hours until the submission and I am having significant trouble getting realistic results using NAM; it is possible that I need to use a NAM ensemble, it is also possible that NAM is just extremely sensitive to hyperparameters
GA$^2$M EBM        & 1.053$\pm$0.138  & 0.124$\pm$0.004 & 0.265$\pm$0.002 & 0.498$\pm$0.004  &  0.894$\pm$0.001 & 0.698$\pm$0.001  \\
NODE-GA$^2$M         &  1.064$\pm$0.056 & \textbf{ 0.111$\pm$0.006 } & \textbf{0.222$\pm$0.005} & 0.521$\pm$0.009  & 0.806$\pm$0.001 & 0.811$\pm$0.000\\
\hline
SIAN-1          & \textbf{ 0.718$\pm$0.007 }  & 0.387$\pm$0.035 & 0.378$\pm$0.007 & 0.551$\pm$0.004 & 0.860$\pm$0.001 & 0.771$\pm$0.001 \\ 
SIAN-2          & 0.763$\pm$0.009    & 0.127$\pm$0.008 & 0.302$\pm$0.002 &  0.497$\pm$0.003 &       0.842$\pm$0.002   & 0.795$\pm$0.001 \\
SIAN-3          & 0.808$\pm$0.026    & 0.125$\pm$0.013 &  0.278$\pm$0.001 & 0.497$\pm$0.003 &    0.831$\pm$0.001     & 0.798$\pm$0.001 \\
SIAN-5          & 0.801$\pm$0.031   & 0.149$\pm$0.011 & 0.272$\pm$0.003 & 0.484$\pm$0.006     &   0.821$\pm$0.001  &   0.802$\pm$0.001  \\ %XXX
\hline
RF             & 1.114$\pm$0.095 & 0.206$\pm$0.009 & 0.271$\pm$0.001 &     \textbf{0.439$\pm$0.005} &   0.994$\pm$0.005  &  0.654$\pm$0.002 \\
SVM            &  0.740$\pm$0.008    & 0.168$\pm$0.001 & 0.262$\pm$0.001 & 0.457$\pm$0.008 & 0.940$\pm$0.012  & 0.698$\pm$0.001 \\ 
XGB            & 1.188$\pm$0.119    & 0.157$\pm$0.003 & 0.229$\pm$0.002 &  0.465$\pm$0.014 &  0.881$\pm$0.002 & 0.740$\pm$0.000 \\
DNN           & 0.945$\pm$0.054  & 0.374$\pm$0.017 & 0.283$\pm$0.005 &   0.495$\pm$0.007 &   \textbf{ 0.791$\pm$0.002}    &  \textbf{0.823$\pm$0.000}   \\ 
\hline
\end{tabular}
}
\label{table:uci_datasets_results_icml}
\end{table*}

For the MIMIC dataset, we can see the models' performances in terms of both AUROC and AUPRC in Table \ref{table:mimic_results_icml}.
In addition to the models we previously described, we also compare against the interpretable medical baselines of SOFA and SAPS II which are simple logic-based scoring methods \cite{1993sapsII,1997sofa}.
We see that the machine learning methods improve over the baseline performance achieved by the SAPS and SOFA methods.

In Table \ref{table:uci_datasets_results_icml}, we can see the combined results for our six other datasets.
First, in the Appliances Energy dataset, we see that the SIAN-1 performs the best, with LASSO and SVM trailing slightly behind.
The success of one-dimensional methods on this dataset could be indicative that many of the dataset's trends are one-dimensional.
As we increase the dimension of the SIAN, the test error slowly increases as the generalization gap grows; the full-complexity DNN has even worse error than all SIAN models.

Second, in the Bike Sharing dataset, we see that the best performing model is the NODE-GA2M, with the EBM, SIAN-2, and SIAN-3 trailing only slightly behind.
All of these methods are two or three dimensional GAM models, again hinting that a significant portion of the important trends in this dataset could be bivariate.
Indeed, the most important bivariate trend accounts for more than $50\%$ of the variance in the dataset, as we explore in Figure \ref{fig:bikeShare_interactionsAndHourXWorkday} below. %yyy

For the remaining datasets, we see that the performance of the SIAN model improves as we add higher and higher-order feature interactions.
For the California Housing dataset, we find the best performance using the differentiable tree method of NODE-GA2M.
For the Wine Quality dataset, we find the best performance using the random forest algorithm.
For the larger-scale datasets of Song Year and Higgs Boson, we find that the best performance is still obtained by a full-complexity deep neural network.
These two are the only datasets where SIANs of order five or less are not sufficient to outperform vanilla deep neural networks, implying there are important feature interactions of degree greater than five.

\subsection{Training Speed and Storage} 
\label{sec:compute_speedups}
In Table {4} %yyy
below we see how our SIAN network compares against other popular differentiable additive models in both training time and size on disk.
For fair comparison we do not utilize our interaction selection algorithm for SIAN in this section, instead training SIAN-2 with all possible pairs.
We see that our implementation of different modes for the SIAN architecture allows us to outperform both NAM and NODE-GAM, with 2-8x faster training on GAM-1 models and 10-80x faster training on GAM-2 models.
We reiterate that because the overhead for switching between modes is negligible, SIAN enjoys the benefits of all modes: training quickly in the block sparse mode and saving succinctly in the compressed mode.

\begin{table}[h]
\centering 
\caption{Training Time and Memory Size. Experiments are run with a machine using a GTX 1080 GPU and 16GB of RAM.  The average wall-clock time over three runs is reported. } 
\label{table:training_speed_table}
\resizebox{1.\columnwidth}{!}{%
\begin{tabular}{|l|l||r|r|r|r||r|r|r|r|}
%\hline
\cline{3-10}
 \multicolumn{2}{c||}{} &  \multicolumn{1}{|c}{\textbf{SIAN-2}} & \multicolumn{1}{c}{(block-} & (comp- & \textbf{NODE-} &  \multicolumn{1}{|c}{\textbf{SIAN-1}} & \multicolumn{1}{c}{(block-} & (comp- & \textbf{NAM}   \\ 
 \multicolumn{2}{c||}{} & \multicolumn{1}{|c}{(default)} & \multicolumn{1}{c}{sparse)} & ressed) & \textbf{GA2M} & \multicolumn{1}{|c}{(default)} & \multicolumn{1}{c}{sparse)} & ressed) & \\
%\cline{3-10} \cline{1-10} 
\hline\hline
\multirow{3}{*}{\begin{tabular}[c]{@{}l@{}}Training Time\\(minutes)\end{tabular}} & Wine &
 17.57  & \textbf{0.69} & --  & 55.19  & 
%  0:17.57  & \textbf{0.69} & --  & 55.19  & 
  3.03   & \textbf{0.66} & -- & 5.45 \\ %\cline{3-10} 
& Bike 
& 159.65 & \textbf{5.38} & --  & 60.75 
%& 2:39.65 & \textbf{5.38} & --  & 60.75     
& 25.65  & \textbf{4.58} & -- &  12.09  \\ %\cline{3-10} 
& House 
 & 88.67  & \textbf{6.09} & --  & 58.03  
% & 1:28.67  & \textbf{6.09} & --  & 58.03    
& 22.72  & \textbf{5.88} & -- &  14.74  \\ \hline%\cline{3-10} 
                                                                               \multirow{3}{*}{\begin{tabular}[c]{@{}l@{}}Size on Disk\\(KB)\end{tabular}}      & Wine
                                                                               
& 7,113   & 7,113 & \textbf{537}  & 1,040      
& 182    & 182  & \textbf{86} & 1,203  \\ 
 & Bike 
& 9,727   & 9,666 & \textbf{626}  & 1,040      
& 218    & 212  & \textbf{92} & 308 \\ 
& House 
& 1,526   & 1,526 & \textbf{251} & 1,040      
& 83     & 83   & \textbf{59} & 1,921 \\ \hline
\end{tabular}
} %end resizebox
\end{table}

\subsection{Beyond ReLU Networks}
\label{sec:not_just_relu}
The FIS algorithm we describe can be applied to other combinations of FID algorithm + functional model besides Archipelago + ReLU Neural Networks.
To demonstrate the general applicability of our scheme, we replace the continuous, piecewise-linear functions of ReLU neural networks with the piecewise-constant, differentiable decision trees using NODE-GAM.
We extend the original implementation of NODE-GAM to handle feature triplets, extending the method to a trivariate function or GAM-3 model.
We run our FIS algorithm using $K=3$ on both the Housing and Wine datasets to fit GAM-3 models using the inductive biases of the NODE architecture.

Extending from NODE-GA2M to NODE-GA3M is able to improve performance from $0.222$ to $0.175$ on the Housing dataset, a $21.2\%$ further improvement over the state of the art method.
The same extension is only able to deliver a $4.8\%$ improvement from $0.521$ to $0.496$ on the Wine dataset; however, both the SIAN-5 and NODE-GA3M are able to improve performance over all previously available additive models.
These two examples demonstrate the ability of our FIS pipeline to be applied to more general machine learning techniques to model higher-order interactions.
%and hopefully will inspire further research of this kind.

\begin{figure}[h]
    \centering
    \includegraphics[width=1.\columnwidth]{{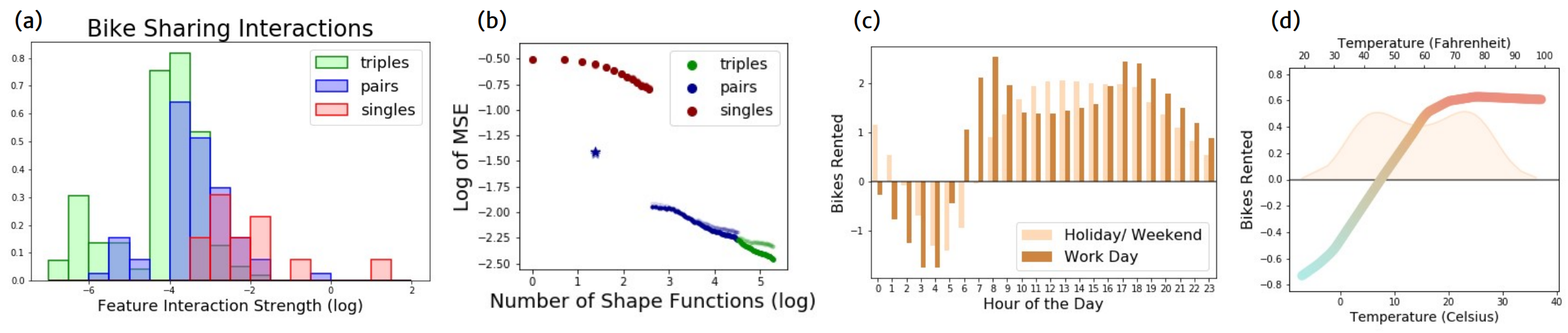}}
    \caption{
    \small
    Analysis of Bike Sharing. 
    (a) We plot the histogram of feature interaction strength, adjusting the vertical scaling by degree to display in a single plot.
    (b) We plot the training and validation performance (darker and lighter, respectively) as we add feature interactions to SIAN. We use an exponentially weighted moving average to reduce variance instead of training multiple networks of each size. The blue star is the performance of the shape functions in \ref{fig:bikeShare_interactionsAndHourXWorkday}c and \ref{fig:bikeShare_interactionsAndHourXWorkday}d.
    (c) The most important shape function showing the interaction between hour of the day and work day. (d) The next shape function showing temperature's effect on bikes rented.
    }
    \label{fig:bikeShare_interactionsAndHourXWorkday}
\end{figure}
\begin{figure}[h]\
    \centering
    \includegraphics[width=1.\columnwidth]{{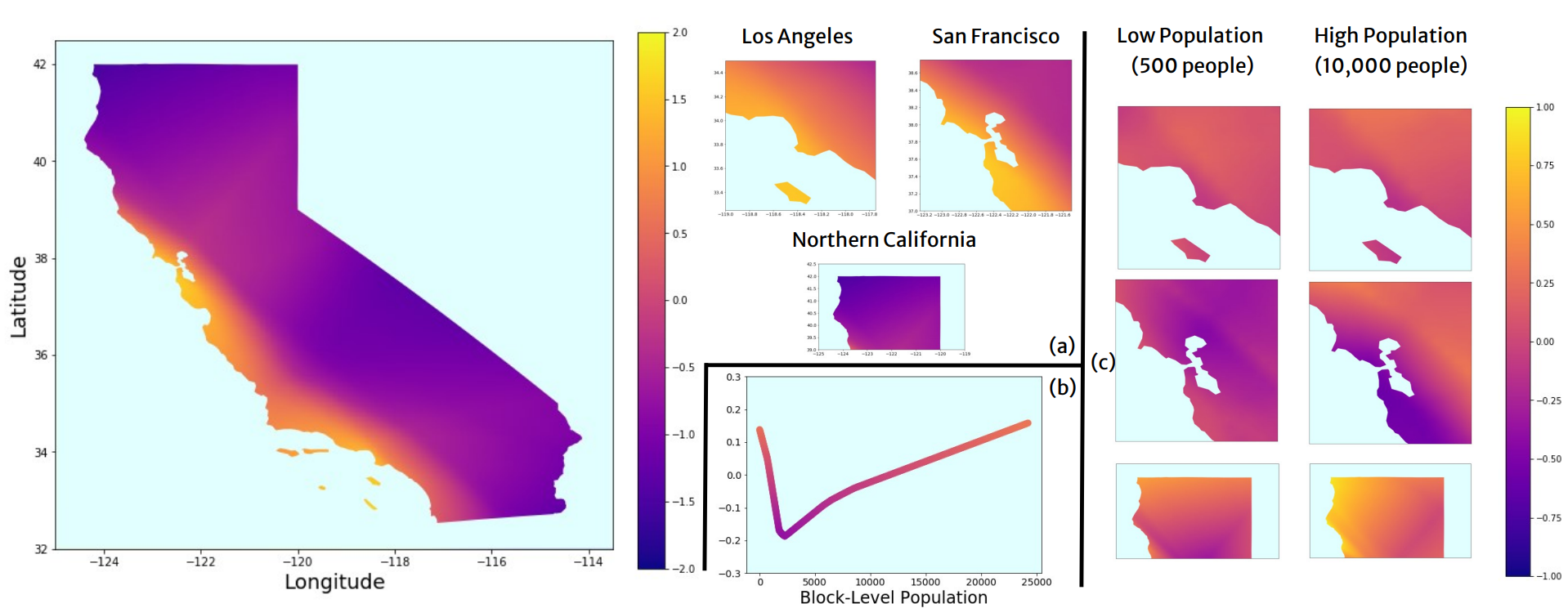}}
    \caption{
    \small
    3D feature interaction for California Housing. 
    (a) The effect of the location only for the entire state of California and zoomed in to three selected locations. (b) The effect of only the block-level population. (c) The three-dimensional interaction effect of block-level population with location.  We plot the effect for both a low and high population to highlight the differences in effect for each location. Note the different scales per panel.
    }
    \label{fig:caliHouse_threeDimensionalVisualization}
\end{figure}

\section{Discussion}
In this section, we further explore the feature interactions learned by SIAN and FIS across multiple datasets.
We provide multiple visualizations of the shape functions learned by SIAN to get a sense of the diverse analysis and interpretations which are made possible by additive models.
Further discussion and graphics are provided in Appendix \ref{app_sec:feature_interaction_selection_details} and \ref{app_sec:mimic_shape_visualizations}.
%{We} provide some visualizations of the shape functions learned across our datasets to get an idea of the kind of interpretations and data exploration which are possible with SIAN.

In Figure \ref{fig:bikeShare_interactionsAndHourXWorkday}a, we visualize the interaction strength measured for each of the singles, pairs, and triples of features from the Bike Sharing dataset.
The four top rated interactions (three of which are visually separated from the main body of the histogram) are, from right to left, [``hour'', ``hour X workday'', ``workday'', ``temperature''].
In \ref{fig:bikeShare_interactionsAndHourXWorkday}b we see the effect of gradually adding feature interactions to our model.
We see there is a steep jump in performance when we are able to model the first pair, depicted in \ref{fig:bikeShare_interactionsAndHourXWorkday}c.
Together with \ref{fig:bikeShare_interactionsAndHourXWorkday}d, these two visualizations alone explain 84\% of the variance in the Bike Sharing dataset.
Importantly, these two trends are interpretable and agree with our intuition about when people are more likely to bike:
on work days, we see peak spikes at 8 a.m. and 5 p.m., corresponding to the beginning and end of work hours; on weekends and holidays, there is a steady demand of bikes throughout the afternoon.
There are also more bikers during warmer temperatures.

In Figure \ref{fig:caliHouse_threeDimensionalVisualization}, we set out to visualize one of the three dimensional interactions which occurs between the three features of latitude, longitude, and population.
In \ref{fig:caliHouse_threeDimensionalVisualization}a, we see the rise in housing price along the coast of California, especially around the metropolitan areas of Los Angeles and San Francisco.
In \ref{fig:caliHouse_threeDimensionalVisualization}b, we see that house price does not monotonically increase with population as we might expect for higher population densities.
A detailed look into the dataset reveals that the `population' feature being used is accumulated at the census block level, creating an inverse relationship with population density as areas like LA and SF are subdivided more than their suburban and rural counterparts.
Although it is possible location and population density might have independent effects on housing price, the dataset nevertheless induces an interaction between location and population.
In \ref{fig:caliHouse_threeDimensionalVisualization}c, we attempt to visualize this 3D interaction by viewing subsamples across two population sizes and over three regions.
We see the network has learned to differentiate the urban regions of LA and SF from the rural coast of northern California, where high population becomes indicative of higher population densities and higher housing prices.
We note that the trends learned by SIAN tends to be very continuous, which is a potential shortcoming in learning fine-grained block-level information in cities like Los Angeles and San Francisco.
In light of these concerns, we run experiments on a three-dimensional extension of the piecewise constant, discontinuous NODE-GAM model in section \ref{sec:not_just_relu}.

\begin{wrapfigure}{R}{0.6\textwidth}
    \centering
    \includegraphics[width=0.29\textwidth]{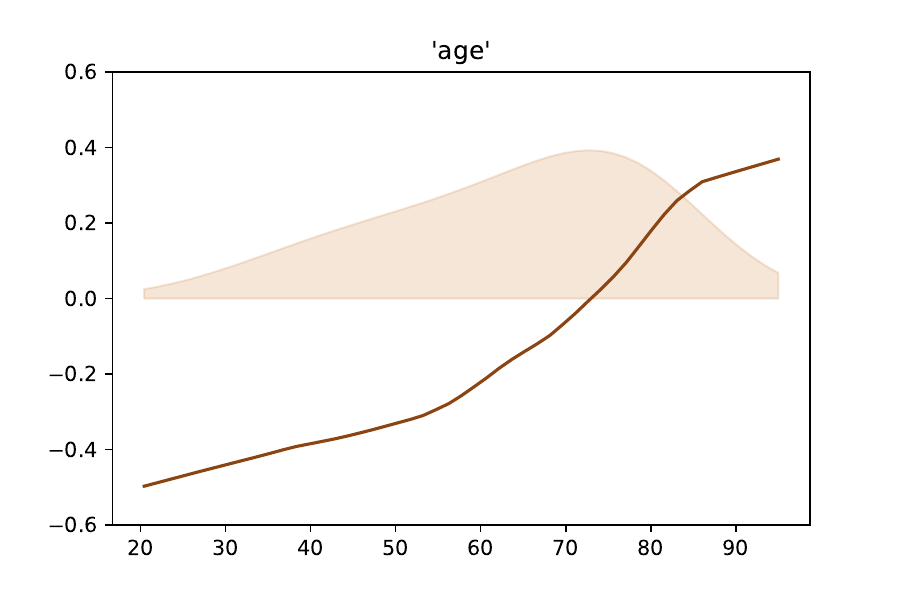}
    \centering
    \includegraphics[width=0.29\textwidth]{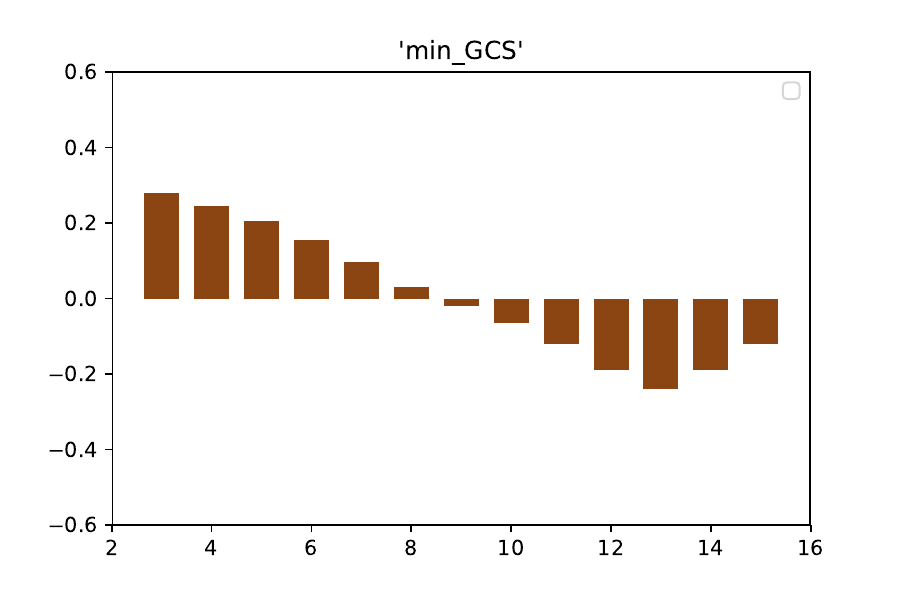}
    %\caption{MIMIC-III: Age Risk and GCS Risk Shape Functions}
    \caption{MIMIC-III. Age Risk and GCS Risk}
    \label{fig:mimiciii_shape_functions}
%\end{figure}
\end{wrapfigure}

In Figure \ref{fig:mimiciii_shape_functions}, we see two example trends from the SIAN trained to predict mortality given hospital data.
On the left, we see that health risk gradually increases with age.
On the right, we see the trend with respect to the Glasgow Coma Scale indicator which is a measure of alertness and consciousness.
%The lowest and highest scores are 3 and 15, corresponding to severe head injury and full alertness.
The lowest and highest scores correspond to severe head injury and full alertness, respectively.
We see that very low GCS corresponds to high risk and that the risk decreases as GCS increases.
At the score of 13, however, we see that an increase to 14 or 15 actually increases mortality risk, defying the intuition that risk should be monotone in the GCS score.
It is highly likely that this dip in mortality risk comes from the special care given to patients with GCS scores below 15, compared to their counterparts with perfect scores.

This trend illustrates a phenomenon occurring throughout machine learning applications in healthcare where correlation and causation are conflated with one another.
%It is highly likely that this dip in mortality risk comes from the special care given to patients with GCS scores below 15, compared to their counterparts with perfect scores.
Similar results have been found linking asthma to a decreased risk of death from pneumonia \cite{caruana2015intelligible}.
These issues might go unnoticed in black-box models whereas interpretable models can uncover and resolve such discrepancies before deployment.
While discovering and correcting interpretable trends brings clear advantages, it is also possible for these trends to be misinterpreted by non-experts as a causal relationship.
Such false causal discoveries can not only lead to physical harm in the domain of medicine, but also to larger social harm in broader AI systems and applications.

\section{Limitations and Future Work}
%\jam{other architectures}
A primary limitation of the current work is its focus on multilayer perceptrons whereas modern state-of-the-art results are dominated by geometric deep learning and transformer architectures.
Extending this procedure to more general architectures is a key direction for bringing interpretability to domains like computer vision and natural language.
Such domains can further customize the FIS algorithm to respect specific structures like spatial locality and knowledge graph semantics.

Another important direction of research includes a better theoretical understanding how the benefits of SIAN scale with the dimensionality of the dataset and the number of available samples, providing yet another lens to study the implicit biases of deep neural networks.
Developing theory which accurately models the empirically observed distributions of feature interactions in real-world data, especially in the presence of heteroscedastic noise and correlated features, is of great interest for this direction.

Multiple experiments confirm that SIAN can produce powerful and interpretable machine learning models which match the performance of state-of-the-art benchmarks like deep neural networks and NODE-GA2M.
Further experiments show that FIS can be applied to more general machine learning algorithms.
Hopefully, future work will be able to further clarify this sparse interaction perspective and help deepen our understanding of the generalization performance of neural networks.

\begin{ack}
We would like to thank the anonymous reviewers for their helpful comments.
We are grateful for support for this work from the National Science Foundation (NSF) under grant CCF-1837131.
%Use unnumbered first level headings for the acknowledgments. All acknowledgments
%go at the end of the paper before the list of references. Moreover, you are required to declare
% funding (financial activities supporting the submitted work) and competing interests (related financial activities outside the submitted work).
\end{ack}

\newpage

%\bibliography{refs}
\bibliography{neurips_2022}

\begin{thebibliography}{10}

\bibitem{abdul2020cognitiveCogam}
A.~Abdul, C.~von~der Weth, M.~Kankanhalli, and B.~Y. Lim.
\newblock Cogam: Measuring and moderating cognitive load in machine learning
  model explanations.
\newblock In {\em Proceedings of the 2020 CHI Conference on Human Factors in
  Computing Systems}, CHI '20, page 1–14, New York, NY, USA, 2020.
  Association for Computing Machinery.

\bibitem{agarwal2020nam}
R.~Agarwal, N.~Frosst, X.~Zhang, R.~Caruana, and G.~E. Hinton.
\newblock Neural additive models: Interpretable machine learning with neural
  nets.
\newblock {\em CoRR}, abs/2004.13912, 2020.

\bibitem{pierre2014exoticParticlesHiggsDataset}
P.~Baldi, P.~Sadowski, and D.~Whiteson.
\newblock Searching for exotic particles in high-energy physics with deep
  learning.
\newblock {\em Nature communications}, 5:4308, 07 2014.

\bibitem{BertinMahieux2011millionSongYearDataset}
T.~Bertin-Mahieux, D.~P. Ellis, B.~Whitman, and P.~Lamere.
\newblock The million song dataset.
\newblock In {\em {Proceedings of the 12th International Conference on Music
  Information Retrieval ({ISMIR} 2011)}}, 2011.

\bibitem{bien2013hierarchicalLasso}
J.~Bien, J.~Taylor, and R.~Tibshirani.
\newblock A lasso for hierarchical interactions.
\newblock {\em Annals of statistics}, 41(3):1111, 2013.

\bibitem{candanedo2017appliancesEnergyDataset}
L.~M. Candanedo, V.~Feldheim, and D.~Deramaix.
\newblock Data driven prediction models of energy use of appliances in a
  low-energy house.
\newblock {\em Energy and Buildings}, 140:81--97, 2017.

\bibitem{caruana2015intelligible}
R.~Caruana, Y.~Lou, J.~Gehrke, P.~Koch, M.~Sturm, and N.~Elhadad.
\newblock Intelligible models for healthcare: Predicting pneumonia risk and
  hospital 30-day readmission.
\newblock In {\em Proceedings of the 21th ACM SIGKDD International Conference
  on Knowledge Discovery and Data Mining}, pages 1721--1730. ACM, 2015.

\bibitem{chang2022nodegam}
C.-H. Chang, R.~Caruana, and A.~Goldenberg.
\newblock {NODE}-{GAM}: Neural generalized additive model for interpretable
  deep learning.
\newblock In {\em International Conference on Learning Representations}, 2022.

\bibitem{chang2021howInterpAndTrustGAMs}
C.-H. Chang, S.~Tan, B.~Lengerich, A.~Goldenberg, and R.~Caruana.
\newblock How interpretable and trustworthy are gams?
\newblock In {\em Proceedings of the 27th ACM SIGKDD Conference on Knowledge
  Discovery \& Data Mining}, KDD '21, page 95–105, New York, NY, USA, 2021.
  Association for Computing Machinery.

\bibitem{chen2016xgboost}
T.~Chen and C.~Guestrin.
\newblock Xgboost: A scalable tree boosting system.
\newblock In {\em Proceedings of the 22nd acm sigkdd international conference
  on knowledge discovery and data mining}, pages 785--794. ACM, 2016.

\bibitem{Chipman1995BayesianVS}
H.~A. Chipman.
\newblock Bayesian variable selection with related predictors.
\newblock {\em Canadian Journal of Statistics-revue Canadienne De Statistique},
  24:17--36, 1995.

\bibitem{cortez2009wineQualityDataset}
P.~Cortez, A.~Cerdeira, F.~Almeida, T.~Matos, and J.~Reis.
\newblock Modeling wine preferences by data mining from physicochemical
  properties.
\newblock {\em Decision Support Systems}, 47:547--553, 11 2009.

\bibitem{dhamdhere2019shapley}
K.~Dhamdhere, A.~Agarwal, and M.~Sundararajan.
\newblock The shapley taylor interaction index.
\newblock {\em arXiv preprint arXiv:1902.05622}, 2019.

\bibitem{2013bikeSharing}
H.~Fanaee-T and J.~Gama.
\newblock Event labeling combining ensemble detectors and background knowledge.
\newblock {\em Progress in Artificial Intelligence}, pages 1--15, 2013.

\bibitem{friedman2008predictive}
J.~H. Friedman and B.~E. Popescu.
\newblock Predictive learning via rule ensembles.
\newblock {\em The Annals of Applied Statistics}, pages 916--954, 2008.

\bibitem{hastie1990originalGAM}
T.~J. Hastie and R.~J. Tibshirani.
\newblock Generalized additive models, 1990.

\bibitem{he2015resnet}
K.~He, X.~Zhang, S.~Ren, and J.~Sun.
\newblock Deep residual learning for image recognition, 2015.

\bibitem{hooker2007functionalANOVA}
G.~Hooker.
\newblock Generalized functional anova diagnostics for high-dimensional
  functions of dependent variables.
\newblock {\em Journal of Computational and Graphical Statistics},
  16(3):709--732, 2007.

\bibitem{janizek2020explaining}
J.~D. Janizek, P.~Sturmfels, and S.-I. Lee.
\newblock Explaining explanations: Axiomatic feature interactions for deep
  networks.
\newblock {\em arXiv preprint arXiv:2002.04138}, 2020.

\bibitem{2016mimicIII}
A.~E. Johnson, T.~J. Pollard, L.~Shen, L.~wei H.~Lehman, M.~Feng, M.~Ghassemi,
  B.~Moody, P.~Szolovits, L.~A. Celi, and R.~G. Mark.
\newblock Mimic-iii, a freely accessible critical care database.
\newblock {\em Scientific Data}, 3(160035), 2016.

\bibitem{kandasamy16salsa}
K.~Kandasamy and Y.~Yu.
\newblock Additive approximations in high dimensional nonparametric regression
  via the salsa.
\newblock In M.~F. Balcan and K.~Q. Weinberger, editors, {\em Proceedings of
  The 33rd International Conference on Machine Learning}, volume~48 of {\em
  Proceedings of Machine Learning Research}, pages 69--78, New York, New York,
  USA, 20--22 Jun 2016. PMLR.

\bibitem{1997caliHousing}
R.~{Kelley Pace} and R.~Barry.
\newblock Sparse spatial autoregressions.
\newblock {\em Statistics \& Probability Letters}, 33(3):291--297, 1997.

\bibitem{1993sapsII}
J.-R. {Le Gall}, S.~Lemeshow, and F.~Saulnier.
\newblock A new simplified acute physiology score (saps ii) based on a
  european/north american multicenter study.
\newblock {\em JAMA}, 270(24), 1993.

\bibitem{lin2006cosso}
Y.~Lin and H.~H. Zhang.
\newblock Component selection and smoothing in multivariate nonparametric
  regression.
\newblock {\em The Annals of Statistics}, 34(5):2272--2297, 2006.

\bibitem{liu2020ssam}
G.~Liu, H.~Chen, and H.~Huang.
\newblock Sparse shrunk additive models.
\newblock In H.~D. III and A.~Singh, editors, {\em Proceedings of the 37th
  International Conference on Machine Learning}, volume 119 of {\em Proceedings
  of Machine Learning Research}, pages 6194--6204. PMLR, 13--18 Jul 2020.

\bibitem{lundberg2020local2global}
S.~M. Lundberg, G.~Erion, H.~Chen, A.~DeGrave, J.~M. Prutkin, B.~Nair, R.~Katz,
  J.~Himmelfarb, N.~Bansal, and S.-I. Lee.
\newblock From local explanations to global understanding with explainable ai
  for trees.
\newblock {\em Nature Machine Intelligence}, 2(1):2522--5839, 2020.

\bibitem{lundberg2020treeSHAP}
S.~M. Lundberg, G.~Erion, H.~Chen, A.~DeGrave, J.~M. Prutkin, B.~Nair, R.~Katz,
  J.~Himmelfarb, N.~Bansal, and S.-I. Lee.
\newblock From local explanations to global understanding with explainable ai
  for trees, 2020.

\bibitem{lundberg2017shapleySHAP}
S.~M. Lundberg and S.-I. Lee.
\newblock A unified approach to interpreting model predictions.
\newblock In I.~Guyon, U.~V. Luxburg, S.~Bengio, H.~Wallach, R.~Fergus,
  S.~Vishwanathan, and R.~Garnett, editors, {\em Advances in Neural Information
  Processing Systems 30}, pages 4765--4774. Curran Associates, Inc., 2017.

\bibitem{meier2009hdam}
L.~Meier, S.~van~de Geer, and P.~Bühlmann.
\newblock {High-dimensional additive modeling}.
\newblock {\em The Annals of Statistics}, 37(6B):3779 -- 3821, 2009.

\bibitem{nori2019interpretml}
H.~Nori, S.~Jenkins, P.~Koch, and R.~Caruana.
\newblock Interpretml: A unified framework for machine learning
  interpretability.
\newblock {\em arXiv preprint arXiv:1909.09223}, 2019.

\bibitem{donnel2008someTopicsBooleanFunctions}
R.~O'Donnell.
\newblock Some topics in analysis of boolean functions.
\newblock In {\em Proceedings of the Fortieth Annual ACM Symposium on Theory of
  Computing}, STOC '08, page 569–578, New York, NY, USA, 2008. Association
  for Computing Machinery.

\bibitem{odonnel2021booleanFunctionBook}
R.~O'Donnell.
\newblock {\em Analysis of Boolean Functions}.
\newblock Cambridge University Press, 2014.

\bibitem{oneill2021regressionNetworks}
L.~O'Neill, S.~Angus, S.~Borgohain, N.~Chmait, and D.~L. Dowe.
\newblock Creating powerful and interpretable models with regression networks,
  2021.

\bibitem{peixoto1987hierarchicalPolynomials}
J.~L. Peixoto.
\newblock Hierarchical variable selection in polynomial regression models.
\newblock {\em The American Statistician}, 41(4):311--313, 1987.

\bibitem{raskhodnikova2012pseudoBooleanKDNF}
S.~Raskhodnikova and G.~Yaroslavtsev.
\newblock Learning pseudo-boolean k -dnf and submodular functions.
\newblock {\em Proceedings of the Annual ACM-SIAM Symposium on Discrete
  Algorithms}, 08 2012.

\bibitem{ravikumar2009spam}
P.~Ravikumar, J.~Lafferty, H.~Liu, and L.~Wasserman.
\newblock Sparse additive models.
\newblock {\em Journal of the Royal Statistical Society. Series B (Statistical
  Methodology)}, 71(5):1009--1030, 2009.

\bibitem{deepmind2017alphaZero}
D.~Silver and et~al.
\newblock Mastering the game of go without human knowledge.
\newblock {\em Nature}, 550:354--359, 2017.

\bibitem{sorokina2008detecting}
D.~Sorokina, R.~Caruana, M.~Riedewald, and D.~Fink.
\newblock Detecting statistical interactions with additive groves of trees.
\newblock In {\em Proceedings of the 25th international conference on Machine
  learning}, pages 1000--1007. ACM, 2008.

\bibitem{stobbe2012sparseFourierSet}
P.~Stobbe and A.~Krause.
\newblock Learning fourier sparse set functions.
\newblock In N.~D. Lawrence and M.~Girolami, editors, {\em Proceedings of the
  Fifteenth International Conference on Artificial Intelligence and
  Statistics}, volume~22 of {\em Proceedings of Machine Learning Research},
  pages 1125--1133, La Palma, Canary Islands, 21--23 Apr 2012. PMLR.

\bibitem{sundararajan2017integratedGradients}
M.~Sundararajan, A.~Taly, and Q.~Yan.
\newblock Axiomatic attribution for deep networks.
\newblock In {\em Proceedings of the 34th International Conference on Machine
  Learning-Volume 70}, pages 3319--3328. JMLR. org, 2017.

\bibitem{tsang2018neural}
M.~Tsang, H.~Liu, S.~Purushotham, P.~Murali, and Y.~Liu.
\newblock Neural interaction transparency (nit): Disentangling learned
  interactions for improved interpretability.
\newblock In {\em Advances in Neural Information Processing Systems}, pages
  5804--5813, 2018.

\bibitem{tsang2020archipelago}
M.~Tsang, S.~Rambhatla, and Y.~Liu.
\newblock How does this interaction affect me? interpretable attribution for
  feature interactions.
\newblock {\em arXiv preprint arXiv:2006.10965}, 2020.

\bibitem{tyagi2016spam2}
H.~Tyagi, A.~Kyrillidis, B.~Gärtner, and A.~Krause.
\newblock Learning sparse additive models with interactions in high dimensions.
\newblock In A.~Gretton and C.~C. Robert, editors, {\em Proceedings of the 19th
  International Conference on Artificial Intelligence and Statistics},
  volume~51 of {\em Proceedings of Machine Learning Research}, pages 111--120,
  Cadiz, Spain, 09--11 May 2016. PMLR.

\bibitem{Vaswani2017transformer}
A.~Vaswani and et~al.
\newblock Attention is all you need.
\newblock {\em CoRR}, abs/1706.03762, 2017.

\bibitem{1997sofa}
J.~L. Vincent, R.~Moreno, J.~Takala, S.~Willatts, A.~D. Mendonça, H.~Bruining,
  C.~K. Reinhart, P.~M. Suter, and L.~G. Thijs.
\newblock The sofa (sepsis-related organ failure assessment) score to describe
  organ dysfunction/failure.
\newblock {\em Intensive Care Med.}, 22(7):707--710, 1997.

\bibitem{wahba1994ssanova}
G.~Wahba, Y.~Wang, C.~Gu, R.~Kleins, and B.~Kle.
\newblock Smoothing spline anova for exponential families.
\newblock In {\em The Annals of Statistics}, 1994.

\bibitem{wang2021partiallyInterpretableEstimatorPIE}
T.~Wang, J.~Yang, Y.~Li, and B.~Wang.
\newblock Partially interpretable estimators (pie): Black-box-refined
  interpretable machine learning, 2021.

\bibitem{xu2022snam}
S.~Xu, Z.~Bu, P.~Chaudhari, and I.~J. Barnett.
\newblock Sparse neural additive model: Interpretable deep learning with
  feature selection via group sparsity, 2022.

\bibitem{yang2020gamiNet}
Z.~Yang, A.~Zhang, and A.~Sudjianto.
\newblock Gami-net: An explainable neural network based on generalized additive
  models with structured interactions, 2020.

\end{thebibliography}
\bibliographystyle{abbrv}
%\bibliographystyle{apalike}

%%%%%%%%%%%%%%%%%%%%%%%%%%%%%%%%%%%%%%%%%%%%%%%%%%%%%%%%%%%%
\newpage
\section*{Checklist}

\begin{comment}
%%% BEGIN INSTRUCTIONS %%%
The checklist follows the references.  Please
read the checklist guidelines carefully for information on how to answer these
questions.  For each question, change the default \answerTODO{} to \answerYes{},
\answerNo{}, or \answerNA{}.  You are strongly encouraged to include a {\bf
justification to your answer}, either by referencing the appropriate section of
your paper or providing a brief inline description.  For example:
\begin{itemize}
  \item Did you include the license to the code and datasets? \answerYes{See Section~\ref{gen_inst}.}
  \item Did you include the license to the code and datasets? \answerNo{The code and the data are proprietary.}
  \item Did you include the license to the code and datasets? \answerNA{}
\end{itemize}
Please do not modify the questions and only use the provided macros for your
answers.  Note that the Checklist section does not count towards the page
limit.  In your paper, please delete this instructions block and only keep the
Checklist section heading above along with the questions/answers below.
%%% END INSTRUCTIONS %%%
\end{comment}

\begin{enumerate}

\item For all authors...
\begin{enumerate}
  \item Do the main claims made in the abstract and introduction accurately reflect the paper's contributions and scope?
    \answerYes{}
  \item Did you describe the limitations of your work?
    \answerYes{}
  \item Did you discuss any potential negative societal impacts of your work?
    \answerYes{}
  \item Have you read the ethics review guidelines and ensured that your paper conforms to them?
    \answerYes{}
\end{enumerate}

\item If you are including theoretical results...
\begin{enumerate}
  \item Did you state the full set of assumptions of all theoretical results?
    \answerNA{}
        \item Did you include complete proofs of all theoretical results?
    \answerNA{}
\end{enumerate}

\item If you ran experiments...
\begin{enumerate}
  \item Did you include the code, data, and instructions needed to reproduce the main experimental results (either in the supplemental material or as a URL)?
    \answerYes{}
  \item Did you specify all the training details (e.g., data splits, hyperparameters, how they were chosen)?
    \answerYes{}
        \item Did you report error bars (e.g., with respect to the random seed after running experiments multiple times)?
    \answerYes{}
        \item Did you include the total amount of compute and the type of resources used (e.g., type of GPUs, internal cluster, or cloud provider)?
    \answerYes{}
\end{enumerate}

\item If you are using existing assets (e.g., code, data, models) or curating/releasing new assets...
\begin{enumerate}
  \item If your work uses existing assets, did you cite the creators?
    \answerYes{}
  \item Did you mention the license of the assets?
    \answerNA{}
  \item Did you include any new assets either in the supplemental material or as a URL?
    \answerYes{}
  \item Did you discuss whether and how consent was obtained from people whose data you're using/curating?
    \answerNA{}
  \item Did you discuss whether the data you are using/curating contains personally identifiable information or offensive content?
    \answerNA{}
\end{enumerate}

\item If you used crowdsourcing or conducted research with human subjects...
\begin{enumerate}
  \item Did you include the full text of instructions given to participants and screenshots, if applicable?
    \answerNA{}
  \item Did you describe any potential participant risks, with links to Institutional Review Board (IRB) approvals, if applicable?
    \answerNA{}
  \item Did you include the estimated hourly wage paid to participants and the total amount spent on participant compensation?
    \answerNA{}
\end{enumerate}

\end{enumerate}

%%%%%%%%%%%%%%%%%%%%%%%%%%%%%%%%%%%%%%%%%%%%%%%%%%%%%%%%%%%%

\appendix

\newpage
\section{Theoretical Discussion}
\label{app_sec:theoretical_discussion}

\subsection{Introduction with pseudo-boolean function}
Let us first discuss a pseudo-boolean function of the form $f : \{-1,1\}^d \to \bbR$.
Every such function has a unique decomposition as a multilinear polynomial:
\[
f(x) = \sum_{I\subseteq[d]} c_I x_I
\]
where $x_I$ is the monomial $x_I = \prod_{i\in I} x_i$.
This multilinear decomposition is usually referred to as the Fourier expansion of the function $f$ and the $c_I$ are called the Fourier coefficients.
Instead of estimating the $2^d$ functional values, we can instead estimate the $2^d$ Fourier coefficients.
%completely uniform, independent/ no corr
Although there is no immediate benefit to this shift in perspective, an abundance of practical assumptions like submodularity, disjunctive normal forms, low-order approximates, and decaying Fourier spectrum can be made to imply that only a sparse subset of the full $2^d$ coefficients are worth estimating.
A further discussion of these assumptions in the context of boolean functions can be found in \cite{donnel2008someTopicsBooleanFunctions,stobbe2012sparseFourierSet,raskhodnikova2012pseudoBooleanKDNF,odonnel2021booleanFunctionBook}.

For our context, it suffices to say that only $s<<2^d$ of the full $c_I$ coefficients are nonzero.
Our problem hence becomes a sparse selection problem in the space of Fourier coefficients, which can be equivalently thought of as a Feature Interaction Selection (FIS) problem in the functional space of pseudo-boolean functions.

\begin{lemma}
Given exact functional values, Archipelago recovers the upper set of Fourier coefficients.
Specifically, given a subset $A\subseteq[d]$ and defining:
\[
\cA_A(f) :=  \bbE_{x,y} \bigg[\bigg(\sum_{C\subseteq A} (-1)^{|C|-|A|} f(x_{C}+y_{C'}) \bigg)^2 \bigg] \cdot \frac{1}{2^{|A|}}
\]
We have that the upper cone set of $A$ is recovered by the Archipelago measure:
\[
\cA_A(f) = \sum_{I\supseteq A} |c_I|^2
\]
\end{lemma}

\begin{proof}
\small
\begin{align*}
& & &
\bbE_{x,y} \bigg[\bigg(\sum_{C\subseteq A} (-1)^{|C|-|A|} f(x_{C}+y_{C'}) \bigg)^2 \bigg]
    \\
&&=\hspace{0.5em}&
\bbE_{x,y} \bigg[\sum_{C_1\subseteq A}\sum_{C_2\subseteq A} (-1)^{|C_1|-|A|} (-1)^{|C_2|-|A|} f(x_{C_1}+y_{C_1'}) f(x_{C_2}+y_{C_2'}) \bigg] &\\
&&=\hspace{0.5em}&
\bbE_{x,y} \bigg[\sum_{C_1\subseteq A}\sum_{C_2\subseteq A} (-1)^{|C_1|+|C_2|} \sum_{I_1\subseteq [d]}\sum_{I_2\subseteq [d]}  c_{I_1} c_{I_2} \cdot (x_{C_1}+y_{C_1'})_{I_1} \cdot (x_{C_2}+y_{C_2'})_{I_2}  \bigg] \\
&&=\hspace{0.5em}&
\bbE_{x,y} \bigg[\sum_{C_1\subseteq A}\sum_{C_2\subseteq A} (-1)^{|C_1|+|C_2|} \sum_{J_1\subseteq A}\sum_{J_2\subseteq A}\sum_{K_1\subseteq A'}\sum_{K_2\subseteq A'}  c_{I_1} c_{I_2} \cdot (x_{C_1}+y_{C_1'})_{I_1} \cdot (x_{C_2}+y_{C_2'})_{I_2}  \bigg] \\
\end{align*}
where $I_1=J_1\cup K_1$ and $I_2=J_2\cup K_2$.
We also write $B_1=A-C_1$ and $B_2=A-C_2$.
Continuing:
\scriptsize
\begin{align*}
&=\hspace{.5em}&
\sum_{K_1\subseteq A'}\sum_{K_2\subseteq A'} 
\bbE_{x,y} \bigg[\sum_{C_1\subseteq A}\sum_{C_2\subseteq A} \sum_{J_1\subseteq A}\sum_{J_2\subseteq A} 
(-1)^{|C_1|+|C_2|} c_{I_1} c_{I_2} \cdot (x_{C_1}+y_{B_1}+y_{A'})_{I_1} \cdot (x_{C_2}+y_{B_2}+y_{A'})_{I_2}  \bigg] \\
&=\hspace{.5em}&
\sum_{K\subseteq A'}
\bbE_{x,y} \bigg[\sum_{C_1\subseteq A}\sum_{C_2\subseteq A} \sum_{J_1\subseteq A}\sum_{J_2\subseteq A} 
(-1)^{|C_1|+|C_2|} c_{I_1} c_{I_2} \cdot (x_{C_1}+y_{B_1}+y_{A'})_{I_1} \cdot (x_{C_2}+y_{B_2}+y_{A'})_{I_2}  \bigg] \\
&+\hspace{.5em}&
\sum_{K_1\neq K_2\subseteq A'}
\bbE_{x,y} \bigg[\sum_{C_1\subseteq A}\sum_{C_2\subseteq A} \sum_{J_1\subseteq A}\sum_{J_2\subseteq A}
(-1)^{|C_1|+|C_2|} c_{I_1} c_{I_2} \cdot (x_{C_1}+y_{B_1}+y_{A'})_{I_1} \cdot (x_{C_2}+y_{B_2}+y_{A'})_{I_2}  \bigg]
\end{align*}
\small %XX todo
We know that the second summand will have every term equal to zero.
Without loss of generality, let $k_1$ belong to $K_1$ but not $K_2$ which we can always find such an element since the two sets are not equal.
We can then factor out an independent term of $y_{k_1}$ and take the expectation over these independent quantities:
\[
\sum_{K_1\neq K_2\subseteq A'}
\bbE_{x,y} \bigg[\sum_{C_1\subseteq A}\sum_{C_2\subseteq A} \sum_{J_1\subseteq A}\sum_{J_2\subseteq A}
(-1)^{|C_1|+|C_2|} c_{I_1} c_{I_2} \cdot (x_{C_1}+y_{B_1}+y_{A'})_{I_1} \cdot (x_{C_2}+y_{B_2}+y_{A'})_{I_2}  \bigg]
=
\]\[
\sum_{K_1\neq K_2\subseteq A'} \bbE_{y_{k_1}} [y_{k_1}] \cdot 
\bbE_{x,y-y_{k_1}} \bigg[ \dots  \bigg]
=
\sum_{K_1\neq K_2\subseteq A'}
0 \cdot \bbE_{x,y-y_{k_1}} \bigg[ \dots  \bigg]
=
0.
\]
Hence, we can now focus on a given $K\subseteq A'$ and consider what happens to:
\[
\bbE_{x,y} \bigg[\sum_{C_1\subseteq A}\sum_{C_2\subseteq A} \sum_{J_1\subseteq A}\sum_{J_2\subseteq A} 
(-1)^{|C_1|+|C_2|} c_{(J_1+K)} c_{(J_2+K)} \cdot (x_{C_1}+y_{B_1})_{J_1} \cdot (x_{C_2}+y_{B_2})_{J_2}  \bigg]
\]
Following a similar argument, we can partition into two summands consisting of $\{J_1=J_2\}$ and $\{J_1\neq J_2\}$.
Without loss of generality, pick one such $j_1$ in $J_1$ but not in $J_2$.
For each choice of $C_1,C_2$, there is a choice of either $x$ or $y$ which can be factored out from one term and not the other, by assumption there is no $j_1\in J_2$.
Without loss of generality, let it be $x$.
\[
\sum_{J_1\neq J_2\subseteq A}
\bbE_{x,y} \bigg[\sum_{C_1\subseteq A}\sum_{C_2\subseteq A} 
(-1)^{|C_1|+|C_2|} c_{I_1} c_{I_2} \cdot (x_{C_1}+y_{B_1})_{J_1} \cdot (x_{C_2}+y_{B_2})_{J_2} \bigg]
=
\]
\[
\sum_{J_1\neq J_2\subseteq A} \bbE_{x_{j_1}} [x_{j_1}] \cdot 
\bbE_{x-x_{j_1},y} \bigg[ \dots  \bigg]
=
\sum_{J_1\neq J_2\subseteq A}
0 \cdot \bbE_{x-x_{j_1},y} \bigg[ \dots  \bigg]
=
0
\]

Now, we are focusing on only a specific $J$ and $K$, giving:
\[
 c_{(J+K)}^2 \cdot
\sum_{C_1\subseteq A}\sum_{C_2\subseteq A}
\bbE_{x,y} \bigg[
(-1)^{|C_1|+|C_2|} (x_{C_1}+y_{B_1})_{J} \cdot (x_{C_2}+y_{B_2})_{J}  \bigg]
\]

Writing the symmetric difference as $(C_1 \Delta C_2) := (C_1-C_2)\cap (C_2-C_1)$, our quantity is equal to:
\[
 c_{(J+K)}^2 \cdot
\sum_{C_1\subseteq A}\sum_{C_2\subseteq A}
(-1)^{|C_1|+|C_2|} \cdot
\bbE_{x,y} \bigg[
 \prod_{j\in(J\cap(C_1 \Delta C_2))} x_j y_j \bigg]
\]

This expectation will be equal to zero if the product is nonempty, following the same logic as before by taking the expectation over one of the independent $x_j$ or $y_j$.
Otherwise, the expectation will be equal to one.
Hence, it is only equal to one for the terms in which $(J\cap(C_1 \Delta C_2))=\emptyset$ and our quantity of interest is equal to:
\[
 c_{(J+K)}^2 \cdot
\sum_{C_1,C_2\subseteq A : (C_1\Delta C_2)\subseteq(A-J)}
(-1)^{|C_1|+|C_2|}
\]
This sum will, however, be equal to zero unless $J=A$.
Suppose instead that there is an $a\in A-J$.
Consider the mapping from pairs $(C_1,C_2)$ where $a\notin C_1$ to the pairs which have $a\in C_1$.
These sets clearly partition $\{(C_1,C_2)\subseteq A\times A : (C_1\Delta C_2)\subseteq(A-J) \}$ and the mapping is bijective inside the subsets.
Hence, the two sets are of the same size and with odd parity of $(-1)^{|C_1|+|C_2|}$.
Together, this implies the sum above will be equal to zero for such cases.
In the alternate case that $J=A$, we know that $(C_1\Delta C_2)=\emptyset$ or $C_1=C_2$, meaning we instead consider $ c_{(J+K)}^2 \cdot \sum_{C\subseteq A} 1 = c_{(J+K)}^2 \cdot 2^{|A|}$.

Altogether, our original quantity of interest has now become:
\[
\sum_{K\subseteq A'} \sum_{J=A} c_{(J+K)}^2 \cdot 2^{|A|}
=
2^{|A|} \cdot \sum_{K\subseteq A'} c_{(A+K)}^2 
=
2^{|A|} \cdot \sum_{I\supseteq A} c_{I}^2 
\]

\end{proof}

\normalsize
\begin{proposition}
The FIS Algorithm \ref{alg:feature_interaction_selection} recovers the downwards closure of the true nonzero entries of the function. %fourier coefficients
Let $S\subseteq \cP([d])$ be the set of nonzero Fourier coefficients of size $|S|=s$.
Define the closure or downwards closure of $S$ as $\Bar{S}:=\{I\subseteq[d] : I\subseteq J \hspace{.5em}\text{for some } J\in S\}$.
Algorithm \ref{alg:feature_interaction_selection} returns exactly $\Bar{S}$.
\end{proposition}
\begin{proof}
Using the lemma above, we know that the measurement on the psuedo-boolean function will be the upper cone.
From this it fairly simply follows that any member of the downwards closure of $S$ will have positive upper cone size and will be added to the candidate set of interactions, whereas nonmembers of the downwards closure will all have zero upper cone size and will not be added.
\end{proof}

\subsection{Decomposition of a general function}
Given a more general space $\cX\subseteq\bbR^d$ and a function $f:\cX\to \bbR$ we can define another decomposition similar to the Fourier decomposition from before.
Given some fixed distribution on our input space $x\sim\cP(\cX)$, we can define the conditional expectation functions as:
\[
\Tilde{f}_\cI (x_\cI) = \bbE_{x_{\setminus\cI}} [ f(x_\cI + x_{\setminus\cI})]
\]
and inductively define $f_\cI(x_\cI)$ with:
\[
\sum_{\cJ\subseteq\cI} f_\cJ(x_\cJ)  = \Tilde{f}_\cI(x_\cI)
\]

We can now decompose our original function $f$ as:
\[
f(x) = \sum_{\cI\subseteq [d]} f_\cI (x_\cI)
\]
where each of these $2^d$ functions corresponds to a subset $\cI$ of the entire feature space which focuses only on these sets of features.
%Our Fourier coefficients have been replaced with functions 

Moreover, we see that these functions are orthogonal with respect to the distribution.
\[
\bbE_{x}\bigg[ \big|f(x) \big|^2 \bigg] = \sum_{\cI\subseteq[d]} \bbE_{x}\bigg[  \big|f_\cI (x_\cI) \big|^2 \bigg]
\]
We can more enticingly write this as:
\[
\|f\|^2 = \sum_{\cI\subseteq[d]} \|f_\cI\|^2
\]

\paragraph{Example} For instance, the function $f_\emptyset$ corresponds to the average value of the function f given the input distribution.
The function $f_{x_1}$ corresponds to the expectation of the function conditioned on knowing variable $x_1$ minus the mean value.
As a running example, let's consider $f(x,y) = (1 + x + xy)$ using the uniform distribution over the square $x,y\sim U(-1,1)^2$.
We will have $\Tilde{f}_\emptyset = 1$, $\Tilde{f}_{x} = 1 + x$,  $\Tilde{f}_{y} = 1$,  and $\Tilde{f}_{x,y} = 1 + x + xy$, so that $f_\emptyset = 1$, $f_{x} = x$,  $f_{y} = 0$,  $f_{x,y} = xy$.
In our example this means that:
\begin{multline*}
\bbE_{x,y}[ (f(x,y))^2 ] = \bbE [ 1^2 ] + \bbE_{x}[ x^2 ] + \bbE_{y}[ 0^2 ] + \bbE_{x,y}[ (xy)^2 ]
  = 1 + \frac{1}{3} + 0 + \frac{1}{9} = \frac{13}{3}
\end{multline*}

\bigskip
\bigskip
Formally, let us assume that our function $f$ is measurable, $L^1$ and $L^2$ integrable with respect to our probability measure, and that additionally  each such subfunction is well defined and similarly integrable for all product measures.
In this case it should be that all of the following expectations (and integrations) that we write are well defined and exchangeable.
Under these assumptions, we have the sum decomposition as just written above.

\begin{proof}
\small
Let us first note that:
\[
\|f\|^2 = \bbE[ \sum_{I_1\subseteq[d]} \sum_{I_2\subseteq[d]} f_{I_1} f_{I_2} ] = \sum_{I_1\subseteq[d]} \sum_{I_2\subseteq[d]}  \bbE[ f_{I_1} f_{I_2} ] 
\]
and consider arbitrary $I_1,I_2\subseteq[d]$.
Let us define 

\noindent
$K=I_1\cap I_2$, $K_1 = I_1\setminus I_2$, $K_2 = I_2\setminus I_1$, $L=I_1^C\cap I_2^C$.

Further, we see that:
\[
 \bbE[ f_{I_1} f_{I_2} ] = \bbE[ \sum_{A_1\subseteq K,B_1\subseteq K_1} \sum_{A_2\subseteq K,B_2\subseteq K_2} (-1)^{\epsilon} \Tilde{f}_{A_1\cup B_1} \Tilde{f}_{A_2\cup B_2} ]
\]
\[
 \text{where} \quad \epsilon = (|I_1|+|I_2|-|A_1|-|B_1|-|A_2|-|B_2|)
 \]
 
Let us further focus on 
\[
 \bbE[ \Tilde{f}_{A_1\cup B_1} \Tilde{f}_{A_2\cup B_2} ] =  \bbE\bigg[  \bbE_{\setminus(A_1\cup B_1)}[f] \cdot \bbE_{\setminus(A_2\cup B_2)}[f] \bigg]
\]

We are now able to take $E=A_1\cap A_2$ and we see that our outside expectation can be reduced to an expectation over $E$ because for each feature not in the set $E$, at least one of the functional terms in the product  is constant with respect to that variable.
In more detail, anything outside of $(A_1\cup B_1)$ will have the first term be constant and anything outside of $(A_2\cup B_2)$ will have the second term be constant.
With respect to any of these variables, we are able to use the linearity of expectation to show that our reduced expectation will have the same value as the original expectation.
Consequently, we may reduce this expectation except for the variables in $(A_1\cup B_1)\cap(A_2\cup B_2) = (A_1\cap A_2) = E$ and we have that our expectation reduces to:
\[
\bbE_E[  \bbE_{\setminus(E)}[f] \cdot \bbE_{\setminus(E)}[f] ] = \bbE_E[  (\bbE_{\setminus E}[f]^2) ]
\]

In greater detail, we have that when we define $A_1'=K\setminus A_1$, $A_2'=K\setminus A_2$, $B_1'=K_1\setminus B_1$, $B_2'=K_2\setminus B_2$, and  $E'=K\setminus E$ that:
\begin{align*}
 \hspace{.5em}&\bbE\bigg[  \bbE_{\setminus(A_1\cup B_1)}[f] \cdot \bbE_{\setminus(A_2\cup B_2)}[f] \bigg]  \\
=\hspace{.5em}&\bbE_{A_1,B_1,A_2,B_2} \bigg[  \bbE_{A_1',B_1',K_2,L}[f] \cdot \bbE_{A_2',K_1,B_2',L}[f] \bigg]  \\
=\hspace{.5em}&\bbE_{E,A_1\setminus E,A_2\setminus E} \bigg[ 
\bbE_{B_1,B_2} \bigg[  \bbE_{A_1',B_1',K_2,L}[f] \cdot \bbE_{A_2',K_1,B_2',L}[f] \bigg] \bigg]  \\
=\hspace{.5em}&\bbE_{E,A_1\setminus E,A_2\setminus E} \bigg[ 
\bbE_{B_1} \bigg[ \bbE_{A_1',B_1',K_2,L}[f] \cdot \bbE_{B_2} \bigg[   \bbE_{A_2',K_1,B_2',L}[f] \bigg] \bigg] \bigg]  \\
=\hspace{.5em}&\bbE_{E,A_1\setminus E,A_2\setminus E} \bigg[ 
\bbE_{B_1} \bigg[ \bbE_{A_1',B_1',K_2,L}[f] \bigg] \cdot   \bbE_{A_2',K_1,K_2,L}[f]  \bigg]  \\
=\hspace{.5em}&\bbE_{E} \bigg[ 
\bbE_{A_1\setminus E,A_2\setminus E} \bigg[ 
\bbE_{A_1',K_1,K_2,L}[f]  \cdot   \bbE_{A_2',K_1,K_2,L}[f]  \bigg] \bigg]  \\
=\hspace{.5em}&\bbE_{E} \bigg[ 
\bbE_{A_1\setminus E} \bigg[ \bbE_{A_1',K_1,K_2,L}[f]  \cdot 
\bbE_{A_2\setminus E} \bigg[ 
  \bbE_{A_2',K_1,K_2,L}[f]  \bigg] \bigg] \bigg]  \\
=\hspace{.5em}&\bbE_{E} \bigg[ 
\bbE_{A_1\setminus E} \bigg[ \bbE_{A_1',K_1,K_2,L}[f]  \bigg] \cdot 
  \bbE_{E',K_1,K_2,L}[f]   \bigg]  \\
=\hspace{.5em}&\bbE_{E} \bigg[  \bbE_{E',K_1,K_2,L}[f] \cdot 
  \bbE_{E',K_1,K_2,L}[f]   \bigg] \\
=\hspace{.5em}&\bbE_{E} \bigg[  \bigg( \bbE_{\setminus E}[f] \bigg)^2   \bigg]
\end{align*}
where each step is done by the linearity of expectation and the law of total expectation.

Since this is true for any $I_1,I_2,A_1,A_2,B_1,B_2$, we now focus on entire summation.
Let us assume that $I_1\neq I_2$ so that we will have that either $K_1$ or $K_2$ is nonempty.
Let's assume without loss of generality that it is $K_1$ which is nonempty.
It then follows that 
\[
\sum_{B_1\subseteq K_1} (-1)^{|K_1|-|B_1|} \bbE\bigg[ \Tilde{f}_{A_1\cup B_1} \Tilde{f}_{A_2\cup B_2} \bigg] =
\]
\[
\sum_{B_1\subseteq K_1} (-1)^{|K_1|-|B_1|} \bbE_E\bigg[ \Tilde{f}_{\setminus E}^2 \bigg]
= 
\]
\[\bbE_E\bigg[ \Tilde{f}_{\setminus E}^2 \bigg] \cdot \sum_{B_1\subseteq K_1} (-1)^{|K_1|-|B_1|}
= \bbE_E\bigg[ \Tilde{f}_{\setminus E}^2 \bigg] \cdot 0 = 0
\]
Altogether, we can go back to our original reformulation to get our desired result:
\[
\|f\|^2 = \bbE[ \sum_{I_1\subseteq[d]} \sum_{I_2\subseteq[d]} f_{I_1} f_{I_2} ] = \sum_{I_1\subseteq[d]} \sum_{I_2\subseteq[d]}  \bbE[ f_{I_1} f_{I_2} ] =
\]
\[
\sum_{I\subseteq[d]}  \bbE[ f_{I} f_{I} ]  +  \sum_{I_1\neq I_2\subseteq[d]} \bbE[ f_{I_1} f_{I_2} ] =
\]
\[
\sum_{I\subseteq[d]}  \bbE[ f_{I}^2 ]  +  \sum_{I_1\neq I_2\subseteq[d]} 0 =
\sum_{I\subseteq[d]}  \bbE[ f_{I}^2 ]  
\]

\end{proof}
%\end{proposition}
\normalsize

Moreover, our model $g$ approximating the function $f$ decomposes in the same way and we divide our estimation problem into $2^d$ different signal estimation problems with:
\[
\|f-g\|^2 = \sum_{\cI\subseteq[d]} \|f_\cI -g_\cI\|^2
\]

In the next two subsections, we will further develop this framework to layout why such a decomposition should succeed by only estimating a subset of all possible interactions.
In particular, we will show that when $|\cI|$ is large, it will become likely that $\|f_\cI\|^2 < \|f_\cI-g_\cI\|^2$ as the effective signal-to-noise ratio shrinks as $|\cI|$ grows.
In other words,
noisy estimates of the signal on average perform even worse than the baseline guess of zero signal.
%This is proved the same (assuming our model is an equally nice function which is true practically.)

\subsection{Higher dimensional signals decay with dimension}
%\subsection{Higher dimensional signals decay with dimension ($\|f_\cI\|^2$ shrinks)}
For this section we will focus on a Fourier signal $f$ on the $d$-dimensional rescaled torus $[-1,1]^d$.

\begin{proposition}
If we assume that our signal is in the class $C^k$ of $k$-differentiable smooth functions, then we have that our Fourier coefficients decay like $(\prod_{j=1}^n \frac{1}{m_j})^k$.
This is a known fact from Fourier analysis and can be easily derived using a Fourier series.
\end{proposition}

Consequently, we will consider a random signal defined as follows:
Draw coefficients $c_{\overrightarrow{m}}$ and $d_{\overrightarrow{m}}$ from $N(0,\sigma^2)$ independently for all ${\overrightarrow{m}}\in\bbN_0^d$.
Rescale these coefficients according to k-smooth decay as $a_{\overrightarrow{m}} := c_{\overrightarrow{m}} \cdot (\prod_{j:m_j\neq0} \frac{1}{m_j})^k$ and $b_{\overrightarrow{m}} := d_{\overrightarrow{m}} \cdot (\prod_{j:m_j\neq0} \frac{1}{m_j})^k$. 
Define our function as the Fourier series:
\[f({\overrightarrow{x}}) = \sum_{{\overrightarrow{m}}\in\bbN_0^d} a_{\overrightarrow{m}} \cos(2\pi \overrightarrow{m}\cdot\overrightarrow{x}) +
 \sum_{{\overrightarrow{m}}\in\bbN_0^d} b_{\overrightarrow{m}} \sin(2\pi \overrightarrow{m}\cdot\overrightarrow{x})\]
We know that with probability one this will correspond to a k-smooth function on $[-1,1]^n$.
Let us assume that our $\overrightarrow{x}$ are uniformly distributed on the cube/ torus.
%xxx is this obvious, im thinking it might not be

\begin{proposition}
When we consider all interactions of size $K$ ($\bbI_K := \{\cI\subseteq[d] : |\cI|=K\} )$, the norm of our function scales like:
\[
\sum_{\cI\in\bbI_K} \|f_\cI\|^2 = \binom{d}{K} \frac{(\zeta(2k)/2)^K}{(\zeta(2k)/2+1)^d}
\]
\end{proposition}

\begin{proof}
Let us first introduce the 0-norm $\|{\overrightarrow{m}}\|_0$ which denotes how many nonzero entries the index ${\overrightarrow{m}}$.
This will be useful because while our basis functions are orthogonal we are now considering a probabilistic space instead of a normally scaled $L^2$ space meaning we will be integrating against the uniform measure $dx/2$
We know that all of our basis functions are orthogonal in this space and so by Parseval's identity we see that:
\[
\bbE[ f^2 ] = (a_0^2+2a_0b_0+b_0^2) +  \sum_{0\neq{\overrightarrow{m}}\in\bbN_0^d} a_{\overrightarrow{m}}^2\cdot2^{-\|m\|_0}
\]
\[+ \sum_{0\neq{\overrightarrow{m}}\in\bbN_0^d} b_{\overrightarrow{m}}^2\cdot2^{-\|m\|_0}
\]
Moreover, due to direct calculation or usage of our proposition from our previous section, we can calculate the value of the interactions of feature size $K$.
First, let's calculate for $K=0$.
We see that $\bbE[(a_0^2+2a_0b_0+b_0^2)] = 2\sigma^2$
For $K>0$, we can see that 
\[
\sum_{\cI\in\bbI_K} \|f_\cI\|^2 = 2^{-K}
\bbE \bigg[
\sum_{{\overrightarrow{m}}\in\bbN_0^d : \|{\overrightarrow{m}}\|_0 = K}
a_{\overrightarrow{m}}^2 +
b_{\overrightarrow{m}}^2
\bigg]
\]
We can also see that there are $\binom{d}{K}$ regions where $\|{\overrightarrow{m}}\|_0 = K$ will hold corresponding to the $\binom{d}{K}$ feature subsets of size $K$.
We will briefly calculate this sum for only one of such subsets, namely $I=\{1,\dots,K\}\in\bbI_K$.
\[
\sum_{{\overrightarrow{m}}\in\bbN_0^K}
\bbE \bigg[
a_{\overrightarrow{m}}^2 +
b_{\overrightarrow{m}}^2
\bigg]
=
\sum_{m_1,\dots,m_K=1}^\infty \frac{2\sigma^2}{m_1^{2k} \cdot\dots\cdot m_K^{2k}}
\]
Now consider the Riemann-Zeta function defined as $\zeta(a) = \sum_{j=1}^\infty \frac{1}{j^a}$ we see that our summation can be taken over each index as the Riemann-Zeta function and we get $2\sigma^2 (\frac{\zeta(2k)}{2})^K$.
Combining this with the fact that we know there are  $\binom{d}{K}$ such sums we get that the overall contribution is indeed $\binom{d}{K} 2\sigma^2 (\frac{\zeta(2k)}{2})^K$.
%We can now pick our $\sigma^2$ so that our overall function is normalized.
%This is not particularly important but is useful for visualization.
Setting $\|f\|=1$, we get that $\frac{1}{\sigma^2} = 2(\frac{\zeta(2k)}{2}+1)^d$ and we are done.
\end{proof}

Consequently, we find that even for a very general signal, we still ultimately have a decay in the signal.
The signal we consider is a completely generalized $C^k$ Fourier signal and still decays to having zero strength for some $K$ large enough.
This is the first part of our argument and one could equally assume that the higher dimensional interactions are actually sparse or mostly zero.
This is due to the structure we would expect out of real-world data unlike a completely arbitrary signal.
One should further note that only a measure zero subset of $C^k$ functions are $C^{k+1}$ and so even when we assume $k$-times differentiable we implicitly consider a signal that it is not $(k+1)$ times differentiable.
Thus, we have shown there is eventually decay for a very arbitrary signal.

\subsection{Higher dimensional signals are intrinsically more difficult to estimate}
%\subsection{Higher dimensional signals are intrinsically more difficult to estimate ($\|f_\cI-g_\cI\|^2$ grows)}

We should imagine that the amount of data we need to estimate what happens in n-dimensional space should be something like $r^n$ where r is the `resolution' of the space that we require.
For instance, if we use histograms, we could imagine that we only need $2^n$ different bins in a binary space but more like $10^n$ or higher bins in a continuous space.

We do some calculations that show in the `high resolution' regime where we have enough resolution to be able to sufficiently capture the details of our function, we will ultimately need samples scaling like $\frac{r^n}{SNR}$.

Consider a signal with additive noise of the form:
\[
y = \mu\cdot\exp[2\pi im(x_1 + \dots + x_n)] + \sigma\cdot\epsilon
\]
where we define this function with $x$ uniformly distributed on the n-dimensional torus $\bbT^n$ and $\epsilon\sim N(0,1)$.
Note that this is a complex-valued function and we will use the notation $\Bar{y}$ to denote the complex conjugate.

If we consider a histogram estimation technique which divides $\bbT^n$ into $r^n$ bins corresponding to a uniform binning $\{0,\frac{1}{r},\dots,\frac{r-1}{r},1\}$.
We can then calculate exactly the `value' of using our sample as part of our histogram.
%(in terms of the MSE)
% (as compared to not estimating the signal and assuming zero.)
Moreover, if we imagine each bin having $N$ samples we can estimate the overall impact of the N-sample estimate.

%**calculations**

%We can see that this `value' of estimating the signal is strictly positive whenever we have 
Using our sample is beneficial whenever:
\[
(1+\frac{\sigma^2}{\mu^2}) < (N+1)\cdot c(r,m,n)
\]
where our constant can be written 
\[
c(r,m,n) = \left(\frac{\sin(\pi \frac{m}{r})}{\pi \frac{m}{r}}\right)^{2n}
\]

\begin{proof}
\small
Let us take our histogram estimator to have resolution $r\in\bbN$ and with bins having endpoints $0=a_0 < b_0 = \frac{1}{r} = a_1 < b_1 = \frac{2}{r} = a_2 < \dots < \frac{r}{r} = b_r$.
Our histogram estimator will use the estimate defined in each bin as $\Hat{f} = \frac{1}{N}\sum_{j=1}^N (\mu\cdot\exp[2\pi im\cdot x_j] + \sigma\cdot\epsilon_j)$
We will break our estimate of the mean-squared error of our function approximation over each bin of the histogram.
We will then have our approximation error given by:
\[
%\bbE \bigg[
\int_{a^{(1)}}^{b^{(1)}} \dots \int_{a^{(n)}}^{b^{(n)}} \|f-\Hat{f}\|^2 =
%\]
%\[
\int_{a^{(1)}}^{b^{(1)}} \dots \int_{a^{(n)}}^{b^{(n)}} [f\Bar{f}-2f\Bar{\Hat{f}}+\Hat{f}\Bar{\Hat{f}}] =
\]
\[
\|f\|^2-2\bbE[f]\bbE\bigg[\Bar{f}+\frac{1}{N}\sum_{j=1}^N\epsilon_j\bigg] + \bbE\bigg[ \frac{1}{N^2}\sum_{j=1}^N\sum_{k=1}^N
%\]
%\[
 (\mu\cdot\exp[2\pi im\cdot x^{(j)}] + \sigma\cdot\epsilon_j) (\mu\cdot\exp[-2\pi im\cdot x^{(k)}] + \sigma\cdot\epsilon_k) \bigg] =
\]
\[
\|f\|^2-2\bbE[f]\bbE[\Bar{f}] + \frac{\sigma^2}{N^2} \bbE[\sum_j \epsilon_j^2] +
%\]
%\[
\frac{\mu^2}{N^2} \sum_j \bbE[ \exp[2\pi i m\cdot (x^{(j)}-x^{(j)})] ] +
\] \[
\frac{\mu^2}{N^2} \sum_{j\neq k} \bbE[  \exp[2\pi i m\cdot (x^{(j)}-x^{(k)})] ] = 
\]
\[
\|f\|^2-2\mu^2\bbE[f/\mu]\bbE[\Bar{f}/\mu] + \frac{\sigma^2}{N} +
%\]
%\[
\frac{\mu^2}{N^2} (N\cdot 1)
+
\frac{\mu^2}{N^2} ((N^2-N)\cdot \bbE[f/\mu]\bbE[\Bar{f}/\mu] = 
\]

In order to be a good estimate (better than zero), we need that everything but the first term is smaller than zero, because this will mean we are decreasing from the baseline of $\|f\|^2$.
Hence, we desire that 
\[
-2Nc + \frac{\sigma^2}{\mu^2} +
1 + (N-1)\cdot c < 0
\]
where we define $c$ as the quantity $\bbE[f/\mu]\bbE[\Bar{f}/\mu]$.

Hence, we need that:
\[
(N+1)\cdot c > \frac{1}{SNR} + 1
\]

We now calculate $c$.
\[
\bbE[f/\mu] = 
\int_{a^{(1)}}^{b^{(1)}} \dots \int_{a^{(n)}}^{b^{(n)}} \exp[2\pi im\cdot x] =
%\]
%\[
\prod_{j=1}^n \frac{\exp[2\pi imb]-\exp[2\pi ima]}{2\pi im(b-a)}
\]
So,
\[
\bbE[f/\mu]\bbE[\Bar{f}/\mu] = 
%\]
%\[
%\hspace*{-1cm}
\prod_{j=1}^n (\frac{e^{[2\pi imb]}-e^{[2\pi ima]}}{2\pi im(b-a)}) \cdot (\frac{e^{[-2\pi imb]}-e^{[-2\pi ima]}}{-2\pi im(b-a)})
%\]
%\[
= \frac{(-1)^d(-1)^d}{(2\pi)^{2n}}\cdot
\]
And so then because no matter our choice of $a,b$ we always have that $b-a=\frac{1}{r}$( by our choice of uniform spacing) it follows that:
\[ \prod_{j=1}^n \frac{e^{[2\pi im(b-b)]}-e^{[2\pi im(b-a)]}-e^{[2\pi im(a-b)]}-e^{[2\pi im(a-a)]}}{{m(b-a)}\cdot m(b-a)}
\]
\[
= \frac{r^{2n}}{(2\pi)^{2n}m^{2n}}\cdot
\prod_{j=1}^n (1-\cos(2\pi m/r)-\cos(2\pi m/r)+1)
\]
\[
=(\frac{r}{2\pi m})^{2n}\cdot
(2(1-\cos(2\pi m/r))^n
%\]
%\[
=(\frac{r}{2\pi m})^{2n}\cdot
(2(2\sin(\pi m/r)^2))^n
%\]
%\[
=\bigg(\frac{\sin[\pi \frac{m}{r}]}{\pi \frac{m}{r}}\bigg)^{2n}
\]

Altogether, this yields:
\[
1+ \frac{1}{SNR} < (N+1)\cdot \bigg(\frac{\sin[\pi \frac{m}{r}]}{\pi \frac{m}{r}}\bigg)^{2n}
\]
\end{proof}
\normalsize

We can consider the regime of `high resolution' where we have that $r>>m$ meaning that the complexity of the histogram estimator (number of bins) is much higher than the signal complexity (frequency of oscillation).
In this regime, we can see that $c(r,m,n) \approx 1$ meaning our equation reduces to :
\[
(1+\frac{\sigma^2}{\mu^2}) < (N+1)
\]
If we further note that our N is a bin-wise sample size we will note that we actually need $N' = N r^n$ total samples for our estimator and that $\frac{\sigma^2}{\mu^2}$ is just the signal to noise ratio.
Finally, we get that $N' >  \frac{1}{SNR}\cdot r^n$ is the requirement to get an improvement over alternatively assuming the signal is uniformly zero.
At a fixed level of noise, this means that we need samples of the order $\gtrsim r^n$.

We should additionally consider, as we connect this argument with the section above, that we have previously implied that the strength of the signal $\mu$ also shrinks with growing dimension, hence raising the requirement of sample size $N'$ even higher as we increase the dimension $K$.

\subsection{Summary of Theory}
In conclusion, we find multiple strands of theory which support that estimating the full functional form is both unnecessary and at risk of overfitting.
Altogether, we have that $\|f_\cI\|^2$ shrinks as $|\cI|$ grows and that $\|f_\cI-g_\cI\|^2$ grows as $|\cI|$ grows.
Ultimately, we expect there to be an intersection between these for most datasets.
Consequently, there will be an optimal cut-off point where higher-order feature interactions should be left unconsidered.

In practice, we find that it is moreover beneficial to cut out not only those signals which are too small because they are too high dimensional, but also those signals which we estimate to be small from the available data.
Although touched upon in the first section for psuedo-boolean functions, a more fleshed out argument for continuous functions which follow a sparsity assumption of the type in this paper, leveraging work in high-dimensional statistics, would be a worthwhile direction.
Additionally, a more thorough investigation of the distribution of feature interactions under theoretical setups beyond the Gaussian-Fourier setup we consider could be useful to gain further intuition about the empirical distribution of interactions found in real-world data.
Such results could further enhance the practicality of the assumptions made in high-dimensional statistics arguments for sparse feature interactions.

\section{Discussion on GAMs}

Generalized Additive Models (GAMs) have existed for decades since their conception in 1990 \cite{hastie1990originalGAM}.
Originally formulated with only univariate functions and fit using splines, they were soon extended to `Smoothing-Spline ANOVA' which used pairs of features to fit bivariate splines, being able to represent even more complex functions \cite{wahba1994ssanova}.
As the years have passed, the functional model being used to fit the nonlinear shape functions has changed alongside popular ML techniques of the era: random forests, boosting machines, kernel methods, and recently neural networks.
In fact, it has only been in the last few years where neural additive models have grown in popularity as an interpretable alternative to deep neural networks \cite{agarwal2020nam,chang2022nodegam,yang2020gamiNet}.
Despite consistent usage over the years, the application of GAMs to fit trivariate functions has been extremely scant.
Besides scattered works fitting trivariate GAMs on very small-scale synthetic data, only SALSA \cite{kandasamy16salsa} seems to evaluate higher-order GAMs across multiple real-world datasets.
Two key obstructions which have prevented higher-order GAMs from becoming popular are: (1) three-dimensional functions are fundamentally more challenging to visualize and interpret; (2) cubic and higher scaling of the number of interactions is prohibitively expensive to train.

In this work, we argue that higher-order additive models have greater flexibility than existing additive models while having greater interpretability than completely black-box approaches.
We achieve the ability to model higher-order interactions by phrasing the model's interactions as a sparse selection problem.
In the following two sections, we address the major concerns of additive models through a discussion on sparsity and a discussion on feature correlation.

%\subsection{Higher-Order Additive Models}

\subsection{Sparse Additive Models}
Work in the late '90s and early '00s brought work on high-dimensional statistics like LASSO and LARS into the limelight.
With the increasing popularity of sparse linear regression methods also came a body of work focusing on sparse nonlinear additive regression leveraging the same feature shrinkage techniques \cite{lin2006cosso,ravikumar2009spam,meier2009hdam}.
There have also been methods designed for sparse high-dimensional regression with bivariate models in mind \cite{tyagi2016spam2,liu2020ssam}.
Even recently, new methods have explored sparse regression for neural additive models, using the same principles of shrinkage estimation and regularization \cite{xu2022snam,yang2020gamiNet}.

We highlight two key differences between the existing work in sparsity for additive models and our work in sparsity.
First, we do not make the typical assumption of high-dimensional inference that the number of features ($d$) is much greater than the number of samples.
In fact, the highest number of features in the datasets we explore is only 90.
Rather, we make the assumption that the number of feature interactions ($2^d$) is much greater than the number of samples.
In this light, we are doing sparse regression in the interaction space rather than the feature space.
Second, we do not depend on shrinkage estimation and instead leverage heredity to provide sparse solutions.
All existing work in sparse additive models depends on shrinking the estimates of unimportant features through regularization.
In our work, we find that this is an infeasible approach for dealing with the exponentially growing number of higher-order feature interactions.
For this reason, we leverage the heuristics given by a partially trained neural network to guide our sparse selection algorithm.
We find that in conjunction with the heredity assumption, we are able to vastly reduce the search space and effectively complete the sparse selection problem.

Let us also briefly discuss the additional implications of sparse selection.
Fitting a model to fewer learned shape functions actually increases the interpretability of the model by needing to communicate fewer shape functions.
Some existing work has focused explicitly on the cognitive load of additive models by reducing the number of shape functions which must be communicated \cite{abdul2020cognitiveCogam}.
Other work \cite{chang2021howInterpAndTrustGAMs} notes that encouraging sparse solutions can have imbalanced accuracies across small subpopulations, possibly leading to unfair predictions.
Many of these properties are a consequence of the underlying feature distribution, whereas GAMs are most effective in the presence of totally independent features.
%x (moral/ethical)
%x cognitive load
%x Sparse Biased (caruana)

%\subsection{Note on Sparsity}
%-- Caruana says sparsity can be biased (e.g. compas race, bikeshare holiday)
%-- also references work which says some things can be 'more true' but 'less predictive' which isn't focused on in this work

%--key issues for additive models are the underlying correlations of the data generation of X; the outlier points which are worse to estimate from a MSE standpoint but maybe 'more true'; (but our framework handles better the cases where there is high noise or low sample size) 
%\subsection{Note on cognitive load}
%Clearly this is related to sparsity as the more intepretable things take less congitive load to be able to understand.
%Attributing these to sepcific vairables is exactly sparsity,
%but in the context of something like fairness we more importantly want to distinguihs certain variables from their proxies and how this can lead to unfairness.

\subsection{Feature Correlation}
When interpreting additive models, it is important to consider the underlying distribution of the features $x$.
We focus on two simple but extreme examples to illustrate these points.

First, consider the function $f(x_1,x_2) = x_1 x_2$ where $x_1,x_2 \sim U(-1,1)$ but assume that they are completely dependent on each other: $x_1=x_2$.
Although the full 2D function is the XOR function on the domain, conditioning on either $x_1$ or $x_2$ immediately gives the value of the other.
Accordingly, each of the 1D conditional expectations are $\Bar{f}_1(x_1) = x_1^2$ and $\Bar{f}_2(x_2) = x_2^2$.
These univariate trends are then subtracted from the bivariate conditional expectation after normalization/ purification.
Consequently, on the domain of $x_1=x_2$ each of the univariate functions looks like $x_1 x_2$ but the bivariate function actually looks like $-x_1 x_2$, the opposite of what it truly is.
Although real-world data does not have such extreme correlations, similar phenomena can still make interpretation difficult.

Second, consider the function $f(x)$ which is $+1$ when $x = c_1$ and is $-1$ when $x = c_2$.
Further, suppose that $x$ is equal to $c_1$ for 99\% of the time and is $c_2$ only 1\% of the time.
It follows that the average value of $f$ is $0.98$.
The purified univariate value then becomes $+0.02$ at $x=c_1$ and $-1.98$ at $x=c_2$.
This looks like having a value of $x=c_2$ is 99 times more potent than having a value of $x=1$.
Accordingly, for visualization purposes, it might be more natural to normalize by the `counterfactual' distribution that both $x=c_1$ and $x=c_2$ are equally likely.

Tackling these challenges behind feature correlations could be one of the key technical limitations of additive models.
For example, on computer vision datasets, it becomes unreasonable to assume that one pixel is independent of its neighbor.
Similarly in natural language, certain words have greater co-occurrence which bring structure to language.

%x correlated features and sparsity can be weird interpretations
%x images are naturally very correlated features

\newpage
\section{Feature Interaction Selection}
We introduce an aggregation procedure using Archipelago in its original format using zero as a baseline as well as a contrastive version. We turn a local explanation into a global explanation by using our aggregation procedure.
We use a feature interaction detection algorithm alongside heredity, further details are available in Algorithm \ref{alg:feature_interaction_selection}.

\label{app_sec:feature_interaction_selection_details}

\subsection{Existing Feature Interaction Detection (FID) Techniques} 
\paragraph{Archipelago} 
Archipelago estimates the feature interaction strength by approximating the Hessian of the model given a target data instance $x^*$ to be explained and a baseline data instance $x'$ to be compared against.
The secant approximation of the Hessian which is used by Archipelago is defined:
{
\small
\begin{multline}
    \omega_{i,j}(x) = \Big( \tfrac{1}{h_i h_j} (f({x}^{\star}_{\{i,j\}} +  {x}_{\setminus{\{i,j\}}}) - f({x}'_{\{i\}} + {x}^{\star}_{\{j\}} +  {x}_{\setminus{\{i,j\}}})  \\
    - f({x}^{\star}_{\{i\}} + {x}'_{\{j\}} +  {x}_{\setminus{\{i,j\}}}) + f({x}'_{\{i,j\}} +   {x}_{\setminus{\{i,j\}}})) \Big)^2
\end{multline}
\normalsize
}

\noindent
%\newline
where $h_i=x_i^*-x'_i$, $h_j=x_j^*-x'_j$. 
We further approximate the true expectation over all contexts with the two point average used in \cite{tsang2020archipelago}: 
$\Bar{\omega}_{i,j} := \frac{1}{2}(\omega_{i,j}(x^*)+\omega_{i,j}(x'))$.
This generates a score $\Bar{\omega}_{i,j}$ for every possible feature pair $\{(i,j):i\neq j\in[d]\}$.
Higher degree interactions are calculated in the same way using a secant approximation of the higher-order derivative.

\paragraph{Integrated Hessians} The integrated gradient \cite{sundararajan2017integratedGradients} attribution to the feature $i$ given a target feature $x^*$ and a baseline feature $x'$ can be written as:
\[
\phi_i(x^*) := (x_i^* - x'_i) \int_0^1 \frac{\partial f(x'-\alpha(x^*-x'))}{\partial x_i} d\alpha
\]
This definition can be extended to two-dimensional interactions between features $i$ and $j$ and the integrated hessian \cite{janizek2020explaining} attribution is defined as:
\[
\phi_{i,j}(x^*) := (x_i^* - x'_i) (x_i^* - x'_i) \int_{\beta=0}^1 \int_{\alpha=0}^1 \alpha\beta \frac{\partial^2 f(x' + \alpha\beta(x^*-x'))}{\partial x_i \partial x_j} d\alpha d\beta
\]
While this method is able to give accurate feature attributions for differentiable models, the estimation of the integral makes it prohibitively slow, even for calculation of interaction pairs.
Moreover, running this method on higher-order interactions is nearly impossible with current implementations.

\paragraph{SHAP} SHapley Additive exPlanations \cite{lundberg2017shapleySHAP} are of growing popularity because of their nice theoretical properites and applicability to any black-box model.
They are defined by the following equation:
\[
\phi_i(v) = \sum_{S\subseteq([d]\setminus{i})} \frac{|S|!(d-|S|-1)!}{d!} \bigg(v(S\cup{j}) - v(S)\bigg)
\]
Where $v:\cP([d])\to\bbR$ is the value function, typically defined as $v(S):=\bbE[f(X) | X_S = x^*_S]$.

This definition was originally inspired by game theory, attributing how much of the value $v$ of a coalition $S$ was created by each player $i$.
Further extensions to interactions also exist like the Shapley-Taylor interaction index \cite{dhamdhere2019shapley}.
Unfortunately, estimating the sum in the definition of SHAP values is usually very expensive and fast implementations currently only exist for tree-based approaches \cite{lundberg2020treeSHAP}.

\subsection{Beyond Archipelago and ReLU}
As mentioned in the previous section, there are other existing methods for feature interaction attribution, namely Taylor SHAP and integrated hessians.
Unfortunately, both of these methods perform poorly on the main architecture of choice in this work: ReLU neural networks.
There are currently no fast implementations of SHAP outside of the tree-based approaches and Integrated Hessians requires the second partial derivative to exist, which is not the case for piecewise linear ReLU networks.
In theory, however, both of these techniques could be substituted into our Algorithm \ref{alg:feature_interaction_selection}.
In this way, one could imagine using (random forests + SHAP) or (tanh networks + integrated hessians), so long as the implementations are not prohibitively expensive.

In this direction, we pursue one such novel combination in NODE-GAM + Archipelago.
In Section \ref{sec:not_just_relu}, we use Archipelago on trained NODE models to select a subset of feature triples using Algorithm \ref{alg:feature_interaction_selection}.
We next train an adjusted implementation NODE-GA3M which is able to focus on the feature triples we prescribe.
As mentioned in the main body, this technique allows us to push state of the art on the Housing  dataset and demonstrates the capability of our pipeline as a more general approach.

\subsection{Further Visualization}
%XX want to add node-ga3m viz

\begin{figure}[h]
    \centering
    \includegraphics[width=.95\columnwidth]{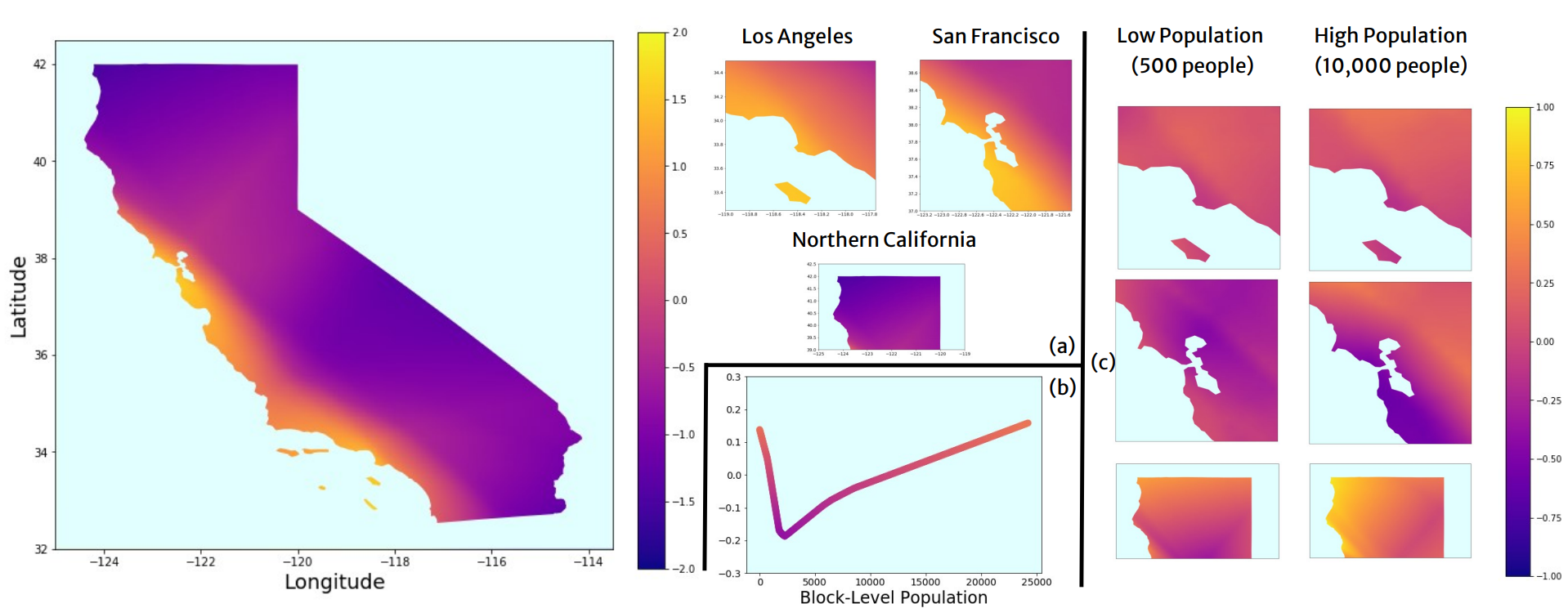}
    \caption{Reproduction of Figure \ref{fig:caliHouse_threeDimensionalVisualization}. Visualization of the 3D feature interaction for California Housing.}
    \label{app_fig:caliHouse_threeDimensionalVisualization}
\end{figure}

One might be interested in knowing how the interpretability of SIAN compares with a regular DNN using post-hoc feature attribution.
In an effort to illustrate the differences between the global interpretability of additive models and the local interpretability of feature attribution methods, we show the analogous interpretations coming from post-hoc attribution.

\begin{figure}[h]
    \centering
    \includegraphics[width=.95\columnwidth]{{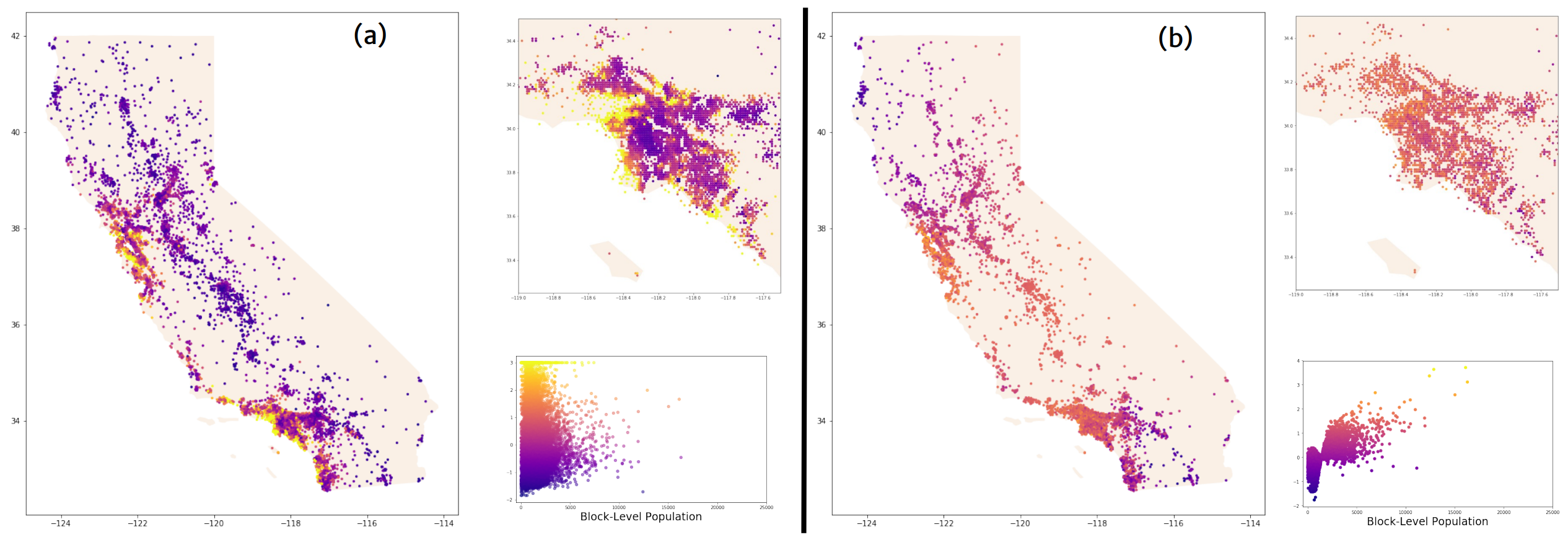}}
    \caption{Feature Attribution Visualization of two feature interaction for California Housing. (a) Directly visualizing the dataset. (b) Visualizing the Archipelago feature attributions.}
    \label{app_fig:caliHouse_postHocVisualization}
\end{figure}

The main qualitative difference comes from the fact that these visualizations are using scatter plots of the training samples rather than the heatmap used in the original visualization.
In Figure \ref{app_fig:caliHouse_postHocVisualization}a, it is clear that there is too much noise to discern any serious trends from location or population, considering it is possible that other available features are creating the visible differences in the map.
In Figure \ref{app_fig:caliHouse_postHocVisualization}b, we can see a somewhat similar trend to the one in Figure \ref{app_fig:caliHouse_threeDimensionalVisualization}a.
There is still some noise in the feature attribution as we can see in the panel zoomed into Los Angeles.
In the final plot of block-level population, we can see a more clear trend than the original data-level plot, however, ultimately it is still very difficult to discern a meaningful trend from this plot.
Further, visualizing a noisy scatter plot like this means we have very minimal insight into how the model would interpolate or extrapolate into new regions.

%experiments showing we can get a 3D node-gam to outperform a 2D Node-gam.

%Particularly exciting if we can use fewer shape functions :)

\twocolumn
\section{Visualizing MIMIC-III Shape Functions}
\label{app_sec:mimic_shape_visualizations}
%maybe i dont even need this here, maybe the two shape functions are enough for the reader

\label{sec:append_mimic_viz}
In this section we fully depict all 45 shape functions required to make a prediction on a patient from the MIMIC-III dataset using the SIAN-2 model.
Although the full model is not quite as interpretable as linear regression, the entire model which beats random forests, deep networks, and boosting machines fits on only a few pages.
It may take some time to digest each of the shape functions, but it is easier than visualizing an arbitrary 30-dimensional function.

\begin{figure}[h]
    \centering
    \includegraphics[width=0.27\textwidth]{images/shape_functions/shape_fn_0.pdf}
    \includegraphics[width=0.23\textwidth]{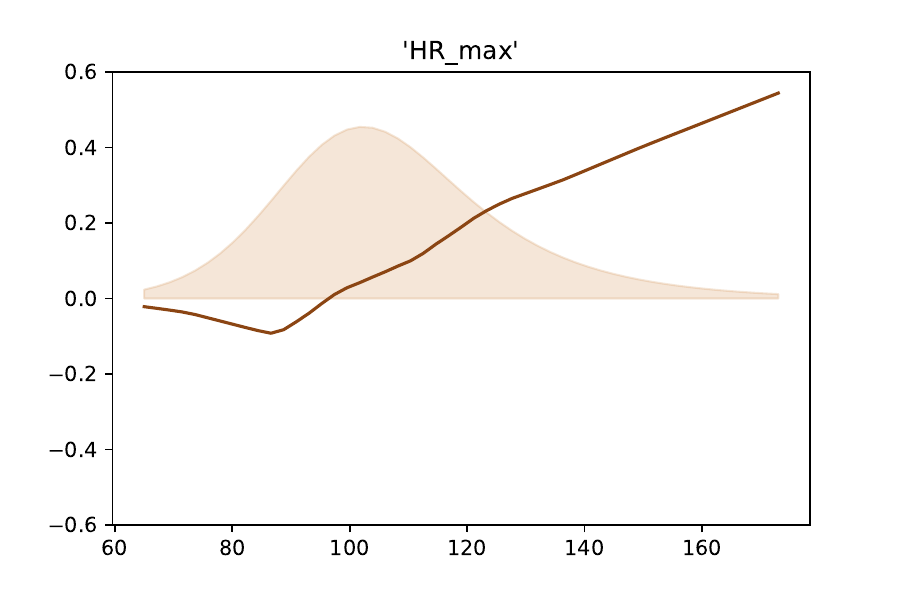}
    \includegraphics[width=0.23\textwidth]{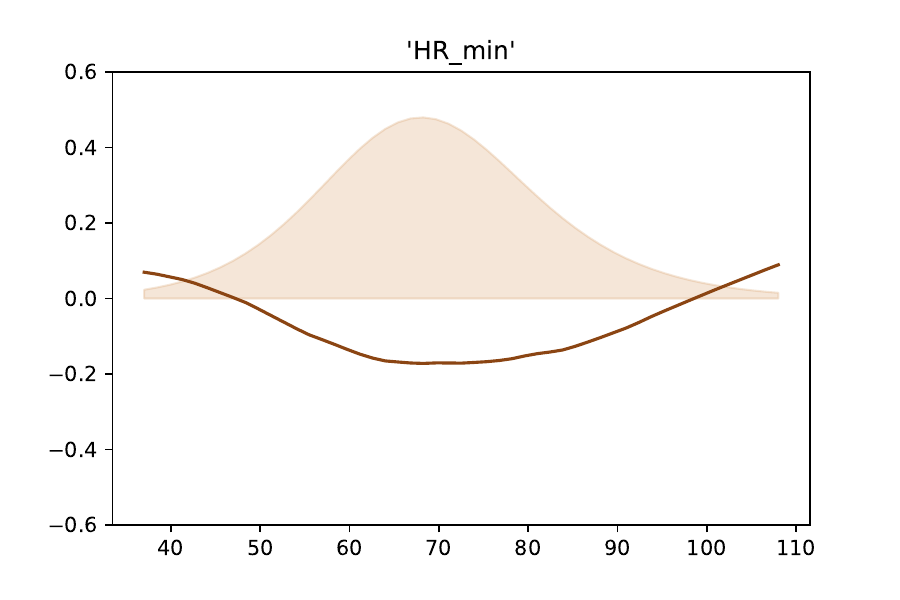}
    \includegraphics[width=0.23\textwidth]{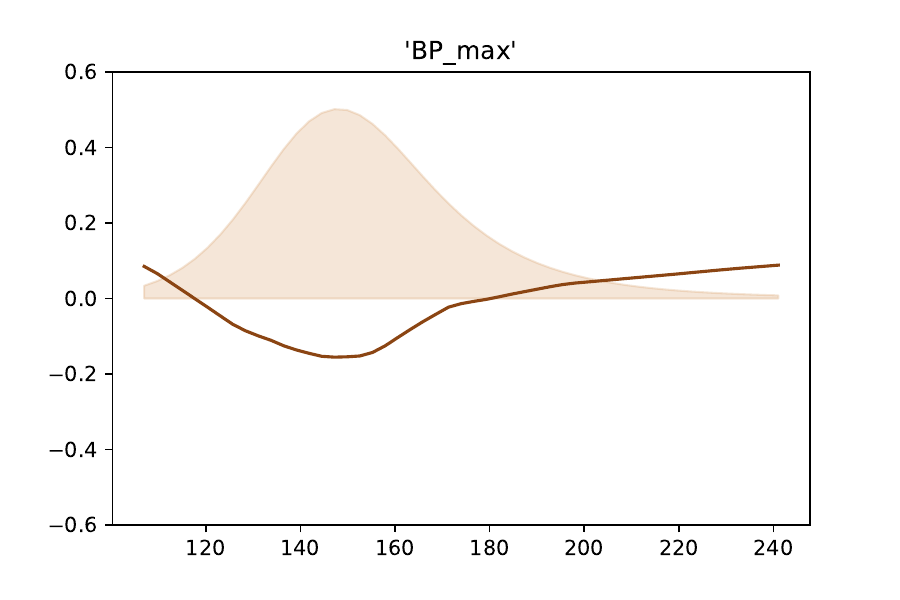}
    \includegraphics[width=0.23\textwidth]{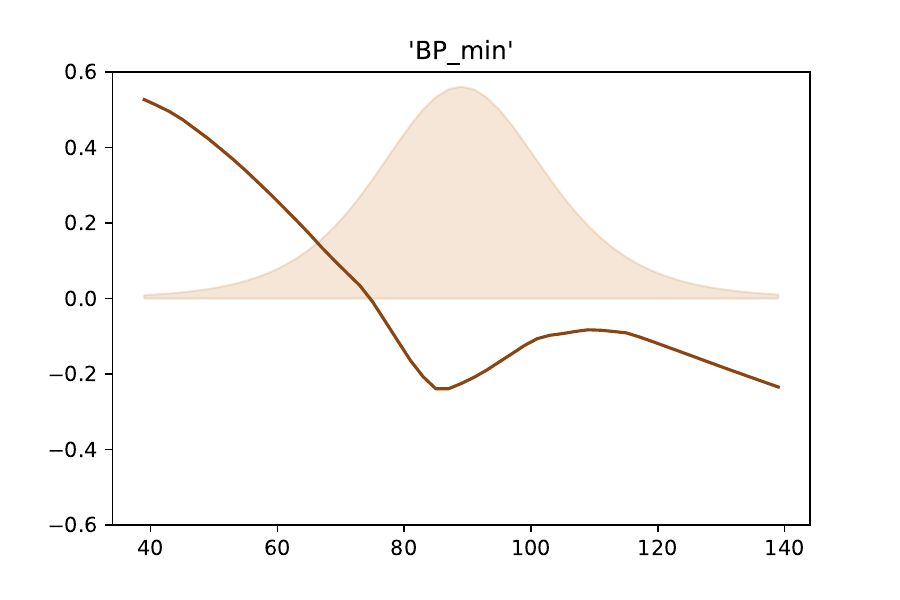}
    \includegraphics[width=0.23\textwidth]{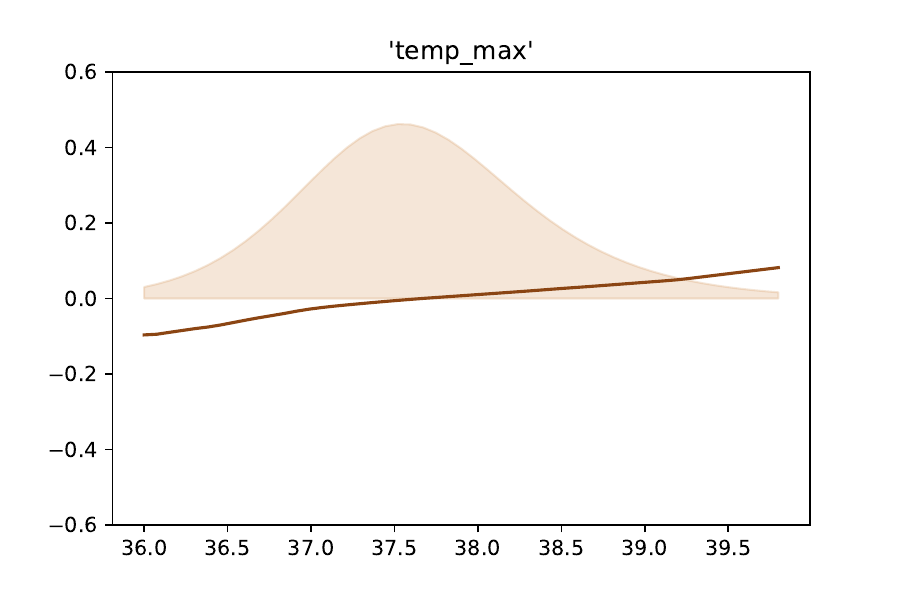}
    \includegraphics[width=0.23\textwidth]{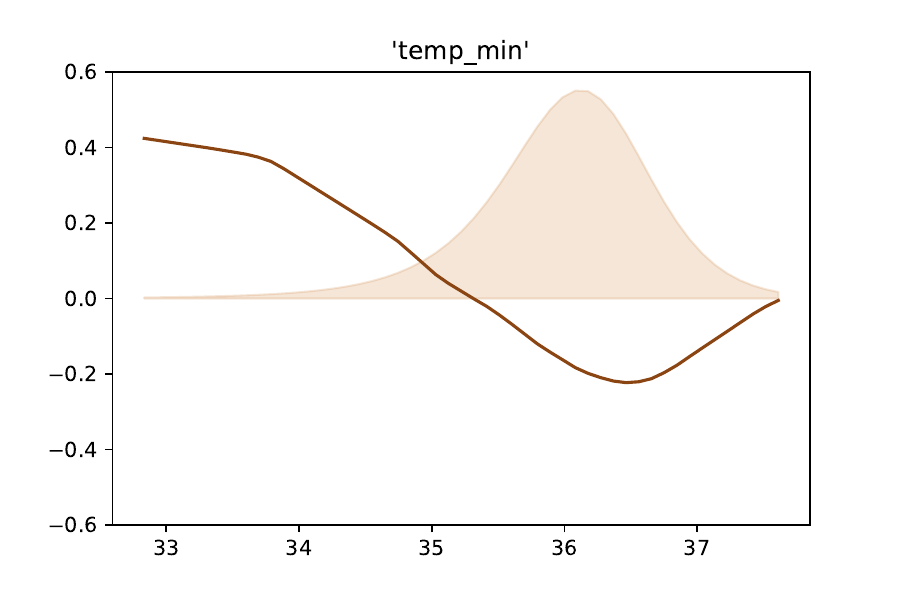}
    \caption{MIMIC 1D Shape Functions \\ (age, heart rate, blood pressure, temperature)}
\end{figure}

\begin{figure}[h]
    \centering
    \includegraphics[width=0.23\textwidth]{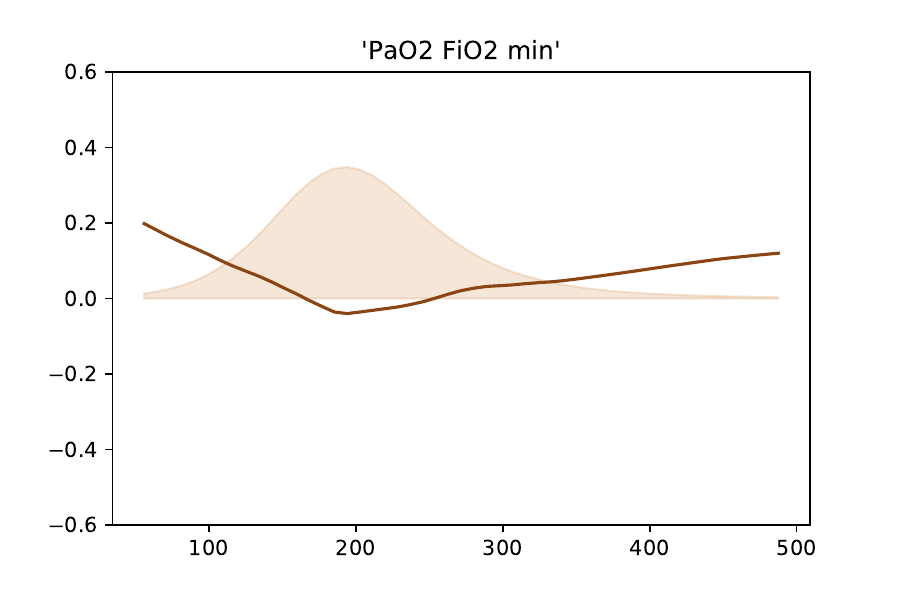}
    \includegraphics[width=0.23\textwidth]{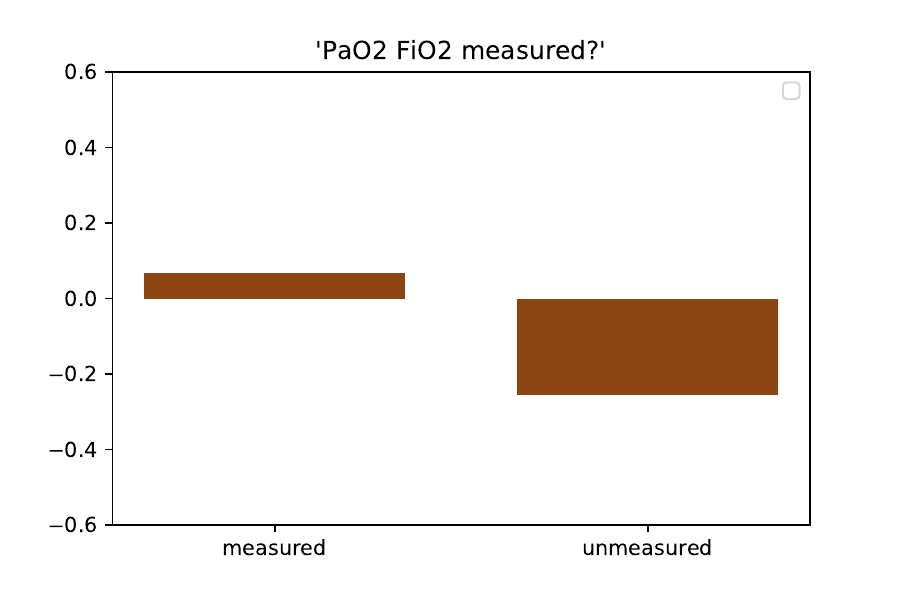}
    \includegraphics[width=0.23\textwidth]{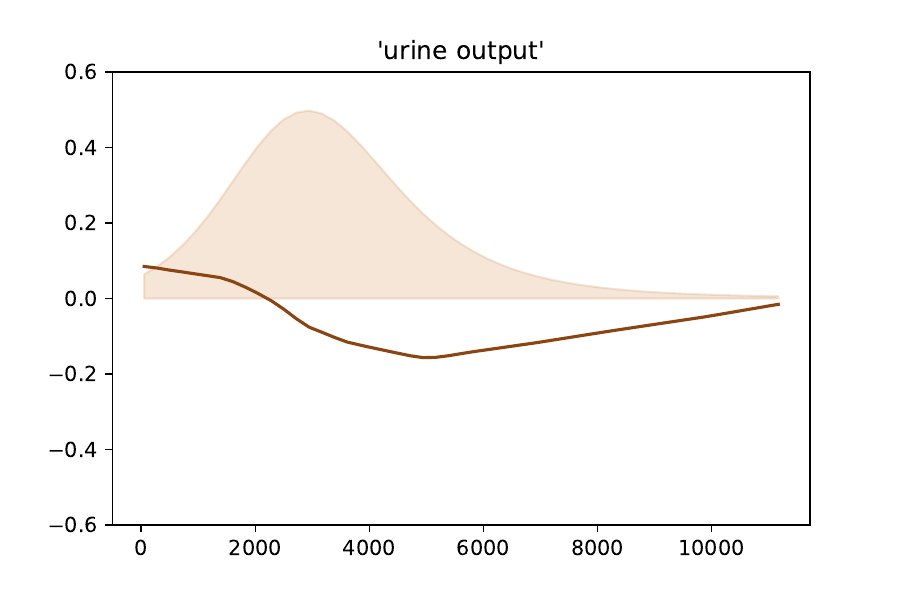}
    \includegraphics[width=0.23\textwidth]{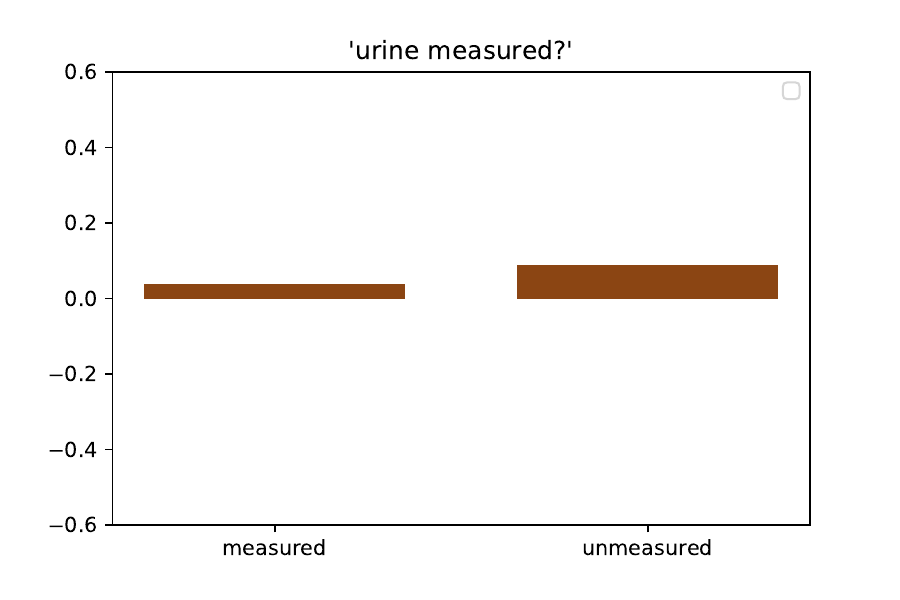}
    \caption{MIMIC 1D Shape Functions \\ (P/F oxygenation ratio, urine) }
\end{figure}
\newpage

\begin{figure}[h]
    \centering
    \includegraphics[width=0.23\textwidth]{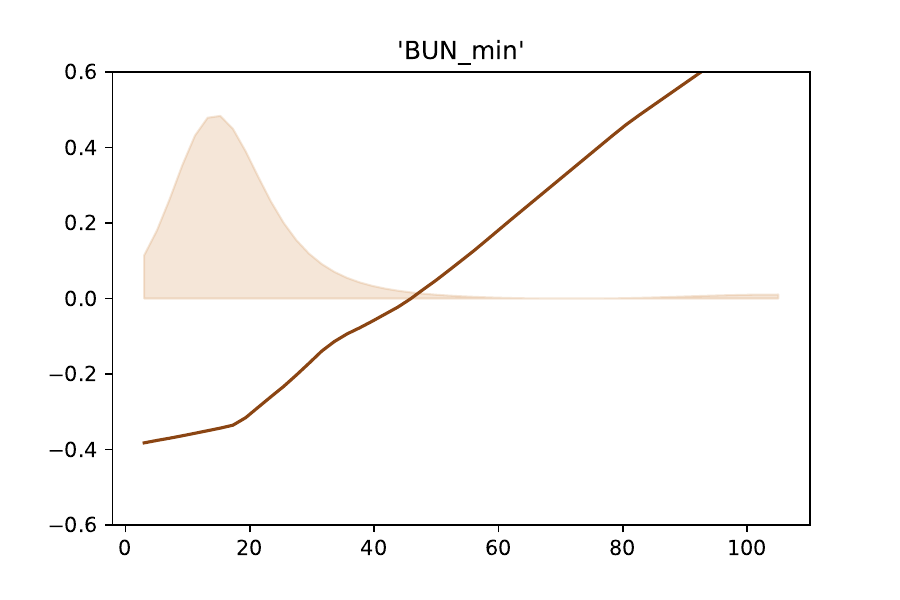}
    \includegraphics[width=0.23\textwidth]{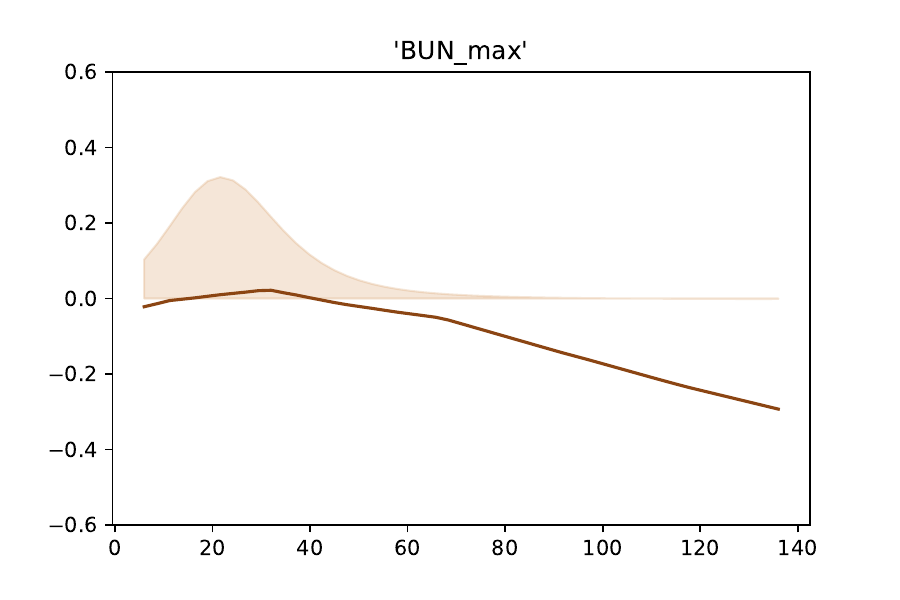}
    \includegraphics[width=0.23\textwidth]{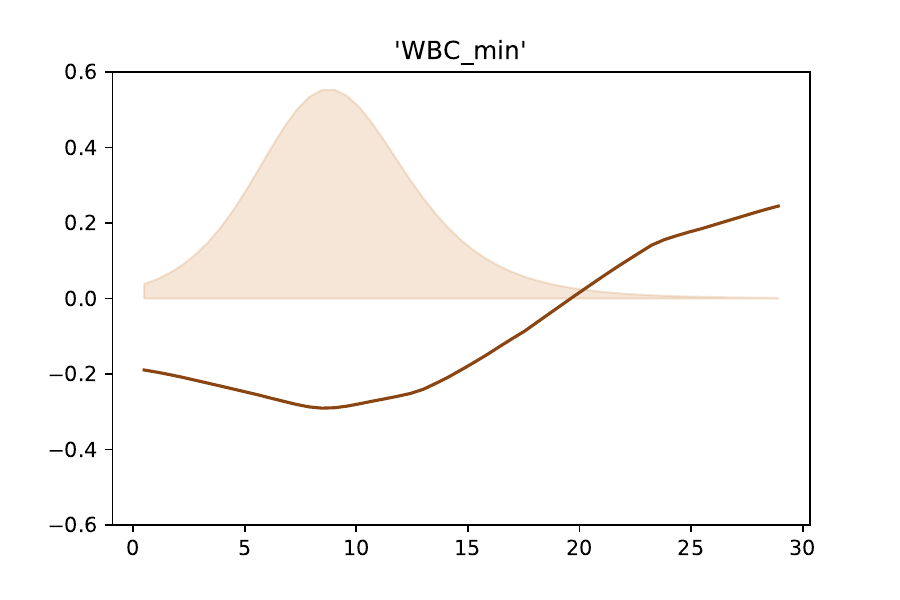}
    \includegraphics[width=0.23\textwidth]{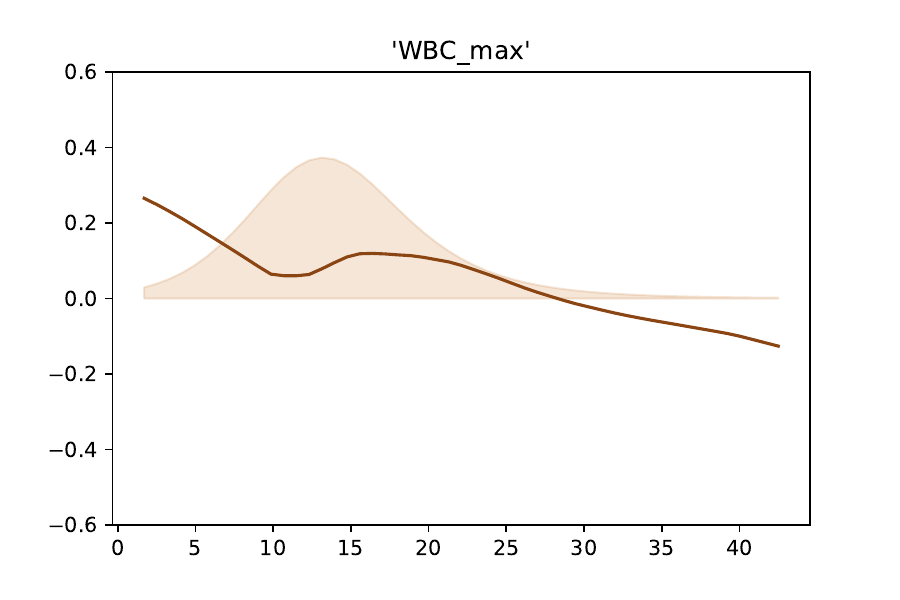}
    \includegraphics[width=0.23\textwidth]{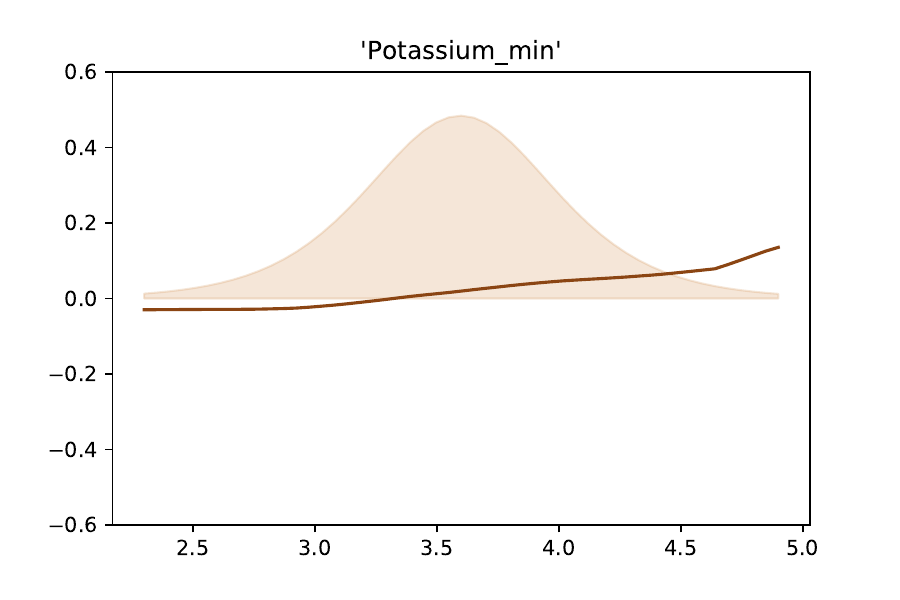}
    \includegraphics[width=0.23\textwidth]{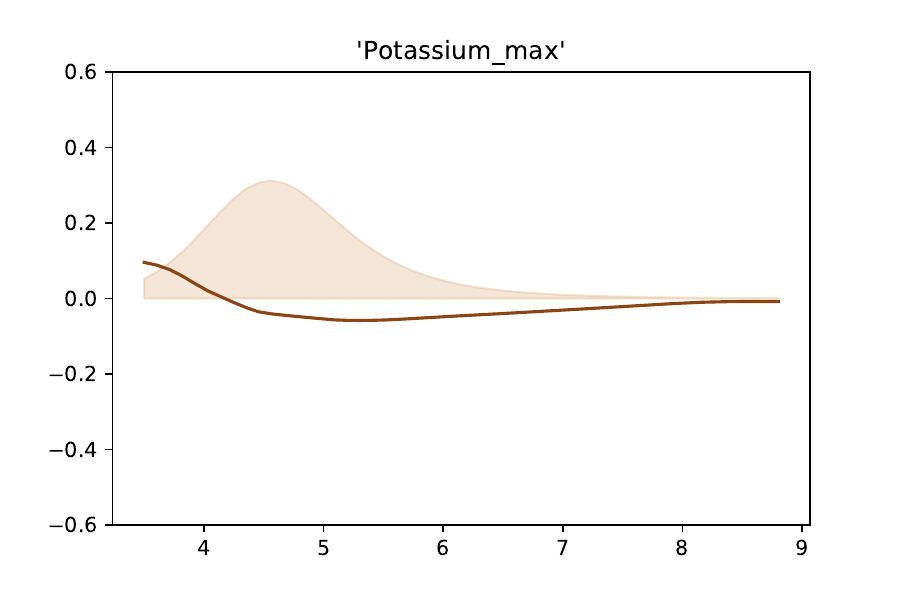}
    \includegraphics[width=0.23\textwidth]{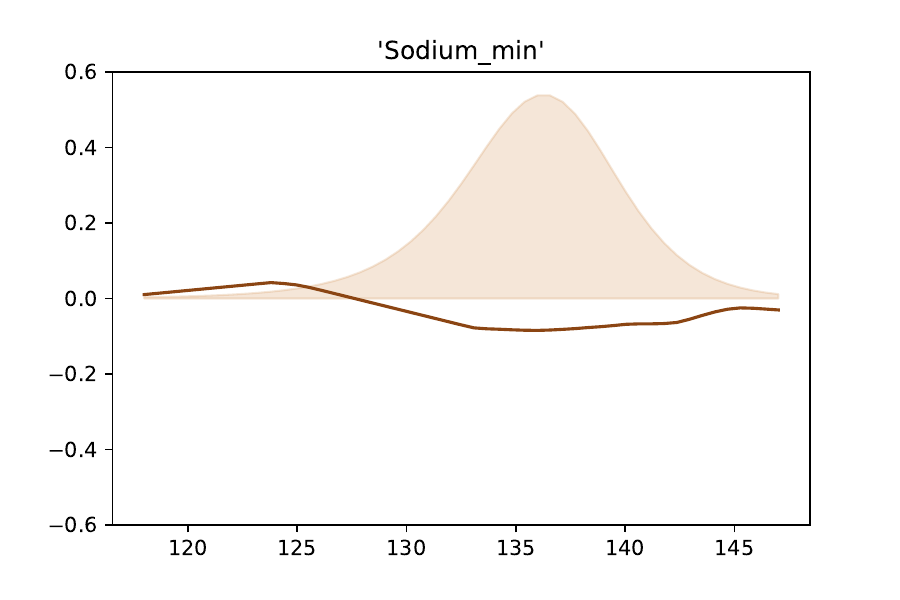}
    \includegraphics[width=0.23\textwidth]{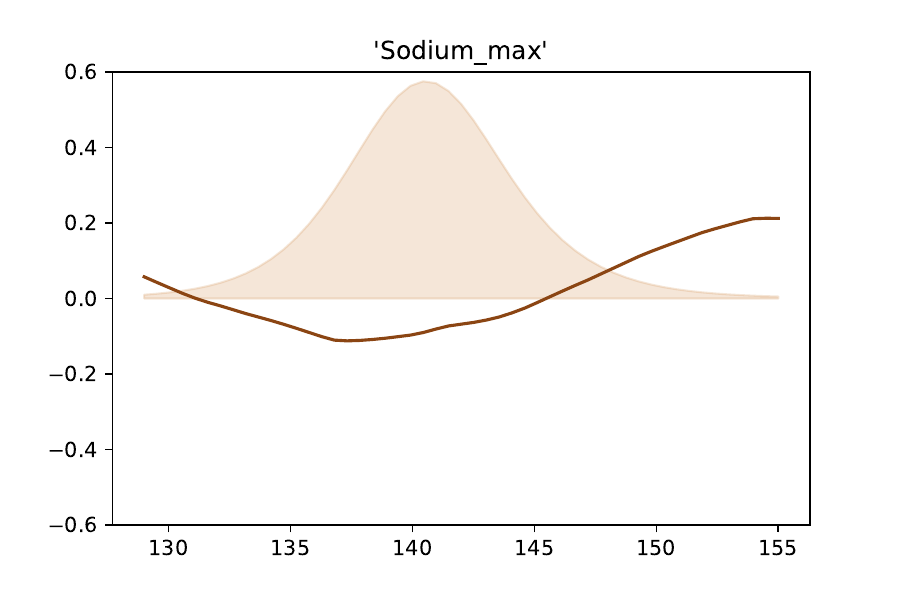}
    \includegraphics[width=0.23\textwidth]{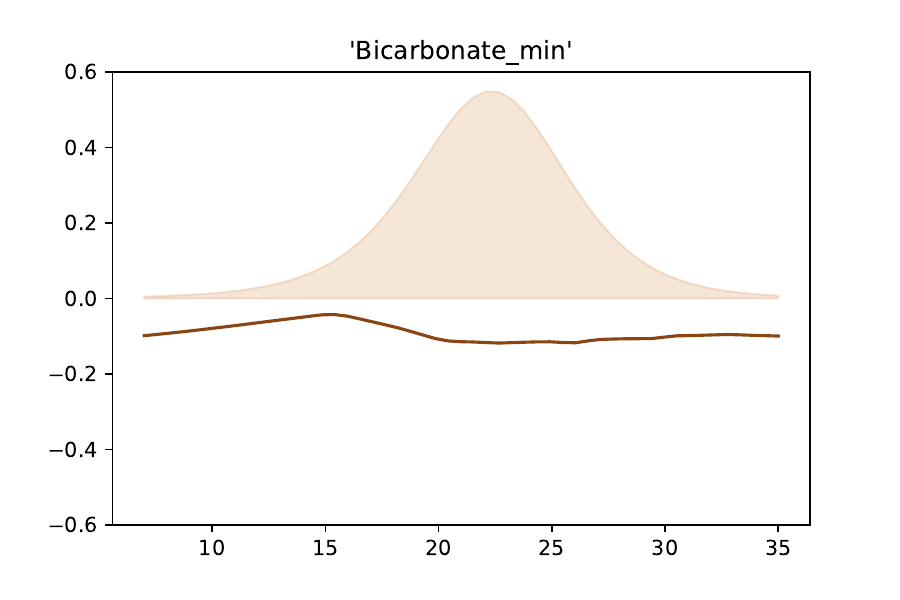}
    \includegraphics[width=0.23\textwidth]{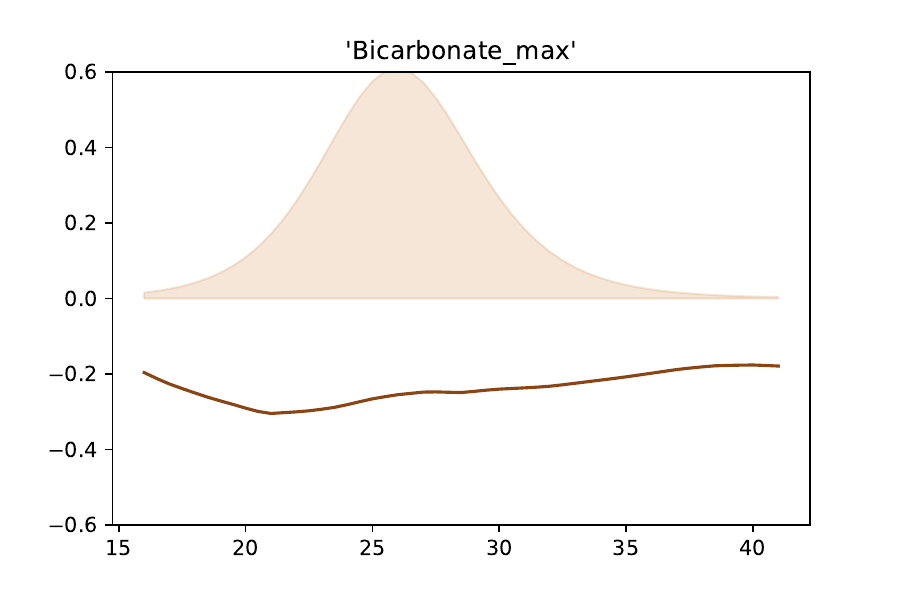}
    \includegraphics[width=0.23\textwidth]{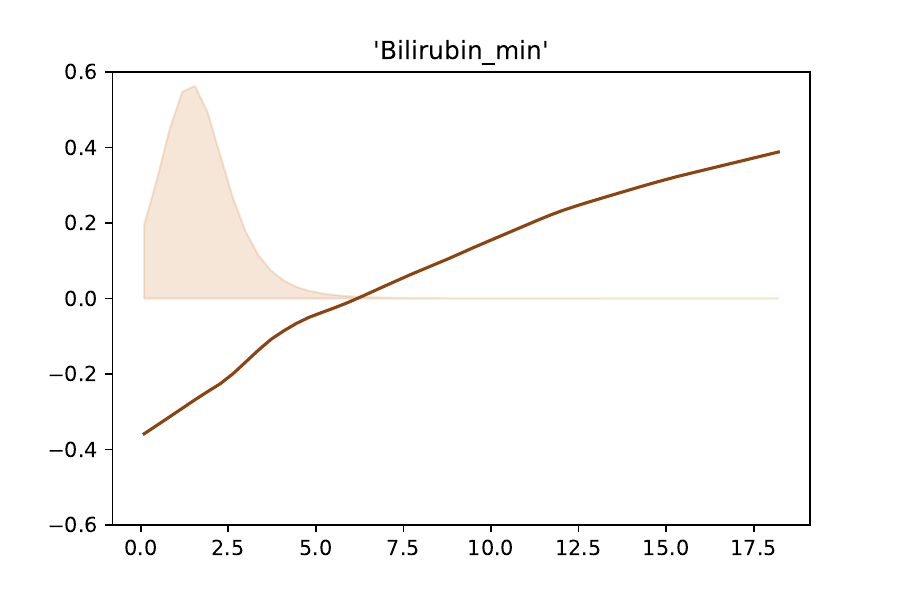}
    \includegraphics[width=0.23\textwidth]{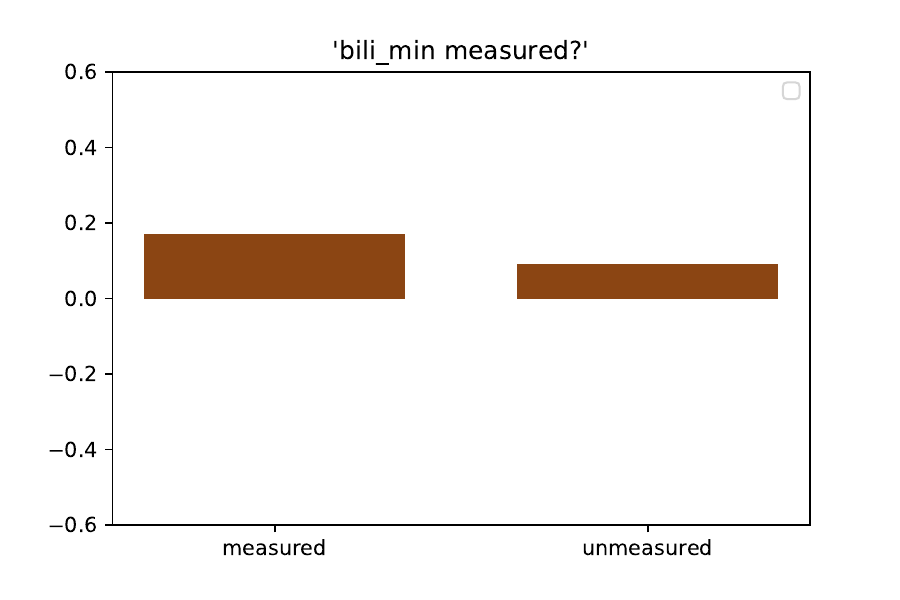}
    \includegraphics[width=0.23\textwidth]{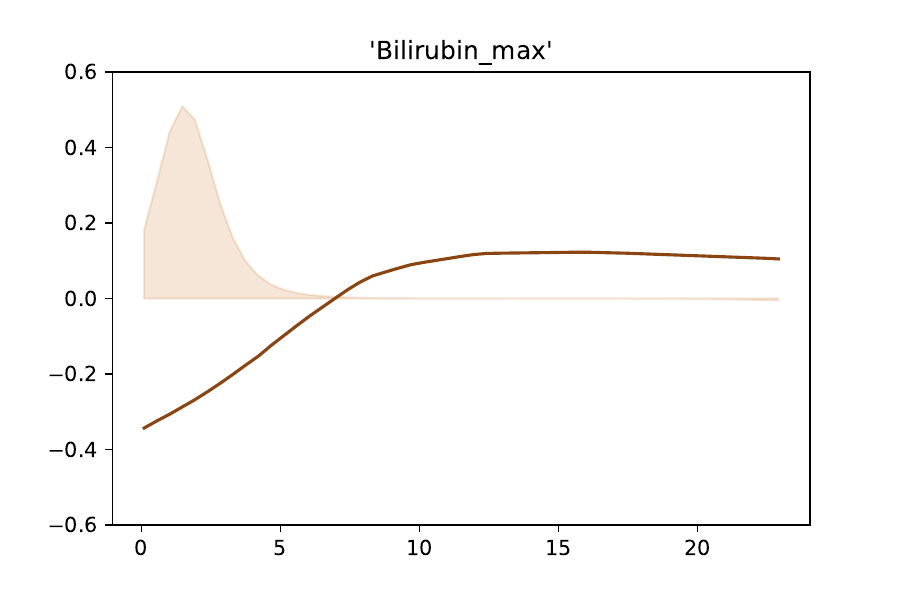}
    \includegraphics[width=0.23\textwidth]{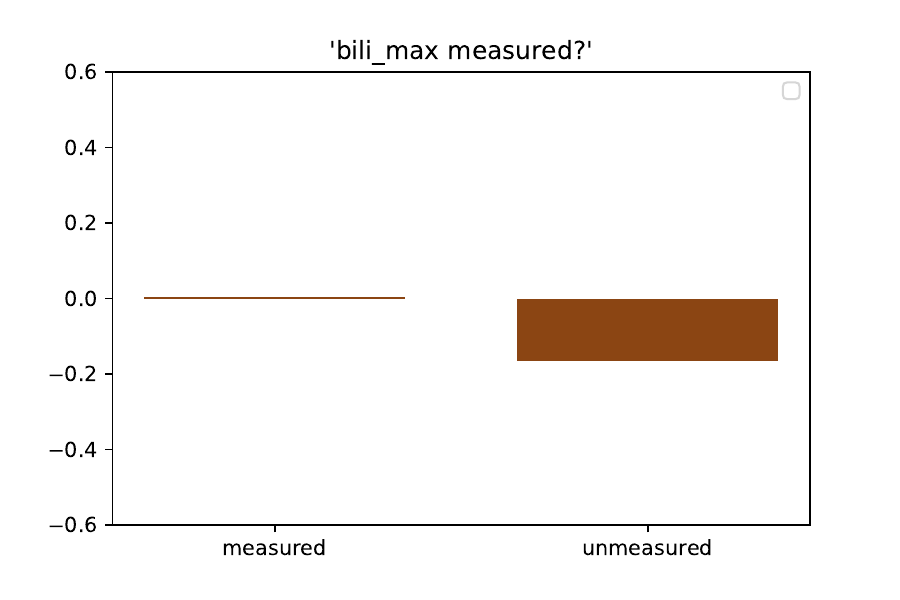}
    \caption{MIMIC 1D Shape Functions \\ (blood urea nitrogen, white blood cell count, potassium, sodium, bicarbonate, bilirubin)}
\end{figure}
\newpage

\begin{figure}[h]
    \centering
    \includegraphics[width=0.23\textwidth]{images/shape_functions/shape_fn_25.pdf}
    \includegraphics[width=0.23\textwidth]{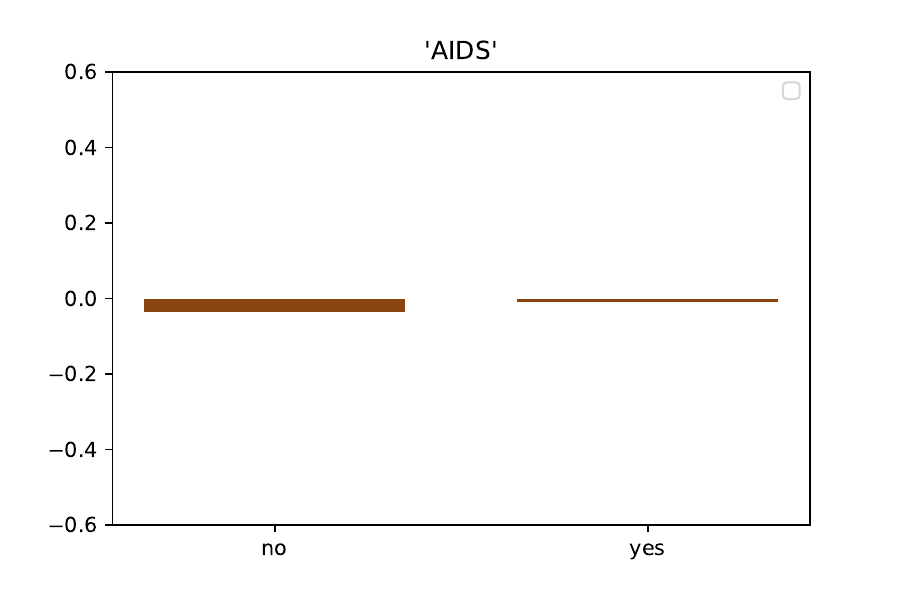}
    \includegraphics[width=0.23\textwidth]{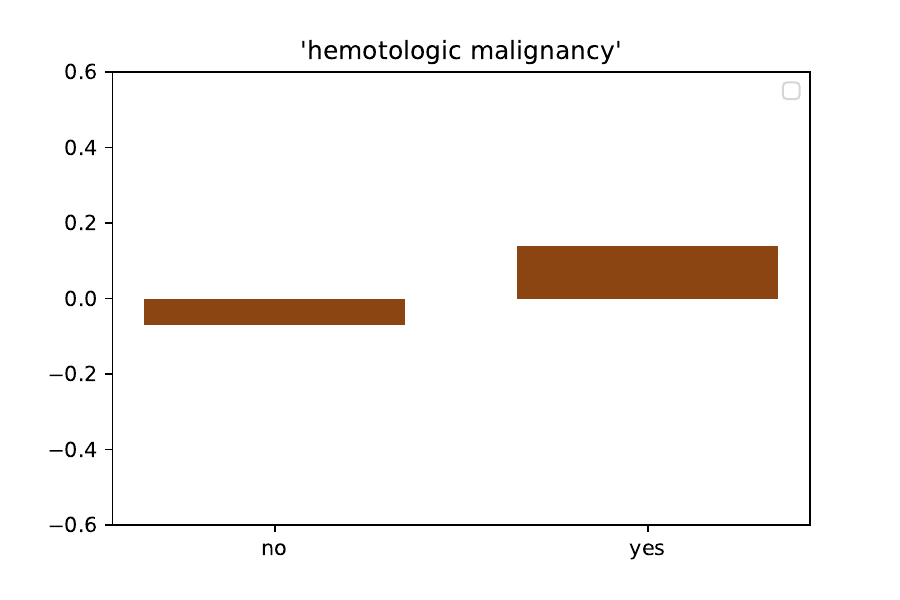}
    \includegraphics[width=0.23\textwidth]{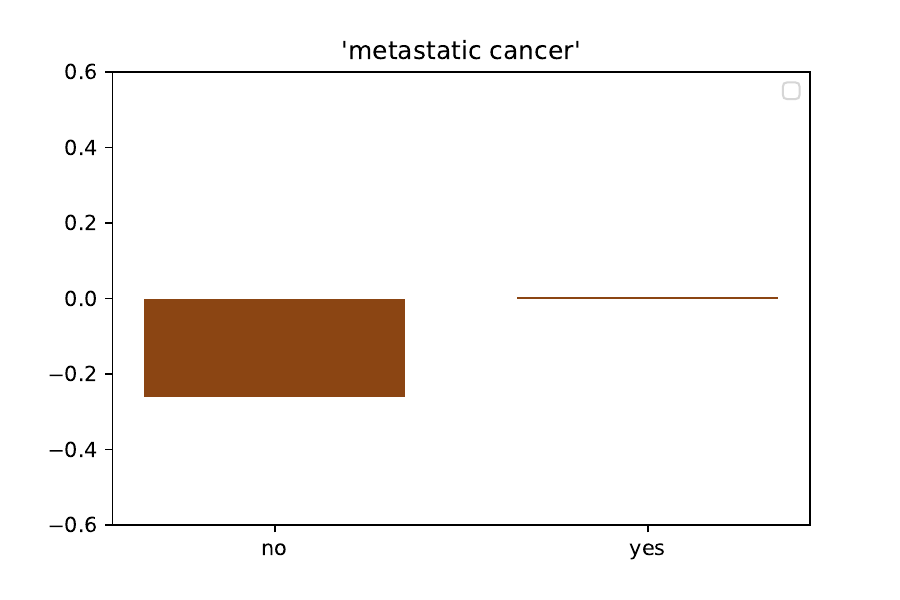}
    \includegraphics[width=0.27\textwidth]{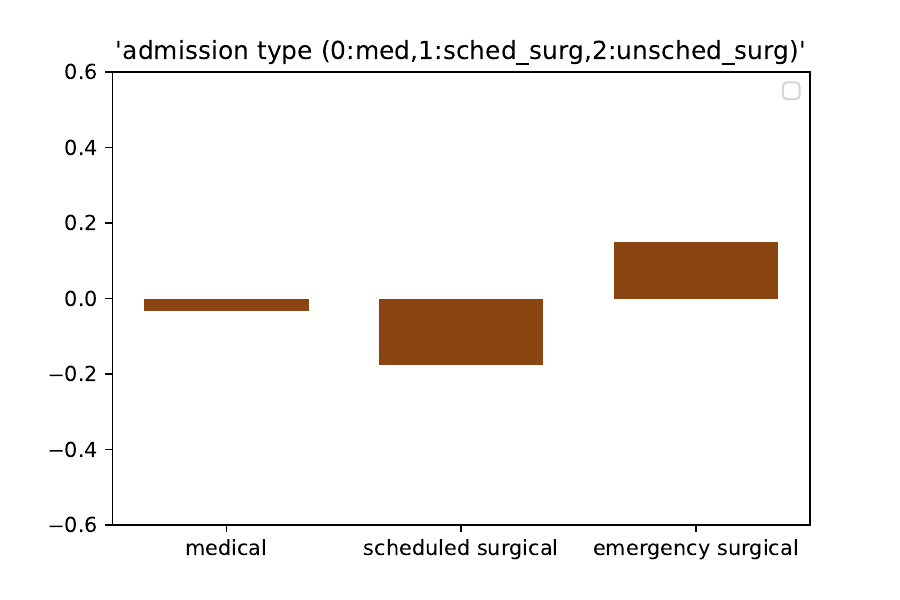}
    \caption{MIMIC 1D Shape Functions \\ (glascow coma scale, AIDS, hematologic malignancy, metastatic cancer, admission visit type)}
\end{figure}

\begin{figure}[h]
    \centering
    \includegraphics[width=0.23\textwidth]{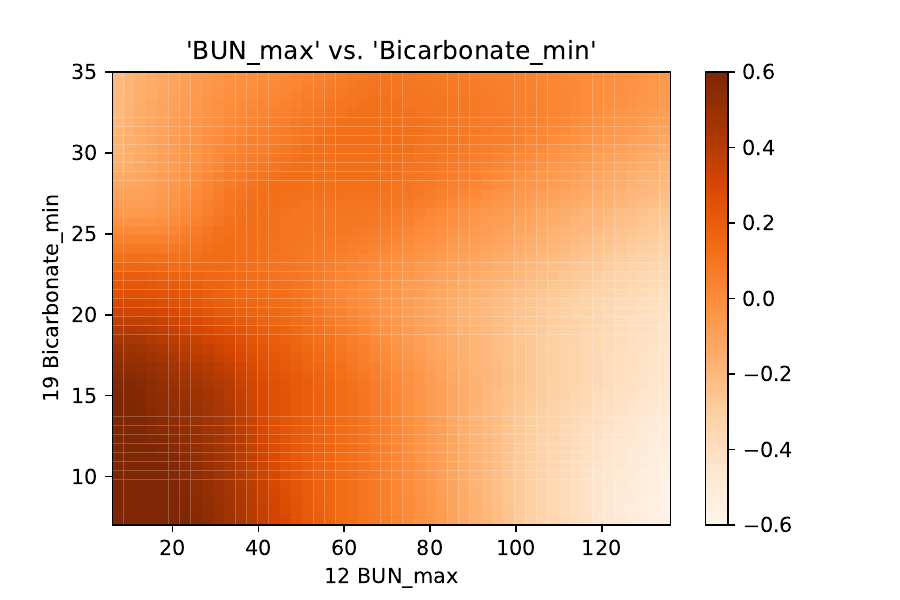}
    \includegraphics[width=0.23\textwidth]{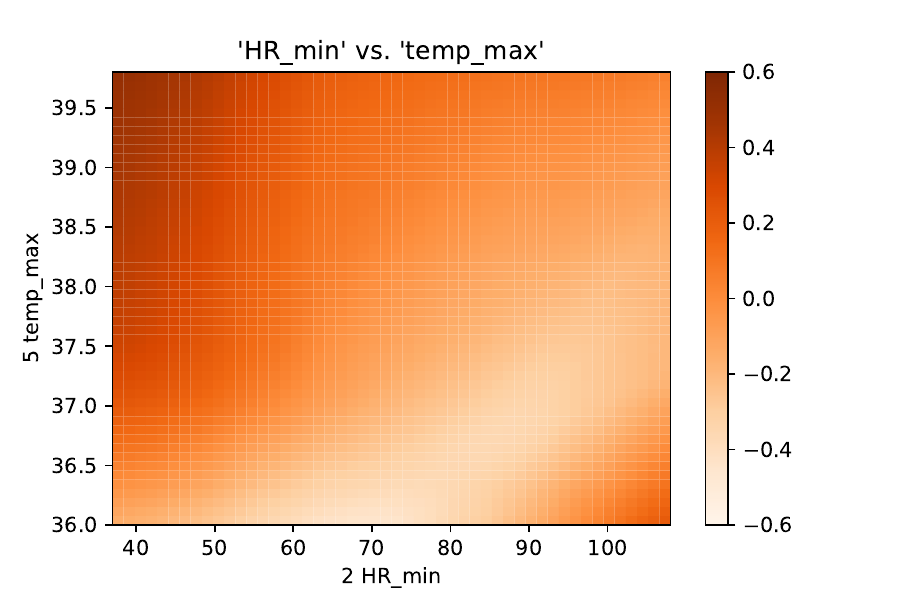}
    \includegraphics[width=0.23\textwidth]{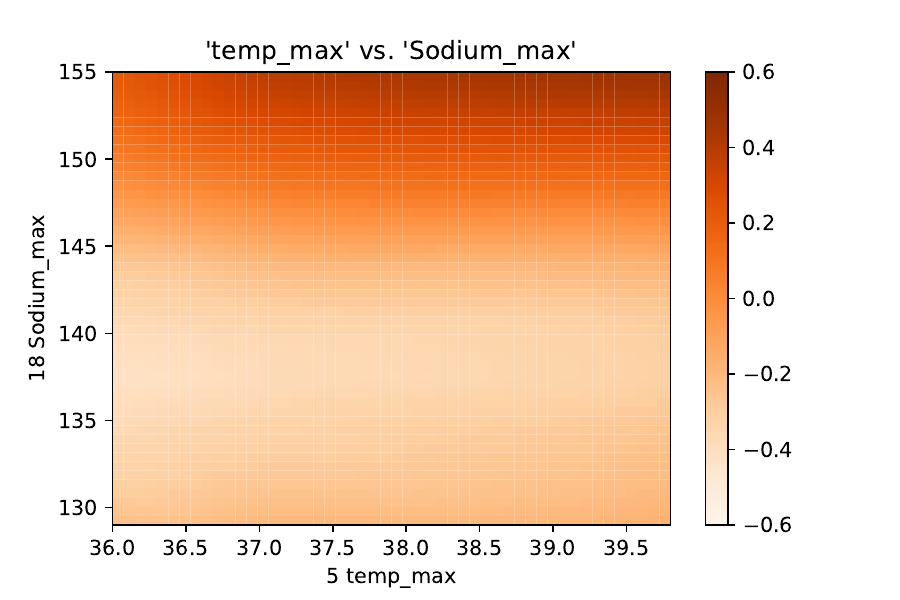}
    \includegraphics[width=0.23\textwidth]{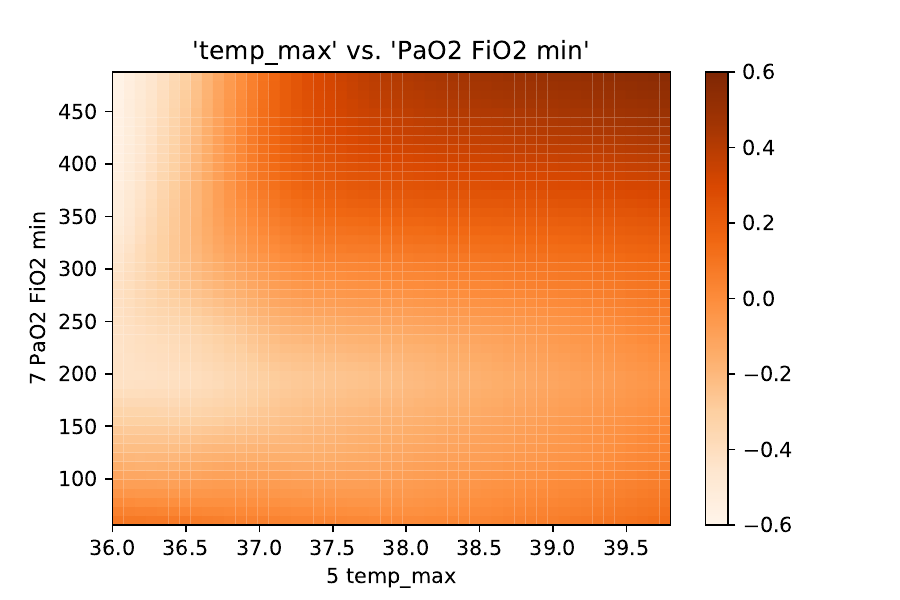}
    \includegraphics[width=0.23\textwidth]{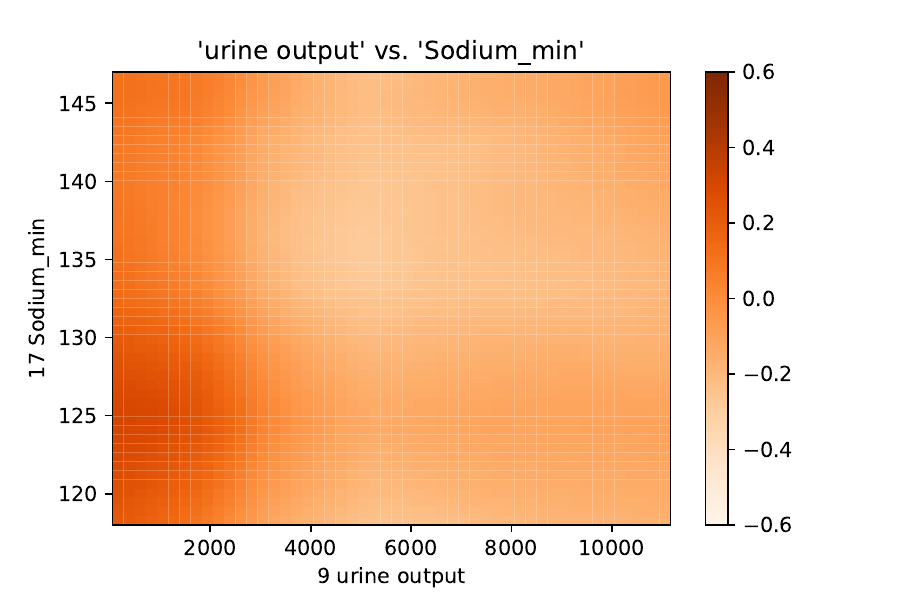}
    \includegraphics[width=0.23\textwidth]{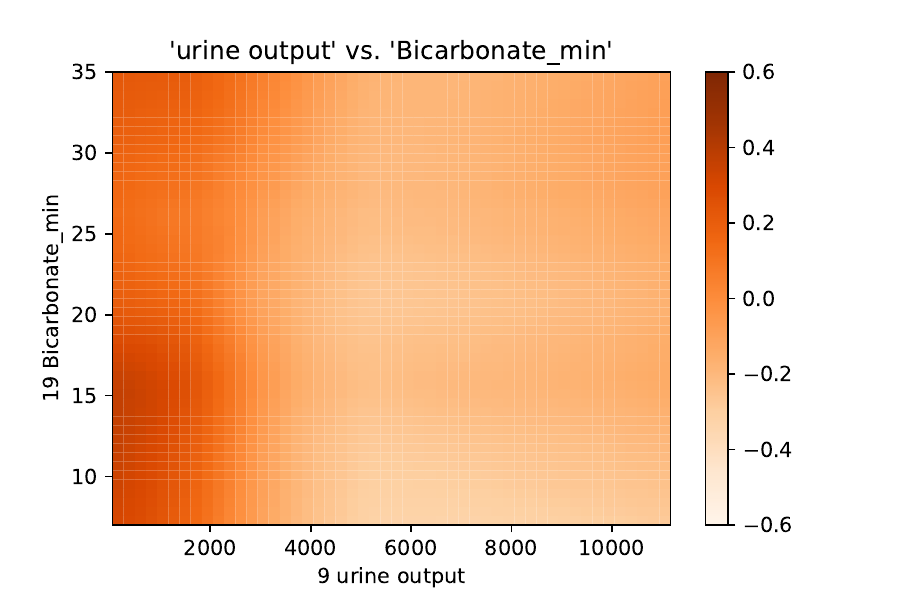}
    \includegraphics[width=0.23\textwidth]{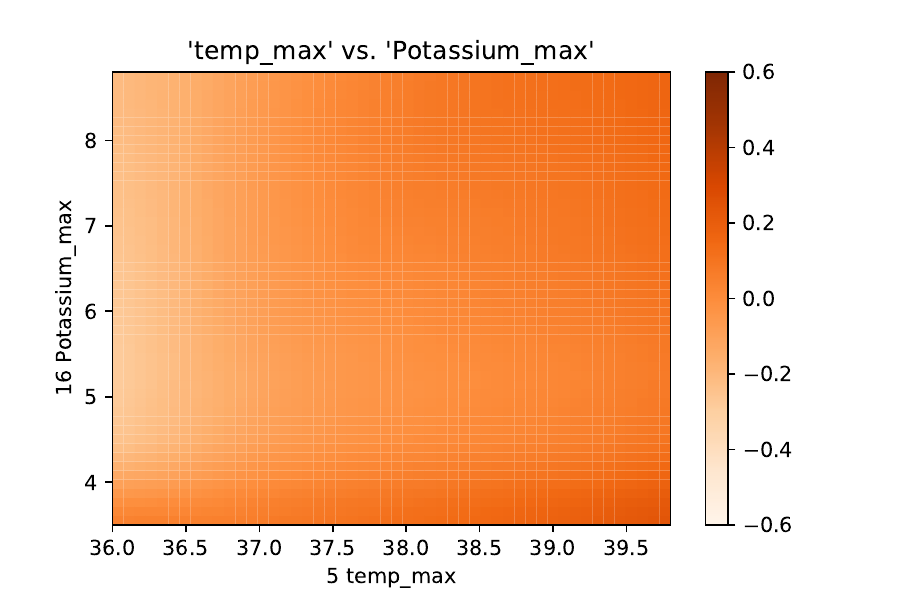}
    \includegraphics[width=0.23\textwidth]{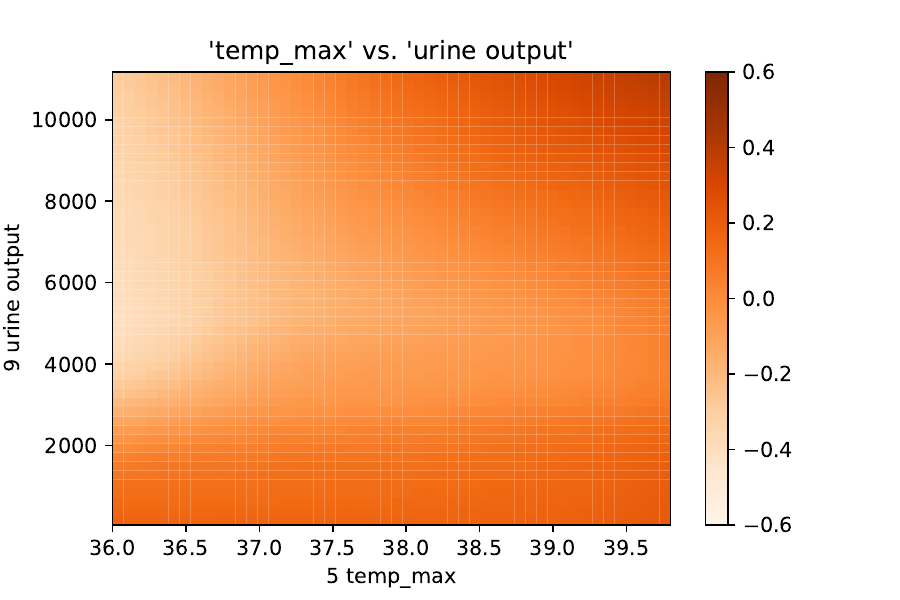}
    \caption{MIMIC 2D Shape Functions}
\end{figure}
\newpage

\begin{figure}[t!]
    \centering
    \includegraphics[width=0.23\textwidth]{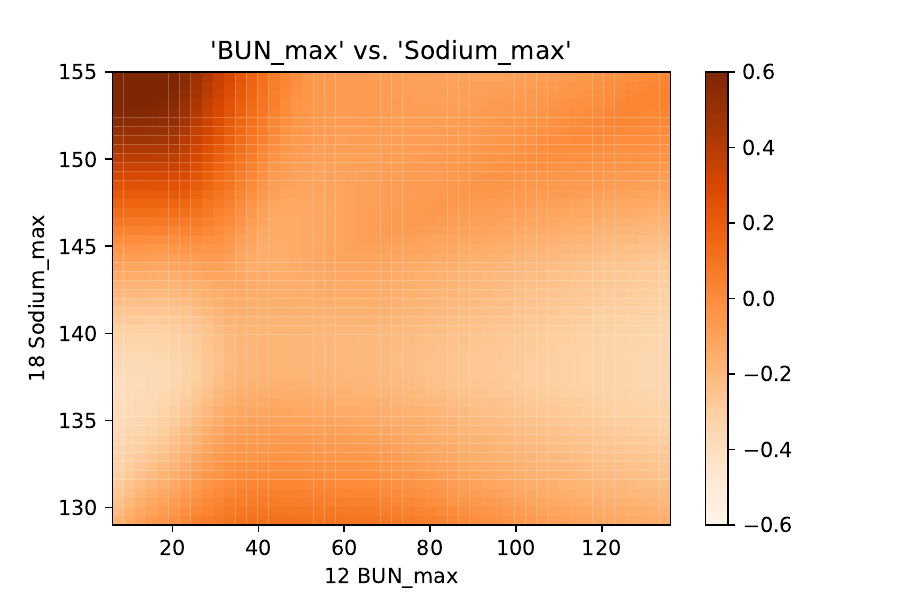}
    \includegraphics[width=0.23\textwidth]{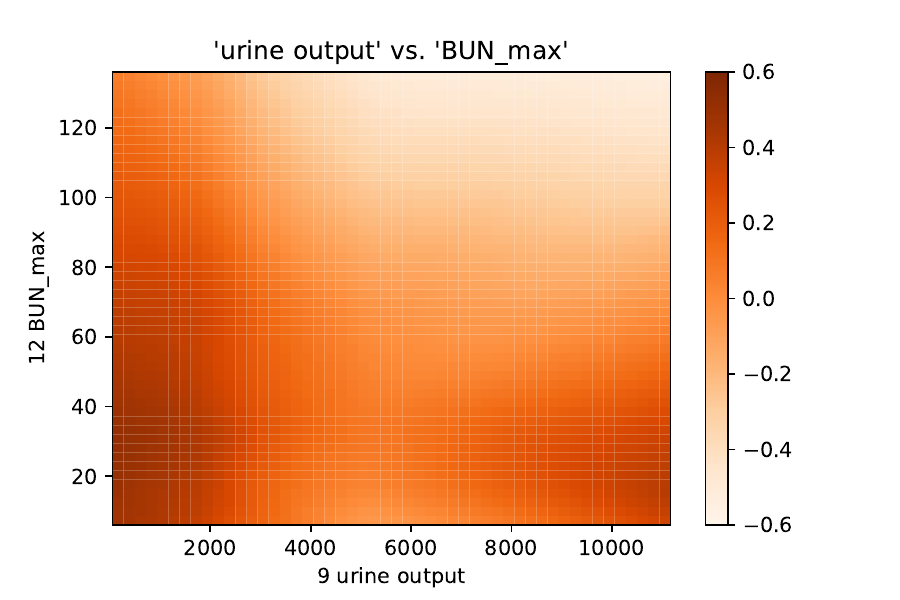}
    \includegraphics[width=0.23\textwidth]{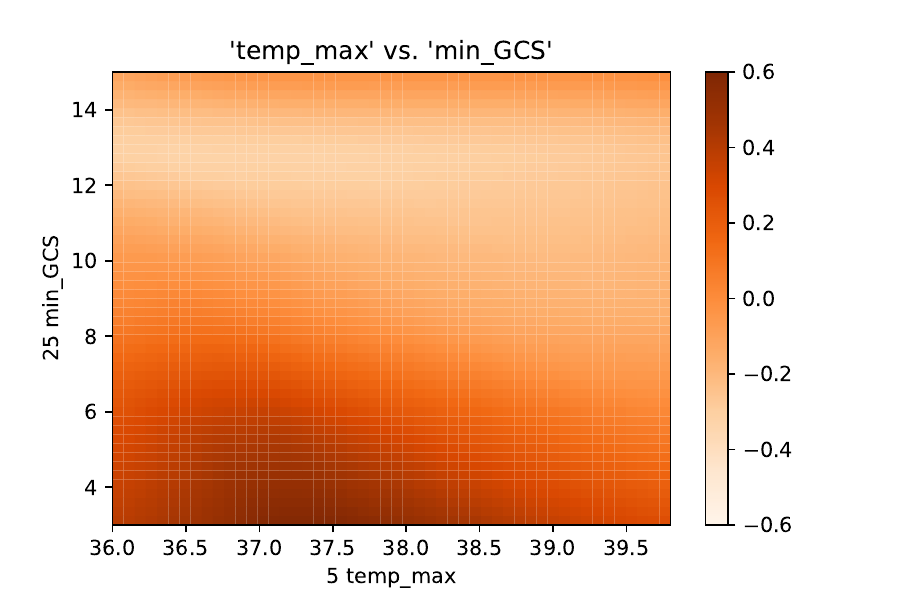}
    \includegraphics[width=0.23\textwidth]{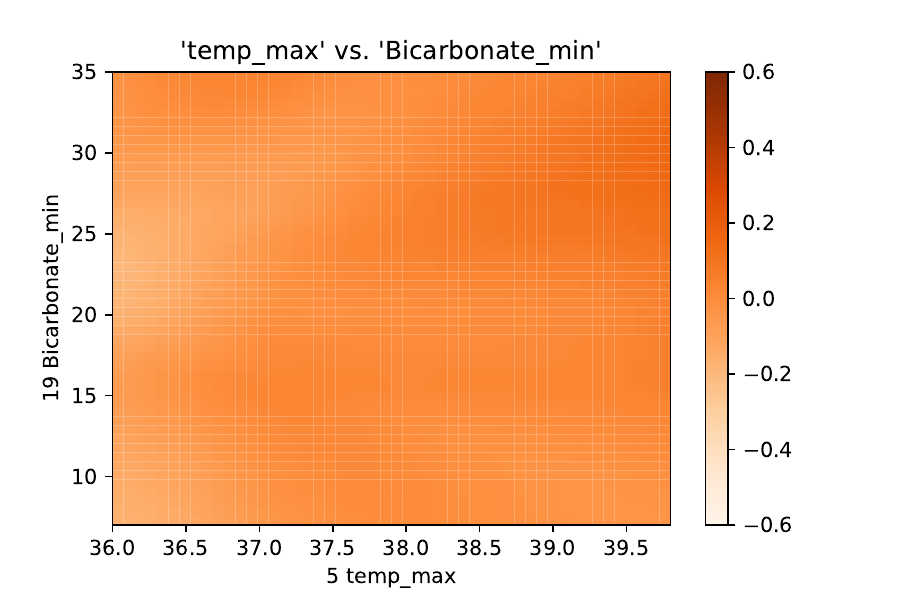}
    \includegraphics[width=0.23\textwidth]{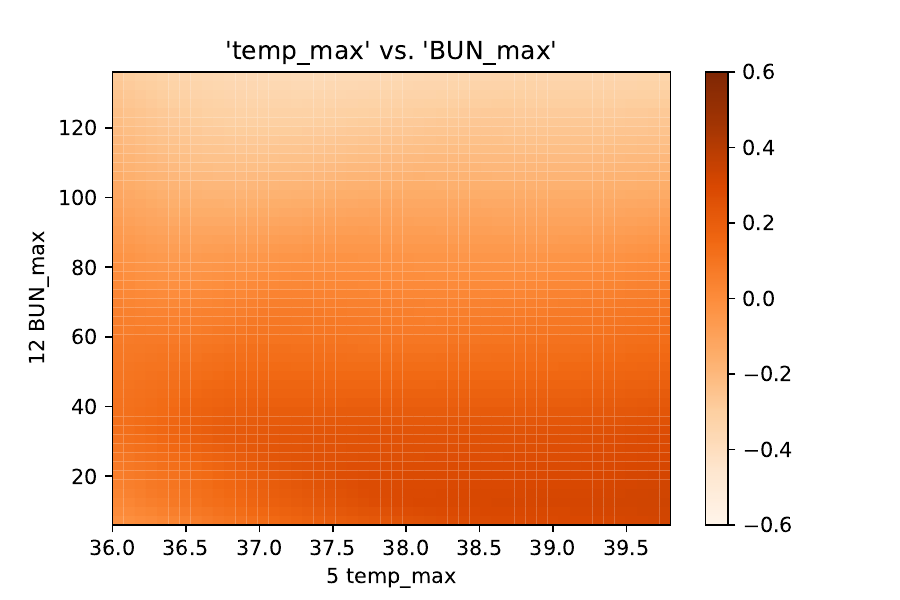}
    \includegraphics[width=0.23\textwidth]{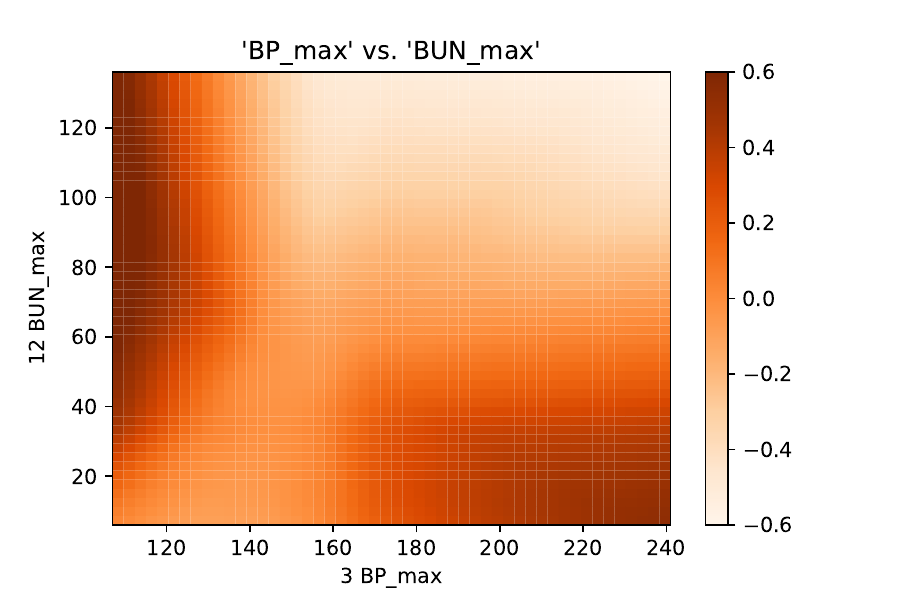}
    \includegraphics[width=0.23\textwidth]{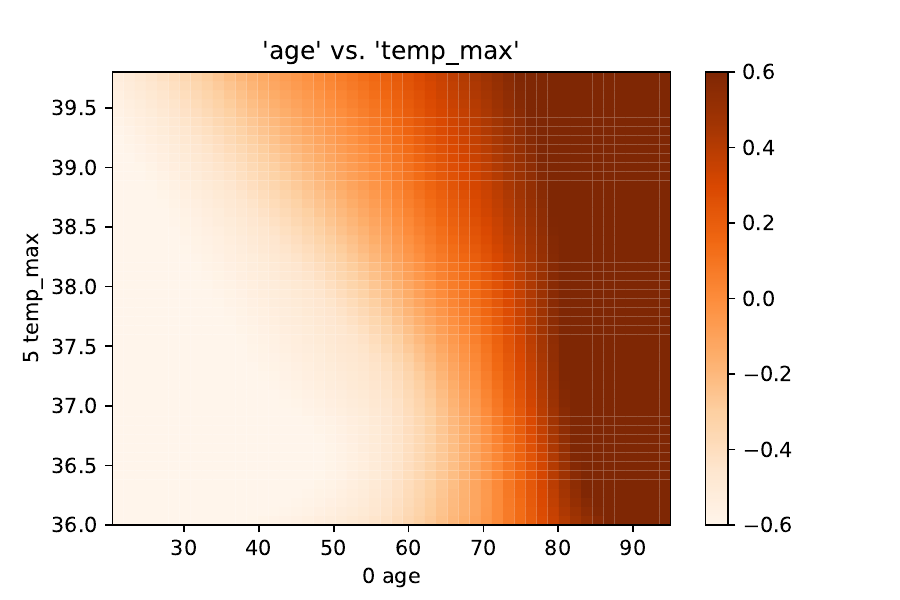}
    \includegraphics[width=0.23\textwidth]{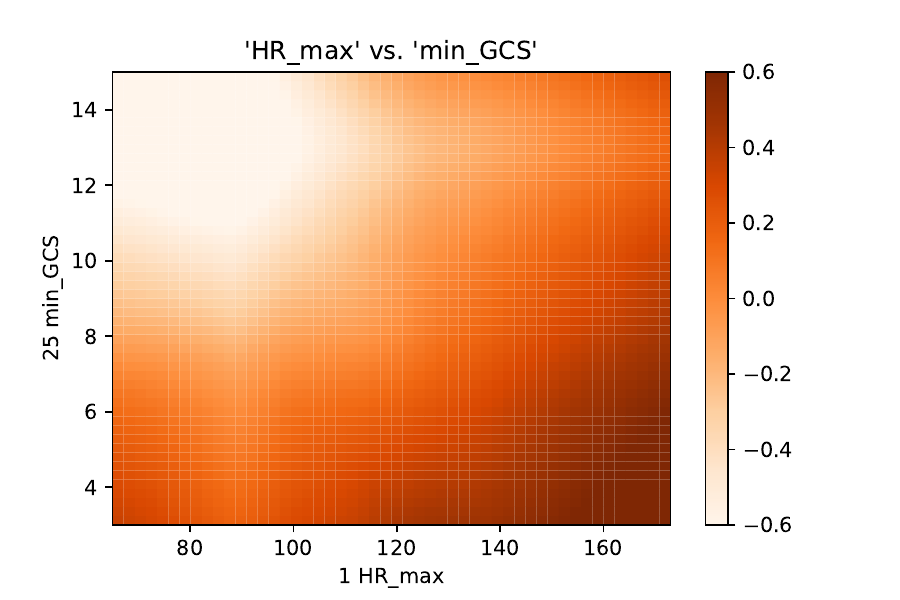}
    \caption{MIMIC 2D Shape Functions}
\end{figure}
\quad %phantom text to keep figure at top of page... idk why

\end{document}